\documentclass{article}
\usepackage{arxiv}

\usepackage[round]{natbib}

\usepackage[utf8]{inputenc} % allow utf-8 input
\usepackage[T1]{fontenc}    % use 8-bit T1 fonts
\usepackage{hyperref}       % hyperlinks
\usepackage{url}            % simple URL typesetting
\usepackage{booktabs}       % professional-quality tables
\usepackage{amsfonts}       % blackboard math symbols
\usepackage{nicefrac}       % compact symbols for 1/2, etc.
\usepackage{microtype}      % microtypography
\usepackage{xcolor}         % colors

\usepackage{subcaption}
\usepackage{algorithm, algpseudocode}

\renewcommand{\Re}{\mathbb{R}}

\usepackage{amsmath}
\usepackage{amssymb}
\usepackage{amsthm}
\usepackage{mathtools}
\usepackage{bm}
\usepackage{enumitem}
\usepackage{lipsum}

\usepackage{graphicx}
\usepackage{dsfont}
\usepackage{xcolor}

\makeatletter
\def\captionof#1#2{{\def\@captype{#1}#2}}
\makeatother
\def\1{\mbox{\bf 1}}
\def\R{\mathbb{R}}

\def\N{\mathbb{N}}
\def\P{\mathbb{P}}
\def\E{\mathbb{E}}

\def\R{\mathbb{R}}

\DeclareMathOperator{\var}{Var}

\newtheorem{theo}{Theorem}
\newtheorem{lem}{Lemma}
\newtheorem{prop}{Proposition}
\newtheorem{Def}{Definition}
\newtheorem{cor}{Corollary}
\newtheorem{Def/Prop}{Definition-Proposition}

\newcounter{exos}
\renewcommand\theexos{\arabic{exos}}

\newcounter{prob}
\renewcommand\theprob{\arabic{prob}}

\definecolor{todo}{RGB}{250,120,70}

\newcommand\given[1][]{\:#1\vert\:}
\newcommand\norm[1]{\lVert #1 \rVert}

\newcommand\abs[1]{\lvert #1 \rvert}

\newcommand*{\QEDA}{\hfill\ensuremath{\blacksquare}}

\title{Nearest Neighbor Sampling for Covariate Shift Adaptation}

\author{
    Fran\c{c}ois Portier \\
    \texttt{francois.portier@ensai.fr} \\
    Department of Statistics,\\
    Univ Rennes, Ensai, CNRS, CREST---UMR 9194, F-35000 Rennes, France
    \AND
    Lionel Truquet \\
    \texttt{lionel.truquet@ensai.fr} \\
    Department of Statistics,\\
    Univ Rennes, Ensai, CNRS, CREST---UMR 9194, F-35000 Rennes, France
    \AND
    Ikko Yamane \\
    \texttt{ikko.yamane@ensai.fr} \\
    Department of Computer Science,\\
    Univ Rennes, Ensai, CNRS, CREST---UMR 9194, F-35000 Rennes, France
}

\begin{document}

\maketitle

\begin{abstract}%
Many existing covariate shift adaptation methods estimate sample weights given to loss values to mitigate the gap between the source and the target distribution.
However, estimating the optimal weights typically involves computationally expensive matrix inversion and hyper-parameter tuning.
In this paper, we propose a new covariate shift adaptation method which avoids estimating the weights.
The basic idea is to directly work on unlabeled target data, labeled according to the $k$-nearest neighbors in the source dataset.
Our analysis reveals that setting $k = 1$ is an optimal choice. This property removes the necessity of tuning the only hyper-parameter $k$ and leads to a running time quasi-linear in the sample size.
Our results include sharp rates of convergence for our estimator, with a tight control of the mean square error and explicit constants.
In particular, the variance of our estimators has the same rate of convergence as for standard parametric estimation despite their non-parametric nature.
The proposed estimator shares similarities with some matching-based treatment effect estimators used, e.g., in biostatistics, econometrics, and epidemiology.
Our experiments show that it achieves drastic reduction in the running time with remarkable accuracy.
\end{abstract}

\section{Introduction}\label{sec:intro}
Traditional machine learning methods assume that
the source data distribution $P$ and the target data distribution $Q$ are identical.
However, this assumption can be violated in practice when there is a \emph{distribution shift}~\citep{chen2022covlabshift} between them.
Various types of shift have been studied in the literature, and one of the most common scenarios is \emph{covariate shift}~\citep{shimodaira2000improving} in which there is a shift in the input distribution: $P_{X} \neq Q_{X}$ while the conditional distribution of the output variable given the input variable is the same: $P_{Y \given X} = Q_{Y \given X}$, where $X$ is the input and $Y$ is the output variable.
The goal of \emph{covariate shift adaptation} is to adapt a supervised learning algorithm to the target distribution using labeled source data and unlabeled target data.

A standard approach to covariate shift is weighting source examples~\citep{shimodaira2000improving}, and many studies focused on improving the weights~\citep{huang2006kmm,gretton2008kmm,yamada2013rulsif,kanamori2009ulsif,sugiyama2007direct,sugiyama2008direct,aminian2022a} in the same line of research.
We refer the reader to Section~\ref{sec:related} for more details of related work.
Since we rarely know the model for how the input distributions can be shifted a priori, non-parametric methods are particularly useful for covariate shift adaptation.
Some of the existing methods allow one to use non-parametric models through kernels.
However, such kernel-based methods take at least quadratic times in computing kernel matrices.
Some methods further need to solve linear systems and take cubic times in the sample size unless one resorts to approximations~\citep{williams2000nystrom,le2013fastfood}.
Moreover, their performance is often sensitive to the choice of hyper-parameters of the kernel.
Typically, one performs a grid search $K$-fold cross-validation for selecting the hyper-parameters, which amplifies the running time by about $K \abs{\Gamma}$, where $\Gamma$ is the set of candidates for the hyper-parameters.
Moreover, the criterion for the hyper-parameter selection is not obvious either because we do not have access to the labels for the target data.
One can use weighted validation scores using the labeled source data with importance sampling, but it is not straightforward to choose what weights to be used for the cross-validation when we are choosing weights.

In this paper, we propose a non-parametric covariate shift adaptation method that is scalable and has no hyper-parameter.
Our idea is to generate synthetic labels for unlabeled target data using a non-parametric conditional sampler constructed from source data. Under the assumption of covariate shift, the target data attached with the generated labels behave like labeled target data.
This sampling technique allows any supervised learning method to be simply applied to the generated data to produce a model already adapted to the target distribution.

While the proposed approach is quite general and can be employed with various sampling methods for the synthetic labeling part, our main result is that a $k$-nearest neighbor ($k$-NN) based sampling method achieves an error of order $ {(k/n)^{1/d} + 1/\sqrt n + 1/\sqrt m}$ for estimating an expectation on the target domain, where $d$ is the data dimensionality, and $n$ and $m$ are the source and the target sample size, respectively. Importantly, our error bounds suggest that $k = 1$ is the most favorable.
This property, which is revealed by a precise scaling of the variance term in $1/\sqrt n $, is a non-trivial and remarkable result, given the $1/\sqrt k$-rate of the variance associated to the $k$-NN estimator of the conditional distribution~\citep[Corollary 1]{portier2021nearest}, and it contrasts with the well-known application of $k$-NN to standard density estimation~\citep{NIPS2014_bb7946e7}, classification~\citep{gadat2016classification,cannings2020local}, or regression problems~\citep{devroye1994strong,jiang2019non}, in which we typically need to let $k$ grow in a polynomial rate in the sample size in order to achieve a good balance in the bias-variance trade-off. This important difference in the rate of convergence, leading to a $k=1$ number of neighbor, has also been noticed in other estimation problems such as the $k$-NN entropy estimator~\citep{berrett2019efficient} or the integral approximation problem~\citep{leluc_nn,blanchet2024can}. Textbooks dealing with the $k$-NN algorithm include \citep{gyorfi2006distribution,devroye2013probabilistic,biau2015lectures}.

In addition of being optimal with respect to the estimation error, setting $k = 1$ circumvent the cumbersome hyper-parameter tuning while providing computational efficiency at the same time.
Our $1$-NN-based algorithm takes only a quasi-linear time $\mathcal{O}((n + m)\log n)$ on average using the optimized $k$-d tree~\citep{bentley1975multidimensional,friedman1977algorithm}. % In fact, \citet{friedman1977algorithm} showed that the time required to construct the $k$-d tree data structure is $\mathcal{O}(n\log n)$
Indeed, our experiments show that the proposed method terminates faster than previous methods, by large margins. Note that the problem of getting a computationally efficient method for covariate shift adaptation, in particular for scalability to large data sets, is a recurrent problem in the existing literature.
In fact, many existing methods resorted to implementation heuristics such as using a fixed number of kernel centers for reducing the computational burden at the cost of statistical guarantee~\citep{kanamori2009ulsif,sugiyama2007direct,sugiyama2008direct,yamada2013rulsif}.

Our method simulates the missing labels of the target sample, which in turn can be used for a variety of downstream supervised learning tasks.
Even though the main focus of this paper is the estimation of expectations in the target domain, for illustrating the usefulness of our method in a typical machine learning downstream task, we also demonstrate consistency properties of parametric M-estimators in the target domain. This is particularly useful when the parametric model is mispecified, since in this case, even the population minimizer changes when the covariate distribution is shifted.

 The problem of interest here is closely related to a well-known matching problem studied in the context of treatment effect estimation.
In particular, $k$-NN estimators have been used to estimate the so-called average treatment effect to tackle missingness of potential outcomes. See, e.g., \citet{Ros,abadie2006large}, for an error bound obtained in this specific instance of the problem.
In  Section~\ref{sec:related}, we discuss the main differences between the two problems and why their result is not generally applicable to ours.

In summary, the key contributions of this paper are the following.
(i) Our method is non-parametric. It does not introduce a model in covariate shift adaptation so that it will have a minimum impact on the model trained for the downstream task.
(ii) Our method is fast. Adaptation only takes a quasi-linear time.
(iii) There is no hyper-parameter to be tuned.
(iv) The proposed method only incurs an error of order  ${(k/n)^{1/d} + 1/\sqrt n + 1/\sqrt m}$ for estimating an expectation on the target domain.

The outline is as follows. In Section~\ref{sec:setup}, the problem of covariate shift adaptation is formally introduced along with the mathematical notation. Section~\ref{sec:method} contains the description of the method. Section~\ref{sec:theory} is dedicated to the main theoretical results while Section~\ref{submit} investigates the empirical risk minimization problem in presence of covariate shift adaptation. Section~\ref{sec:related} provides a description of several alternative approaches to a similar type of problem as well as some points of comparison with our proposal. In Section~\ref{sec:possible_extensions}, several avenues for further research are discussed and finally, the numerical experiments are provided in Section~\ref{sec:experiments}.

\section{Problem setup}\label{sec:setup}
Let $\mathcal X$ and $\mathcal Y$ be measurable spaces.
Let $P \equiv P_{X, Y}$ and $Q \equiv Q_{X, Y}$ be probability distributions defined on $\mathcal X \times \mathcal Y$.
Throughout the paper, we assume that $P$ and $Q$ admit the decomposition
\begin{equation*}
P  = P _{Y \given X} P _X \quad \text{and} \quad Q = Q_{Y  \given  X} Q_{X},
\end{equation*}
where $P _{Y \given X=x}$ and $Q_{Y \given X=x}$ are
probability distributions defined on $\mathcal Y$ for each $x \in \mathcal X$.\footnote{
More formally,
we denote by $P_{Y \given X = (\cdot)}(dy)$
a \emph{regular conditional measure}~\citep[Definition~10.4.1]{bogachev_measure_2007}
such that the marginal distribution of $Y$ can be expressed as $P_{Y}(dy) = \int P_{Y \given X = x}(dy) P_X(dx)$.
We also use $P_{Y\given X}(dy \given \cdot)$ for $P_{Y \given X = (\cdot)}(dy)$.
The same goes for $Q$.
}
Here, $P_X$ and $Q_{X}$ are the marginal distributions of $X$ when $(X, Y)$ is distributed with $P$ and $Q$, respectively.
We shall simply call $P_{Y \given X}$ (or $Q_{Y \given X}$) the \emph{conditional distribution} of $Y$ given $X$ in the source domain (or the target domain).
\begin{Def}[Source sample, source distribution]
  For each integer $n \geq 1$, let $(X_{i}, Y_{i})_{i=1}^{n}$ be a collection of independent and identically distributed random variables with $P$.
  We refer to $(X_{i}, Y_{i})_{i=1}^{n}$ as the (labeled) \emph{source sample} and $P$ as the \emph{source distribution}.
\end{Def}

\begin{Def}[Target sample, target distribution]
  For each integer $m \geq 1$, let $(X^{*}_{i})_{i=1}^{m}$ be a collection of independent and identically distributed random variables with $Q$.
  We refer to $(X^{*}_{i})_{i=1}^{m}$ as the (unlabeled) \emph{target sample} and $Q$ as the \emph{target distribution}.
\end{Def}

\begin{Def}[Covariate shift]
  \emph{Covariate shift} is a situation in which the source and the target distribution have different marginal distributions for $X$
  while sharing a common conditional distribution:
  \begin{enumerate}[label=(C\arabic*) , wide=0.5em,  leftmargin=*]
    \item \label{asmp:C} $P_{Y \given X} = Q_{Y \given X}$, $P_{X}$- and $Q_{X}$-a.s., but $P_{X} \neq Q_{X}$.
  \end{enumerate}
\end{Def}

This paper focuses on the following simple but versatile estimation problem under covariate shift.

\begin{Def}[Mean estimation under covariate shift]\label{def:mean_esti_covshift}
  For each pair of integers $n \geq 1$, $m \geq 1$, and a known integrable function $h \colon \mathcal X \times \mathcal Y \to \Re$, the goal of \emph{mean estimation under covariate shift} is to estimate the mean of $h$ under the target distribution,
  \[
  Q(h) \equiv \int h(x, y) Q(dx, dy),
  \]
  given access to the source sample $(X_{i}, Y_{i})_{i=1}^{n} \sim P$ and the target sample $(X^{*}_{i})_{i=1}^{m} \sim Q_{X}$ under Assumption~\ref{asmp:C}.
\end{Def}

For instance, when $h(x, y) = \ell(f(x), y)$ for a loss function $\ell\colon \mathcal{Y}^2 \to \Re$ and a hypothesis function $f\colon \mathcal{X} \to \mathcal{Y}$, estimation of $Q(h)$ becomes risk estimation, which is the central subtask in \emph{empirical risk minimization}.

\section{Proposed method}\label{sec:method}
The basic idea of our proposed method is to use the source sample for learning to \emph{label the target data}.
Specifically, using the source sample $(X_{i}, Y_{i})_{i=1}^{n}$, we will construct a stochastic labeling function $\hat{\texttt S}$ that inputs any target data point $X^{*}_{i}$ and outputs a random label $Y^{*}_{n, i} \in \mathcal Y$. (The subscript $n$ of $Y^{*}_{n, i}$ is for explicitly denoting the dependence on the source sample.)
Once we succeed in generating labels for target data that behave like true target labels, we will be able to perform any supervised learning method \emph{directly on the target sample} for the downstream task.
For our mean estimation problem, we can simply average the output $h$ evaluated at the target data with the generated labels.

When do the generated labels behave like the true target labels?
Let $\hat P_{Y \given X^{*}_{i}}$ denote the probability distribution of an output $Y^{*}_{n, i}$ of $\hat{\texttt S}$ for input $X^{*}_{i}$.
We wish to obtain $\hat{\texttt S} $ such that the probability distribution $\hat Q \equiv \hat P _{Y \given X} Q_X$ of $(X^{*}_{i}, Y^{*}_{n, i})$ will be a good estimate of $Q  =  Q _{Y \given X} Q_X$.
For this, we want $\hat P_{Y \given X^{*}_{i}}$ to be a good estimate of $P_{Y \given X^{*}_{i}}$.
In fact, if $\hat P_{Y \given X^{*}_{i}} = P_{Y \given X^{*}_{i}}$,
the generated sample $(X^{*}_{i}, Y^{*}_{n, i})$ will follow the target distribution $Q \equiv Q_{Y \given X} Q_{X} = P_{Y \given X} Q_{X}$
under Assumption~\ref{asmp:C}.
In this sense, our task boils down to designing a good conditional sampler $\hat{\texttt S}$ mimicking sampling from $P_{Y|X}$. Algorithm~\ref{alg:general_condisample_covshift} describes an outline of this general framework.

\begin{algorithm}[htb]
   \caption{Conditional Sampling  Adaptation}
   \label{alg:general_condisample_covshift}
\begin{algorithmic}
   \State \textbf{Input:} Conditional sampler $\hat{\texttt S}$ and target sample $(X^{*}_{j})_{j=1}^{m}$.
   \State $Y^{*}_{n, j} \gets \hat{\texttt S}(X_{j}^{*})$ for each $j \in \{1, \dots, m\}$. // Generate a label conditioned on $X_{j}^{*}$.
   \State \Return $m^{-1} \sum_{j=1}^{m} h(X^{*}_{j}, Y^{*}_{n, j})$.
\end{algorithmic}
\end{algorithm}

In this paper, we propose a method using a non-parametric conditional sampler $\hat{\texttt S}$ based on the $k$-Nearest Neighbor ($k$-NN) method,
which randomly picks one of the $k$-nearest neighbors of the input $X^{*}_{j}$ among the source instances $(X_{i})_{i=1}^n$ and output the corresponding label (Algorithm~\ref{alg:kNNcondi}).
We refer to this method as \emph{$k$-NN-based Conditional Sampling Adaptation} (\emph{$k$-NN-CSA}).

\begin{algorithm}[htb]
   \caption{$k$-Nearest Neighbor Conditional Sampler}
   \label{alg:kNNcondi}
\begin{algorithmic}
   \State \textbf{Input:} Source sample $(X_{i}, Y_{i})_{i=1}^{n}$ and target input $X^{*}_{j}$.
   \State $(i_{1}, \dots, i_{k}) \gets \text{the indices of the $k$-nearest neighbors of $X^{*}_{j}$ among the source instances $(X_{i})_{i=1}^n$}$.%\footnote{We assume that the $k$ nearest neighbors are unique with probability one.}
   \State Pick $i^{*} \in \{i_{1}, \dots, i_{k}\}$ uniformly at random.
   \State \Return $Y^{*}_{n, j} := Y^{*}_{i^{*}}$.
\end{algorithmic}
\end{algorithm}

\paragraph{Computing time} Recent advances for nearest neighbor search rely on tree-search to reduce the computing time.
The seminal paper by \citet{bentley1975multidimensional} introduced the $k$-d tree method.
Building such a tree
requires $\mathcal{O}(n\log n)$
and once the tree is available, search for the nearest neighbor of a given point can be done in $\mathcal{O}(\log n)$ time~\citep{friedman1977algorithm}.
As a consequence, the time complexity of $k$-NN-CSA is $\mathcal{O}(n\log n + km\log n)$.

\section{Theoretical analysis}\label{sec:theory}

We now present the theory behind our approach in a didactic way by introducing a key decomposition first and then studying separately each of the terms involved: the sampling error and the $k$-NN conditional sampling error. We will see that the $k$-NN-CSA with $k = 1$ ($1$-NN-CSA for short) achieves the best theoretical performance among those with other $k$'s.

\subsection{The key decomposition}

For the analysis of $k$-NN-CSA, recall that $\hat Q = \hat P_{Y|X} Q_X$ is an estimate of the target distribution $Q = P_{Y|X} Q_X$ that depends on the source sample $(X_i,Y_i)_{i=1} ^ n $, whose probability distribution is $P$.  We introduce the bootstrap sample as a collection of random variable generated according to $\hat Q$.

\begin{Def}[Bootstrap sample]
For each $m\geq 1$ and  $n\geq 1$, let $(X_{i}^*, Y_{n,i}^*)_{1\leq i\leq m}$ be a collection of random variables identically distributed with $\hat Q$ and conditionally independent given $(X_{i}, Y_{i})_{i=1}^{n}$.
\end{Def}

Let $h : \mathcal X\times \mathcal Y \to \mathbb R$ be a measurable function. The quantity of interest is
\begin{equation*}
  \hat Q ^* ( h) = m^{-1} \sum_{i=1} ^m h(X_{i}^*, Y_{n,i}^*),
\end{equation*}
which is the CSA estimate of $  Q ( h) =  \int h(x,y) Q (dx,dy)$ as introduced in Algorithm~\ref{alg:general_condisample_covshift}. 
The following decomposition is crucial in our analysis:
\begin{align}
(\hat Q ^* -  Q ) (h)
&=  \underbrace{( \hat Q ^* -   \hat Q ) (  h)}_{\text{\emph{Marginal sampling error}}}   +  \underbrace{(\hat Q  -  Q ) ( h )}_{\text{\emph{Conditional sampling error}}}\label{eq:decomposition}\\
&\bigg( = (\hat Q^*_X - Q_X)\hat P_{Y \given X}(h) + Q_X(\hat P_{Y \given X} - P_{Y \given X})(h) \bigg),\nonumber
\end{align}
where $\hat Q^*_X(\cdot) \equiv \frac{1}{n}\sum_{i=1}^m \mathds{1}_{X^*_i = (\cdot)}$ is the empirical measure defined with $(X^*_i)_{i=1}^m$.
The first term is the error due to the use of $\hat{Q}^*_X$ in place of $Q_X$, which tends to zero as $m$ grows.
The second term represents the error due to the use of $\hat{P}_{Y \given X}$ in place of $P_{Y \given X}$. When using the $k$-nearest neighbor algorithm to obtain $\hat P_{Y \given X}$, we show that this term is of order $(k/n)^{1/d} + 1/\sqrt n $, which differs from the standard non-parametric convergence rate in $(k/n)^{1/d} + 1/\sqrt k $ found in regression problems.

\subsection{Marginal sampling error}\label{sec:sampling}

First, we will show that the marginal sampling error, $( \hat Q ^* -   \hat Q ) (  h) $ (the first term in our decomposition~\eqref{eq:decomposition}), is of order $1/\sqrt m$.  The analysis relies on martingale tools.
 Define $\mathcal F_n = \sigma ( (X_1,Y_1),\ldots, (X_n,Y_n) )$. For each $1\leq i\leq m$, we have
$$ \E [ h( X_{i}^* , Y_{n,i} ^* ) \given \mathcal F_n ] =  \hat Q  ( h )   .$$
This property implies that $\sum_{i=1} ^ m \{ h( X_{i}^* , Y_{n,i} ^* ) - \int  h (x, y) \, \hat Q ( d x, d y )  \}$ is a martingale and therefore can be analyzed using the Lindeberg-CLT conditionally on the initial sample hence fixing the distribution  $\hat Q$. The next property is reminiscent of certain results about the bootstrap method where sampling is done with the basic empirical measure, see e.g., \cite{van2000asymptotic}. We need this type of results without specifying the measure $\hat Q$ so that we can incorporate a variety of sampling schemes such as $\hat Q = \hat P_{Y \given X} Q_X$. The proof is given in Appendix~\ref{app:sec:slln_clt}.

\begin{prop}\label{prop:slln_clt}
Suppose that $\hat Q$ satisfies the following strong law of large number: for each $h$ such that $Q(h) <\infty$, we have $\lim_{n \to \infty} \hat Q(h) = Q(h)$ almost surely. Then, if $m:= m_n \to \infty$ as $n\to \infty$, we have the following central limit theorem: for each function such that $  Q(h^2)<\infty$,
we have, conditionally to $\mathcal F_n $, almost surely,
\begin{equation*}
 \sqrt m \{ \hat Q^* (h)  -  \hat Q  ( h)  \} \leadsto \mathcal N ( 0 , V) \qquad \text{as } n \to \infty,
\end{equation*}
where $V = \lim_{n \to \infty} \{ \hat Q (h^2) - \hat Q (h)^2 \}$.
\end{prop}

As a corollary of the previous results, we can already deduce that if $m $ goes to $\infty$ and $\hat Q$ satisfies a strong law of large numbers,  then $\hat Q^* (h ) $ converges to $ Q(h)  $ provided that $Q(h^2) $ exists. This is a general consistency result that justifies the use of any resampling distribution $\hat Q$ that converges to $Q$.  In practical situations, it is useful to know a finite-sample bound on the error. This is the purpose of the next proposition, in which we give a non-asymptotic control of the sampling error. A proof is given in Appendix~\ref{app:sec:expsampling_proof}.

\begin{prop}\label{expsampling}
Suppose that $h$ is bounded by a constant $U_h>0$.
Let $\delta\in (0,1)$. Then with probability greater than $1-\delta$,
\begin{equation*} 
\left| \hat{Q}^{*}(h)-\hat{Q}(h) \right|\leq \frac{U_h}{m}\log(2/\delta)+\sqrt{2\frac{\hat v_n }{m}\log(2/\delta)},
\end{equation*}
where $\hat v_n = \hat{Q}(h^2) - (\hat{Q}h)^2$.
\end{prop}

\paragraph{Notes.} A natural ``averaging'' alternative to the above ``sampling'' estimator $\hat Q^*$ can also be investigated using the same tools. Instead of sampling $Y_{n,i}^*$ according to $\hat P_{n} (dy|X_i^*)$, one might consider taking the expectation, leading to
$$ \overline Q (h) = \frac{1 } {m}  \sum_{i= 1 } ^m \int h (X_i^*, y ) \hat P_n (dy | X_i^*) .  $$
This estimate can be studied in a similar way as before and the two above results are still valid with small changes. In particular Proposition~\ref{expsampling} holds true with smaller variance term as, by Jensen's inequality,
$Q_X( \int h (X_i^*, y ) \hat P_n (dy | X_i^*) ^2 ) \leq \hat Q ( h^2)  $. This alternative $\overline Q (h)$
requires more computing time (when measured in terms of evaluation of $h$) and is less appealing for stochastic gradient descent algorithm or in semiparametric estimation problems, as discussed in Section~\ref{sec:possible_extensions}.
Estimators similar to $\overline Q(h)$ have been studied in average treatment effects literature~\citep{Ros,abadie2006large}; see Section~\ref{sec:related} for precise discussion.

\subsection{Conditional sampling error of the nearest neighbor estimate}\label{sec:nn}

 Our aim in this  section is to obtain a bound on $\hat Q(h) - Q(h)$ (the second term in our decomposition~\eqref{eq:decomposition}) when $\hat P_{Y|X}$ is the $k$-nearest neighbor measure.

Let $x\in \mathbb R^d$ and $\norm{\cdot}$ be the Euclidean norm on $\mathbb R^d$. Denote the closed ball of radius $\tau \ge 0$ around $x$ by $ B(x,\tau) := \{z \in \mathbb R^{d} \mid \norm{x-z} \leq \tau\}$. For $n \geq 1$ and $k \in \{1,\; \ldots,\; n\}$, the $k$-nearest neighbor ($k$-NN for short)  radius at $x$ is denoted by  $\hat \tau_{n,k,x} $ and defined as the smallest radius $\tau \geq 0$ such that the ball $B(x, \tau)$ contains at least $k$ points from the collection $\{X_1, \ldots, X_n\}$. That is,
\begin{align*}
\hat \tau_{n,k,x} : =  \inf   \left\{ \tau\geq 0 \, :\,  \sum_{i=1}^n 1 _{ B(x,\tau) }(X_i) \geq k \right\},
\end{align*}
where $1 _A (x)$ is $ 1$ if $x\in A$ and $0$ elsewhere. The $k$-NN estimate of $  P _{Y \given X} (dy  \given  x )$ is given by
\begin{equation*}
\hat{P}_{Y \given X} (dy \given x )=  k^{-1} \sum_{i=1}^{n} 1_ { \norm{X_i-x} \leq \hat \tau_{n,k,x} }  \delta_{Y_i}(dy),
\end{equation*}
where $\delta_{y}(\cdot)$ is the Dirac measure at $y \in \mathcal{Y}$ defined by $\delta_{y}(A) = 1_{A}(y)$ for any measurable set $A \subseteq \mathcal{Y}$.
Consequently, the $k$-NN estimate of the integral  $\int h(y,x)   P _{Y \given X} (dy  \given  x)$ is then defined as
\begin{equation*}
\int h(y,x)  \hat P _{Y \given X} (dy  \given   x ) =  k^{-1} \sum_{i=1}^{n} 1_ { \| X_i-x\| \leq \hat \tau_{n,k,x} }   h(Y_i,x).
\end{equation*}

To obtain some guarantee on the behavior of the nearest neighbors estimate, we consider the case in which covariates $X$ admit a density with respect to the Lebesgue measure. We will need in addition that the support $S_X$ is well shaped and that the density is lower bounded. These are standard regularity conditions to obtain some upper bound on the $k$-NN radius.
\begin{enumerate}[label=(X\arabic*), wide=0.5em, leftmargin=*]
  \item \label{cond:reg0} The random variable $X$ admits a density $p_X$ with compact support $S_X \subset \mathbb R^d $.
  \item \label{cond:reg1} There is $c>0$ and $T>0$ such that
\begin{align*}
&\lambda (S_X \cap  B(x, \tau ) ) \geq c \lambda   ( B(x, \tau )) , \qquad \forall \tau \in (0,T] , \, \forall x\in S_X,
\end{align*}
where $\lambda$ is the Lebesgue measure.
\item \label{cond:reg2} There is $0 < b_X\leq U_X <+\infty$ such that  $b_X\leq p_{X}(x) \leq U_X$, for all $ x \in S_X$.
\end{enumerate}

To obtain our main result, on the estimation property of the $k$-NN measure, we need some assumptions on the target measure $Q_X$.

\begin{enumerate}[resume,label=(X\arabic*) , wide=0.5em,  leftmargin=*]
  \item  \label{cond:reg3} The probability measure $Q_X$ admits a bounded density $q_X$ with support $S_X$.
  We will take $U_X$ large enough such that it will also be an upper bound of $q_{X}$.
\end{enumerate}

Two additional assumptions, different from the one before about $X$, will be needed to deal with the function $h$ and the probability distribution of $(Y,X)$.

\begin{enumerate}[label=(H\arabic*) , wide=0.5em,  leftmargin=*]
  \item  \label{cond:reg4}
For any $x$ in $S_X$,
\begin{equation*}
  \abs{ \E[h(Y,x) \given X = x] - \E[h(Y,x) \given X = x + u] } \leq g_h(x) \norm{u}
\end{equation*}
with $ \int  g_h^2 (x)  Q_X(dx) < \infty$.
\item \label{cond:reg5} There exists $\sigma_+^2>0$ such that
 $\sup_{x\in S_X}\var( h(Y,x) \given X) \leq \sigma_+^2 \text{ a.s.}$, where $\var(h(Y, x) \given X)$ is the conditional variance of $h(Y, x)$ given $X$.
\end{enumerate}

In what follows, we give a control of the RMSE of $\hat Q h$. Let $\Vert X\Vert_2=\sqrt{\E\left(X^2\right)}$, $[x]$ the integer part of a real number $x$ and let $\Gamma(x) := \int_0^{\infty}u^{x-1}\exp(-u)du$ for $x>0$. Finally, we denote by $V_d := \lambda(B(0, 1)) = \frac{\pi^{d/2}}{\Gamma(d/2 - 1)}$ the volume of the unit Euclidean ball in dimension $d$ for the Lebesque measure.

We give an upper-bound for the RMSE with explicit constants with respect to the dimension $d$. Additionally, we give a lower bound for the variance which has a standard parametric rate of convergence. The proof is given in Appendix~\ref{app:sec:lionel_prop_proof}.

\begin{prop}\label{lionel_prop}

Suppose that Assumptions \ref{cond:reg0}, \ref{cond:reg1}, \ref{cond:reg2}, \ref{cond:reg3}, \ref{cond:reg4}, and \ref{cond:reg5} are fulfilled. We have
\[
\hat{Q}h-Qh  =  S_h  +  B_h,
\]
where  $ B_h$ is a bias term (defined in the proof) that satisfies, for any $n\geq 1$,
\[
\mathbb E |B_h|^2   \leq\frac{2\Gamma\left(1+[2/d]\right)}{M_{1,d}^{2/d}}\int g_h^2(x) Q_X(dx)\cdot\frac{k^{2/d}}{n^{2/d}},
\]
 and $ S_h $ is a variance term (defined in the proof) that satisfies, for any $n \geq 2$,
\begin{equation*}
 \frac{\sigma_{-}^2M_{1,d}^2}{4 M_{2,d}^2} n^{-1} \leq \min_ {1\leq k \leq n} \mathbb E\left[S_h ^2\right] \leq \max_{1\leq k \leq n}  \E\left[S_h^2\right] \leq\frac{2^{d+3}\sigma_+^2M_{2,d}^2}{M_{1,d}^2} n^{-1}.
\end{equation*}
For the lower bound to be true,  it is assumed that the mapping $h$ does not depend on $x$, i.e. $h(y,x) = h(y)$ and $\sigma_{-}^2=\inf_{x\in S_X}\var\left(h(Y)\vert X=x\right)$.
\end{prop}

\paragraph{Notes.}

\noindent (i) The two terms $  B_h$ and $  S_h$ correspond respectively to the bias term and the variance term.
The upper bound obtained for the bias term is usual in $k$-NN regression analysis. However, the upper and lower bound on the variance are particular to our framework as they show that the variance behaves as in usual parametric estimation. Consequently, our rates of convergence are sharper than the optimal rate of convergence $n^{-\frac{1}{2+d}}$ for nonparametric estimation of Lipschitz functions. This can be explained by the fact that several $k$-NN estimators are averaged to estimate $Qh$, which is a standard expectation and not a conditional expectation.

\noindent (ii)
Since the rate of convergence of the variance term $S_h$ does not depend on $k$, $k$ might be chosen according to the upper bound on the bias term, which gives $k = 1$. One can deduce the following convergence rates, depending on the dimension.
For $d=1$, we get the rate $n^{-1/2}$. For $d=2$, the contributions of both terms, $B_h$ and $S_h$,  coincide and we get the rate $n^{-1/2}$. For $d\geq 3$, the rate is $n^{-1/d}$.

For the global mean square error which incorporate the marginal sampling error as well as the $k$-NN conditional sampling error, we give the following result in the optimal case $k=1$. The proof can be found in Appendix~\ref{app:sec:global}.

\begin{theo}\label{global}
Suppose that Assumptions \ref{cond:reg0}, \ref{cond:reg1}, \ref{cond:reg2}, \ref{cond:reg3} and \ref{cond:reg4} hold true with $\sup_{x\in S_X}\E\left[h^2(Y,x)\vert X\right]$ bounded. If $k=1$, there then exists $C>0$ only depending on the distribution of $(X,Y)$, $X^{*}$, and on $h$ such that 
$$\E\left\vert \hat{Q}^{*}(h)-Q(h)\right\vert^2\leq C\left\{\frac{1}{m}+\frac{1}{n^{2/d}}+\frac{1}{n}\right\}.$$
\end{theo}

We next give a non-asymptotic control of $\hat{Q}h-Q h$ when $h$ is a bounded function using Bernstein's concentration inequality. This bound affords a complement with respect to the bound for the MSE. However, for technical reasons, this high-probability bound requires that $k$ grows at least logarithmically with respect to $n$, in contrast to Proposition~\ref{lionel_prop}. In our numerical experiments, we will also include the case $k=\log n$ for comparison.
The proof of the next result is given in Appendix~\ref{app:sec:portier+_proof}.

\begin{prop}\label{Portier+}
Suppose that Assumptions \ref{cond:reg0}, \ref{cond:reg1}, \ref{cond:reg2}, \ref{cond:reg3}, \ref{cond:reg4}, and \ref{cond:reg5} are fulfilled.
Suppose that there exists a constant $C>0$ such that $C\log n\leq k\leq n/2$ and that $h$ is bounded by $U_h$.
Let $\delta\in (0,1/3)$. With probability greater than $1-3\delta$, we have 
$$\left\vert \hat{Q}h-Q h\right\vert\leq  L_0 \left(\frac{k}{n}\right)^{1/d}+\frac{2L_2}{n}\log(2/\delta)+\sqrt{\frac{2L_1\sigma^2}{n}\log(2/\delta)},$$
where
$$L_0=\left(\frac{2}{cb_x V_d}\right)^{1/d}\int g_h(x)Q_X(dx),\quad L_1=\frac{4\sigma_+^2U_X^2}{b_X^2c^2},\quad L_2=\frac{4U_h U_X}{3 b_X c}.$$
\end{prop}

\paragraph{Notes.} 

 \noindent (i) The proof needs a bound on $\sup_{x\in S_X}\hat{\tau}_{n,k,x}$ which is given in
\citep[Lemma 4]{portier2021nearest}. 
For this we need that $k$ grows logarithmically w.r.t. $n$ as stated in the assumptions.

\noindent (ii)  The sum of the two last terms in the upper-bound, which corresponds to the variance of our estimator, is of order $1/\sqrt{n}$ and the conditional variance of $h$ appears as a multiplicative factor. Combining Proposition~\ref{expsampling} and Proposition~\ref{Portier+}, we finally obtain that $\hat Q^* (h) - Q(h) $ is of order $1/\sqrt m + 1/ \sqrt n + (k/n ) ^{1/d}$ (up to log factors).

\noindent (iii) In Corollary~\ref{mainpoint} in Appendix~\ref{app:sec:cor_kNN_error_bound}, the upper bound given in Proposition~\ref{expsampling} is refined using Proposition~\ref{Portier+} that allows to (roughly speaking) replace the empirical variance by the true variance.

\section{Applications to empirical risk minimization}\label{submit}

In this section, we illustrate our results with some applications to empirical risk minimization.  This is of particular interest in our context as the optimal linear model for the source distribution might be different from the ideal linear model for the target. In such a case, using covariate adaptation is always better as the source minimizer will be away from the target minimizer. 

\subsection{Mathematical background}

Suppose that $\mathcal{R}^{*}_{m,n}(\theta)=\frac{1}{m}\sum_{i=1}^m m_{\theta}\left({Y}^{*}_{n,i},X_i^{*}\right)$, where
for each $\theta\in\Theta\subset \R^d$, $m_{\theta}$ is a measurable function from $\R\times \R^p$ to $\R$. 
Set 
$$\hat{\theta}^{*}\in \arg\min_{\theta\in\Theta}\mathcal{R}^{*}_{m,n}(\theta).$$
%We also denote by $\mathcal{E}\left(\hat{\theta}^{*}\right)=\mathcal{R}^{*}\left(\hat{\theta}^{*}\right)-\mathcal{R}^{*}\left(\theta^{*}\right)$ the excess risk, 
Similarly, we define
$$\theta^{*}:=\arg\min_{\theta\in\Theta}\mathcal{R}^{*}\left(\theta\right)$$
with $\mathcal{R}^{*}(\theta)=\E m_{\theta}\left(Y^{*},X^{*}\right)$ and $(Y^{*},X^{*})$ is a copy of $(Y_i^{*},X_i^{*})$. Note that the expected value is taken for the unobserved label $Y_i^{*}$ and not the generated label ${Y}_{n,i}^{*}$.
We assume here that for a reference measure $\mu$ on $\mathcal{Y}$, there exists for each $x\in\mathcal{X}$ a conditional density $p(\cdot\vert x)$ such that $(x,y)\mapsto p(y\vert x)$ is jointly measurable and for any Borel set $B$, 
$$\P\left(Y\in B\vert X=x\right)=\int_B p(y\vert x)\mu(dy).$$
One can then include the case of classification ($\mu=\delta_0+\delta_1$), counts ($\mu$ is the counting measure on the set of nonnegative integers)
or regression ($\mu$ is the Lebesgue measure on $\R$).

\subsection{Consistency of general empirical risk minimizers}

We will use the following assumptions. 

\begin{enumerate}[label=(A\arabic*) , wide=0.5em,  leftmargin=*]
  \item  \label{cond:reg_wcv1}
There exist a measurable function $h:\R\rightarrow \R_+$ and $\eta:\R_+\rightarrow \R_+$ satisfying
\begin{align*}
&\sup_{x\in\mathcal{X}}\sup_{\theta\in\Theta}\left\vert m_{\theta}(y,x)\right\vert\leq h(y),\\
&\sup_{x\in\mathcal{X}}\sup_{\vert \theta-\theta'\vert\leq \delta}\left\vert m_{\theta}(y,x)-m_{\theta'}(y,x)\right\vert\leq h(y)\eta(\delta),
\end{align*}
and such that $\E\left[h(Y)^2\vert X\right]$ is a bounded random variable and $\lim_{\delta \rightarrow 0}\eta(\delta)=0$.
\item \label{cond:reg_wcv2}
There exists a measurable function $g_h:\mathcal{X}\rightarrow \R_+$ such that $\int g_h(x) Q_X(dx)<\infty$ and 
$$\int h(y)\left\vert p(y\vert x+u)-p(y\vert x)\right\vert \mu(dy)\leq g_h(x)\vert u\vert,\quad (x,x+u)\in S_X^2.$$
\end{enumerate}

The above assumptions are satisfied, for instance, in the logistic regression framework with compact covariates. In this case, $h$ is a constant function and $\eta(\delta)=\delta$. Note also that $p(1\vert x)$ could be different form $\left(1+\exp\left(-x^T\theta\right)\right)^{-1}$ as soon as $x\mapsto p(1\vert x)\in (0,1)$ is Lipschitz on $S_X$.

In what follows, an assertion of the form $X_{m,n}=o_{\P}(1)$ as $m,n\rightarrow \infty$ means that for any $\epsilon,\zeta >0$, there exists $A>0$ such that 
$$\min(m,n)\geq A \Rightarrow \P\left(\vert X_{m,n}\vert>\epsilon\right)\leq \zeta.$$
Additionally, the assertion $X_{m,n}=O_{\P}(1)$ means that for any $\varepsilon>0$, there exist $A,M>0$ such that
$$\sup_{m,n\geq A}\P\left(\vert X_{m,n}\vert>M\right)\leq \varepsilon.$$
The proof of the following result is in Appendix~\ref{app:sec:notsomuch1_proof}. 

\begin{theo}\label{notsomuch1}
Suppose that Assumptions \ref{cond:reg0}, \ref{cond:reg1}, \ref{cond:reg2}, \ref{cond:reg3}, and \ref{cond:reg_wcv1}, \ref{cond:reg_wcv2} hold true with a compact subset $\Theta$ of $\R^d$ and the unique minimizer $\theta^{*}$ of $M$. Then $\hat{\theta}^{*}-\theta^{*}= o_{\P}(1)$ as $m,n\rightarrow \infty$.
Moreover, the excess risk satisfies $
%\mathcal{E}\left(\hat{\theta}^{*}\right)=
\mathcal{R}^{*}(\hat{\theta}^{*})-\mathcal{R}^{*}(\theta^{*})=o_{\P}(1)$.
\end{theo}

\subsection{Convergence rate for linear least-squares estimators}

We now illustrate our results with an upper-bound on the excess risk for linear least-squares estimators in the misspecified case. Here, the targeted risk is given by
$$\mathcal{R}^{*}(\theta)=\E\left[\left( Y^{*}-{X^{*}}^T\theta\right)^2\right]$$
and any optimal linear rule should simply be satisfied:
$$\theta^{*}\in \arg\min_{\theta \in \R^d}\mathcal{R}^{*}(\theta).$$
 Note that $\theta^*$ is unique  the matrix $\E\left[X^{*}{X^*}^T\right]$ is of full rank.
The empirical risk is defined by
$$\mathcal{R}^{*}_m(\theta)=\frac{1}{m}\sum_{i=1}^m \left({Y}^{*}_{n,i}-{X_i^{*}}^T\theta\right)^2$$
and $\hat{\theta}^{*}$, the empirical risk minimizer, is given by 
$$\hat{\theta}^{*} \in  \arg\min_{\theta \in \R^d}\mathcal{R}^{*}_m (\theta)  .
$$
The excess risk satisfies the following upper bound whose proof is given in Appendix~\ref{app:sec:notsomuch1_proof}. 

\begin{theo}\label{notsomuch}
Suppose that Assumptions \ref{cond:reg0}, \ref{cond:reg1}, \ref{cond:reg2}, \ref{cond:reg3} hold true. Suppose that the mapping $x\mapsto \E\left[Y\vert X=x\right]$ is Lipschitz and that the conditional expectation $\E[Y^2\vert X]$ is bounded. Suppose also that $\Gamma=\E[X^{*}{X^{*}}^T]$ is positive definite. Then, we have 
%$$\lim_{M\rightarrow \infty}\sup_{m,n\geq 1}\P\left(\sqrt{r_{m,n}}\Vert \hat{\theta}^{*}-\theta^{*}\Vert>M\right)=0.$$
%In particular, taking $\Vert x\Vert=\sqrt{x^T \Gamma x}$, we get 
$$  \mathcal{R}^{*}(\hat{\theta}^*) -  \mathcal{R}^{*}(\theta^*)=O_{\P}\left(
m^{-1} + n^{-1} + n^{-2/d}
\right).$$
\end{theo}

\paragraph{Notes.}  The assumptions do not require the linear model for the $(X_i,Y_i)$'s to be valid, i.e., one can consider cases where $ \E [Y|X] $ is not linear.
Also, when the source data follows a non-linear model of the form $Y=r\left(X\right)+\varepsilon$ where 
$\varepsilon$ and $X$ are independent, our regularity assumptions means that $r$ is Lipschitz on the compact set $S_X$.

% \section{Related work and possible extensions}\label{sec:related}
\section{Related work}\label{sec:related}

A standard approach to covariate shift problems is to use some re-weighting in order to ``transfer'' the source distribution with density $p_X$ to the target distribution with density $q_X$. This approach relies on the following type of estimates:
\begin{equation*}
\hat Q_w(h) = n^{-1} \sum_{i=1} ^n  w(X_i) h(X_i, Y_i) ,
\end{equation*}
where ideally the function $w$ would  take the form $q_X/p_X$. Such a choice has the nice property that the expected value $\E [ w(X_i) h(X_i, Y_i)  ] $ is equal to the targeted quantity $Q(h)$. This however cannot be directly computed as $p_X$ and $q_X$ are unknown in practice. There are actually different ways to estimate $w$, and our goal here is to  distinguish between two leading approaches.

\paragraph{Plug-in approach}
The plug-in approach is when the weights are computed using two estimates $\hat p_X$ and $\hat q_X$ in place of $p_X$ and $q_X$, respectively; i.e., simply use $\hat w = \hat q_X/ \hat p_X$ instead $w$ in the above formula, see for instance \citep{shimodaira2000improving,sugiyama2007direct,sugiyama2008direct}. Note that the selection of hyper-parameters for $\hat q_X$ and $\hat p_X$ is needed and the $n$ evaluation $(\hat q_X (X_i) , \hat p_X(X_i)) $ might be heavy in terms of computing time.

For the sake of clarity, we focus on a specific instance of covariate shift problem in which the target probability density $q_X$ is known and ${\hat p}$ is the kernel density estimate (KDE), i.e., 
$ \hat p_X ^{KDE} (x) = (1/n) \sum_{i=1} ^n K_b (x-X_i)$,  where $K_b$ typically is a Gaussian density with mean $0$ and variance $b^2$ (a hyper-parameter to be tuned). Note that such a situation does not involve any changes for our sampling procedure whereas it is clearly advantageous for the weighted approach for which one unknown, $q_X$, is now given.
 In this case, the analysis of $\hat Q_w(h)$ can be carried out using the decomposition
$ \hat Q_{\hat w}(h) -Q(h) =   \hat Q_{\hat w}(h- Th) +   \hat Q_{\hat w}( Th ) -  Q(h)$,
 with $Th (x) = \E [ h(X,Y) |X = x]$. The first term above is a sum of centered random variable which (provided some conditions) satisfies the so-called  Lindeberg condition so that the central limit theorem implies that $ \sqrt n   \hat Q_{\hat w}(h- Th) $ is asymptotically Gaussian. The second term above is more complicated and the analysis can be derived using results in  \cite{delyon2016integral,clemenccon2018beating}. Those results assert (under some conditions) that $ \hat Q_{\hat w}( Th ) -  Q(h) = O_p (nb^d + b )$ (in case $p_X$ is Lipschitz ). As a consequence, we obtain, optimizing over $b$, that
$ \hat Q_{\hat w}(h) -Q(h) =   O_p(n^{-1/2} + n^{-1/(1+d)} )$. 
 This is easily compared to our bound, when $q_X$ is known, $ n^{-1/2} + n ^{-1/d}$, which is smaller than the one given before.

\paragraph{Direct weight estimation}

\citet{huang2006kmm} proposed \emph{Kernel Mean Matching (KMM)} for estimating the ratios of the probability density functions of the source and the target distribution. They used the estimated ratios for weighting the source sample.
\citet{gretton2008kmm} further studied this method theoretically and empirically.
\citet{sugiyama2007direct,sugiyama2008direct} proposed a method that estimates the ratios as a function by minimizing the Kullback-Leibler divergence between the source density function multiplied by the ratio function and the target density function.
The estimated function can predict ratios even outside of the source sample, which enables cross-validation for hyper-parameter tuning.
\citet{kanamori2009ulsif} proposed constrained and unconstrained least squares methods for estimating the ratio function
called \emph{Least-Squares Importance Fitting (LSIF)} and \emph{unconstrained LSIF (uLSIF)}.
\citet{yamada2013rulsif} developed its variant called \emph{Relative uLSIF (RuLISF)}, which replaces the denominator of the ratio with a convex mixture of the source and the target density functions to circumvent issues caused by near-zero denominators.
\citet{zhang_one-step_2021} proposed a covariate shift adaptation method that directly minimizes an upper bound of the target risk in order to avoid estimation of weights.
The method shows great empirical performance while it does not exactly minimize the target risk and hence the minimizer converges to a biased solution.

\paragraph{Connection to treatment effect estimation}\label{app:covshift_and_treatment}
% {Connection between covariate shift adaptation and treatment effect estimation}\label{app:covshift_and_treatment}
% \paragraph{Treatment effect with covariates}
% FP: I am not sure yet but if the treatment effect literature 
% has focused on $ h(Y)$ and not $h(Y,X)$, I would say that our resutls are suitable to the analysis of certain covariates within the treatment effect framework.
% FP: I like the way it is done in the appendix. I suggest to take it and remove it around here.

\newcommand\independent{\protect\mathpalette{\protect\independenT}{\perp}}
\def\independenT#1#2{\mathrel{\rlap{$#1#2$}\mkern2mu{#1#2}}}
One of the quantities of great interest in treatment effect estimation is the average treatment effect on the treated (ATT), $\E[Y^{(1)} - Y^{(0)} \given W = 1]$,
where $W \in \{0, 1\}$ is a treatment assignment variable, $Y^{(1)}$ and $Y^{(0)}$ are potential outcomes corresponding to the treatment $1$ and $0$.\footnote{A common scenario is that we have a treated group (represented by treatment 1) and a non-treated, or controlled group (represented by treatment 0).}
Suppose that we wish to estimate the ATT using i.i.d.\@ observations of $W$ and its outcome $Y := Y^{(W)}$ together with covariates $X$, $\{(Y_i, W_i, X_i)\}_{i=1}^N$.
Under the standard assumptions (see e.g., \citet{hernan2023causal}) including the \emph{conditional exchangeability} $Y^{(w)} \independent W \given X$, the \emph{positivity} $P(\{W = w \given X\}) > 0$, and the \emph{consistency} $W = w \implies Y^{(w)} = Y^{(W)} = Y$, for each $w \in \{0, 1\}$, the ATT equals the difference between
\begin{align}
\E[Y^{(1)} \given W = 1]
&= \int y P_{Y^{(1)} \given X, W = 1}(dy) P_{X \given W = 1}(dx)\nonumber\\
&= \int y P_{Y \given X, W = 1}(dy) P_{X \given W = 1}(dx)\nonumber\\
&= \E[Y \given W = 1] \label{eq:ygivenw1}
\end{align}
and
\begin{align}
\E[Y^{(0)} \given W = 1]
&= \int y P_{Y^{(0)} \given X, W = 0}(dy) P_{X \given W = 1}(dx)\nonumber\\
&= \int y P_{Y \given X, W = 0}(dy) \, r(x) \, P_{X \given W = 0}(dx) \nonumber\\
&= \E[ r(X) Y \given W = 0],\label{eq:yrxgivenw1}
\end{align}
where $r(x)$  is the density ratio defined such that $ r(x) {dP_{X \given W = 0}}(x) = {dP_{X \given W = 1}}(x) $.
We can easily estimate the first term $\E[Y^{(1)} \given W = 1]$ (Eq.~\eqref{eq:ygivenw1}) by the conditional sample average $\frac{1}{N_1} \sum_{i=1}^{N} Y_i \cdot \mathds{1}_{W_i = 1}$, where $N_1 := \sum_{i=1}^{N} \mathds{1}_{W_i = 1}$.
Estimating the second term $\E[Y^{(0)} \given W = 1]$ (Eq.~\eqref{eq:yrxgivenw1}) is more involved. The sample average with the condition $W = 0$, $\frac{1}{N_0} \sum_{i=1}^{M} Y_i \cdot \mathds{1}_{W_i = 0} $, where $N_0 := \sum_{i=1}^{N} \mathds{1}_{W_i = 0}$, would be biased to $\E[Y \given W = 0] \neq \E[Y^{(0)} \given W = 1]$, but the bias is only due to the change in the conditional distributions of $X$ given $W = 1$ and $X$ given $W = 0$ quantified by $r(X)$, similarly to the covariate shift (see Eq.~\eqref{eq:yrxgivenw1}).
One way to correct the bias is to use an estimate $\widehat{r}$ of the ratio $r$ for the weighted average $\frac{1}{N_0} \sum_{i=1}^{N} \widehat{r}(X_i) \cdot Y_i  \cdot \mathds{1}_{W_i = 0}$, similarly to the reweighting approach to covariate shift adaptation, leading to
the following estimate:
\begin{align*}
\widehat{\operatorname{ATT}}
&= \frac{1}{N_1} \sum_{i\colon W_i = 1} Y_i - \frac{1}{N_0} \sum_{i\colon W_i = 0} \widehat{r}(X_i) \cdot Y_i.
\end{align*}

Another popular approach is the nearest neighbor matching~\citet{abadie2006large}. See also \citet{Ros} for a broad introduction to matching problems for evaluating treatment effects.
In \citet{abadie2006large}, the ATT is estimated by 
\[
\overline{\mbox{ATT}}=\frac{1}{N_1}\sum_{i\colon W_i = 1} \left[Y_i-\hat{Y}_i(0)\right],
\]
where $\hat{Y}_i(0) = \frac{1}{k}\sum_{j=1}^N Y_j (1- W_j) 1 _ {\{ \| X_i -X_j\| \leq \tau_{n,k,X_i} \}}   $ is the average of $Y_j$'s over the $k$ first NNs of $X_i$ in the untreated group.
The estimator takes the form
% \[
% \overline{\mbox{ATT}}=\frac{1}{N_1}\sum_{i=1}^N \left[W_i-(1-W_i)\frac{K_k(i)}{k}\right]Y_i,
% \]
% \[
% \overline{\mbox{ATT}}=\frac{1}{N_1}\sum_{i: \text{treated}} Y_i  - \frac{1}{N_1}\sum_{i: \text{nontreated}} \frac{K_k(i)}{k} Y_i,
% \]
\[
\overline{\mbox{ATT}}=\frac{1}{N_1}\sum_{i\colon W_i = 1} Y_i  - \frac{1}{N_1}\sum_{i\colon W_i = 0} \frac{K_k(i)}{k} Y_i,
\]
where $K_k(i) = \sum_{j = 1 } ^N W_j 1_{ \| X_i - X_j\| \leq \tau_{n,k,X_j}} $ is the number of times observation $i$ is used as a match, i.e., the number of times observation $i$ is among the $k$ NNs of variables in the treated group.
Note that $\overline{\mbox{ATT}}$ coincides with $\widehat{\operatorname{ATT}}$ if $\hat{r}(X_i)$ is defined as $\frac{K_k(i)N_0}{k N_1}$. Recently, \citet{lin2023estimation} showed that the latter quantity can be indeed interpreted as an estimate of the density ratio but its consistency requires $k\rightarrow \infty$ while \citet{abadie2006large} considered a fixed value of $k$, as in our problem.
To see an analogy with our method, one can consider the case in which $h$ does not depend on $x$, i.e.\@ $h(y,x)=g(y)$ for some function $g$. Using the notation from the present paper ($W = 0$ and $W = 1$ indicate the target and the source domain, respectively, with $N_0 = n$ and $N_1 = m$) the second term of $\overline{\textrm{ATT}}$ above generalizes to the form
\begin{equation}\label{their}
\frac{1}{m}\sum_{i=1}^n\frac{K_k(i)}{k}g(Y_i)
= \frac{1}{m}\sum_{i=1}^m \int  g( y ) \hat P_ n (dy | X_i^*)
\end{equation}
with $K_k(i) = \sum_{j = 1 } ^m  1 _ {\{ \| X_i - X_j^*\| \leq \tau_{n,k,X_j^*} \}} $. The previous estimate corresponds to the one introduced in the notes following Proposition~\ref{expsampling}.
On the other hand, our estimator applied to this case is given by
\begin{equation}\label{ours}
 \frac{1}{m}\sum_{i=1}^m  g(Y_{n,i}^*)
%\frac{1}{m}\sum_{i=1}^m  h(Y_{n,i}^*, X_i^*).
\end{equation}
%where
%\[
%K^{*}_k(i)=\sum_{j=1}^m \mathds{1}_{\Vert X_j^{*}-X_i\Vert\leq \hat{\tau}_{n,k,X_j^{*}}}
%\]
 Both estimators are different when $k>1$ but they coincide as soon as $ k = 1$. In fact, $  \hat P_n (dy | X_i^*) $ has one single atom when $k = 1$, so that sampling from it and evaluating the average are the same. Here are a few remarks.
\begin{itemize}
    \item When $k>1$, Eq.~\eqref{ours} requires fewer evaluations of $g$ than Eq.~\eqref{their}. This is relevant when evaluation of the function is time-consuming or costly such as observation from physical experiments.
    \item Our theoretical analysis is rather different from that of \citet{abadie2006large}. Since they rely on the expression in the left side of Eq.~\eqref{their}, it is unclear whether they can or not handle the case when $h$ depends on $x$ (required for prediction purpose). In contrast, our approach is based on the decomposition given in Section~\ref{sec:sampling}, with sampling error and estimation error, leveraging $\int \int  g( y ) \hat P_n (dy | x) Q(dx) $ as a centering term. Our results are more general because they include the case when $h$ depends on $x$ and also we can deal with both estimates \eqref{their} and \eqref{ours} in the meantime, as mentioned in the notes following Proposition~\ref{expsampling}). Moreover, Proposition~\ref{lionel_prop} implies a lower bound for Eq.~\eqref{their} and we believe this result to be new in treatment effect literature.     
    
    %. 
%    Expression~\eqref{their} is no more valid when the function $h(y, X_i^*)$ depends on its second argument $X_i^*$, which is the most interesting situation for risk estimation.
     
\end{itemize}

\paragraph{Other references}
The idea of nonparametric sampling is a standard one in the field of texture synthesis. In particular, the choice of 1-NN resampling was often used as a fast method to generate new textures from a small sample. See \citet{truquet2011nonparametric} for a literature review in this context.
Our conditional sampling framework bears resemblance with traditional bootstrap sampling as there is random generation according to some estimated distribution. In contrast, the original bootstrap method is usually made up using draws from the standard empirical measure $(1/n) \sum_{i=1} ^n \delta_{X_i,Y_i}$. Here another distribution, $\hat P $, has been used to generate new samples. Moreover, our goal is totally different here. While the  bootstrap technique was initially introduced for making inference, here the goal is to estimate an unknown quantity $Q(h)$ which appears in many machine learning tasks.
\citet{kpotufe2021covshift} theoretically study covariate shift adaptation under the assumption that we have access to a labeled sample both from the source and the target distribution.
Although they consider a $k$-nearest-neighbor-based method, it is essentially different from ours since they perform the $k$-NN method on the union of the source and the target sample.
\citet{lee2013pseudo} proposed pseudo-labeling unlabeled data in the context of semi-supervised learning.
\citet{wang_pseudo-labeling_2023} proposed a hyper-parameter selection method for kernel ridge regression under covariate shift using pseudo-labeling.
The author focuses on model selection in regression problems while we study the mean estimation that can be applied to a wider range of supervised learning problems.

\section{Extensions}\label{sec:possible_extensions}
Several ways to extend our method beyond the mean estimation problem are considered in this section.

\paragraph{Heterogeneity in target distributions}\label{submit2}

The case where the target covariates distribution $Q_{X}$ changes across the data might be of interest if one wishes to aggregate several pieces of target data whose covariates distributions are not necessarily the same.
This might occur when the target data is obtained by gathering individuals from  different countries, and consequently, the distributions are not the same anymore or when the time between the measurements has caused some changes in the distribution.

%FP: Lionel, Ikko do you have in mind particular examples? \textcolor{red}{I'm not totally sure yet, but for problems such as online learning in which we consider regret minimization, the data distributions are not necessarily identical. It can even be the case that there are no relationships between different data instances, or they are an adversarial sequence. But we still need the condition $P_{Y | X} = Q_{Y | X}$, so it would be like regret minimization under a constraint.}

While such an heterogeneity in target data might be seen as more complicated at first glance, it actually can be examined using a similar decomposition and the same tools as the one used to obtain the non-asymptotic bound in Theorem~\ref{global}. More formally, the target distribution is here  $ Q = (1/m) \sum_{i=1}^m Q_{i} $ with  $ Q_i =  P_{Y|X} Q_{X,i}$. For each $m\geq 1$ and  $n\geq 1$, 
let $(X_{i}^*, Y_{n,i}^*)_{1\leq i\leq m}$ be a collection of 
random variables conditionally independent given $(X_{i}, Y_{i})_{i=1}^{n}$ and such that for each $i = 1,\ldots, m $, 
$ (X_{i}^*, Y_{n,i}) \sim \hat Q_{i,n} $ with $\hat Q_{i,n} = \hat P_{Y|X} Q_{X,i} $.  The quantity of interest and the proposed estimator are therefore slightly different from before, given by, respectively,
$$Q (h) = m^{-1} \sum_{i = 1} ^ m Q_{i} ( h), \qquad \hat Q^* (h) = m^{-1} \sum_{i = 1} ^ m h (X_{i}^*, Y_{n,i}^*) . $$
The decomposition is 
$$\hat Q^* (h) -  Q (h) = m^{-1} \sum_{i = 1} ^ m \left\{h (X_{i}^*, Y_{n,i}^*) -  \hat Q_{i,n}(h) \right\} +m^{-1} \sum_{i = 1} ^ m \left\{ \hat Q_{i,n}(h) -  Q _i (h) \right\}.   $$
The non-asymptotic analysis of the sampling error is similar to before as the Bernstein inequality is tailored to non-identically distributed variables. We obtain that the rate $ O ( m^{-1/2} )$
%$$m^{-1} \sum_{i = 1} ^ m \left\{h (X_{i}^*, Y_{n,i}^*)-  \hat Q_{i,n}(h) \right\} = O ( m^{-1/2} ) $$
as before by simply requiring a bound on the variance of each random variables. The other term concerning the conditional distribution can be analyzed by writing
\[
\left\lVert m^{-1} \sum_{i = 1} ^ m \left\{ \hat Q_{i,n}(h) -  Q_i (h) \right\} \right\rVert_{L_2} \leq m^{-1} \sum_{i = 1} ^ m \lVert  \hat Q_{i,n}(h) -  Q_i (h) \rVert_{L_2}
%\leq \max_{1\leq i\leq m}\|  \hat Q_{i,n}(h) -  Q_i (h) \|_{L_2}     
\]
and therefore we can directly apply Proposition  3 (given the assumptions are satisfied for each $i$ uniformly). We finally obtain the rate
$ m^{-1/2}+ n^{ - 1/d}  + {n^{ - 1/2}} $, similar to the one obtained before.

\paragraph{Stochastic gradient descent}

Our sampling approach can be easily combined with the well-known stochastic gradient descent algorithm (and more generally with stochastic approximation) where only a small part of the data is used at each step to update the estimator. This particular property allows to require a small number of operations at each iteration (in contrast with gradient based optimization).
%is suited to optimize empirical risks in large scale learning problems because standard optimization technique (gradient based) are insufficient.
%when one cannot solve explicitly or in a reasonable time the optimization problem. 
%

To illustrate this idea, consider the empirical risk minimization problem described in Section~\ref{submit} where one is interested in solving $\min _{\theta }\{  R^* (\theta): = \mathbb{E} [m_\theta(Y^*, X^*)]\}$ where $\theta \mapsto m_\theta(y,x)$ is differentiable. Suppose that $n $ source samples have been obtained making the conditional distribution $\hat P_{Y|X} $ available for sampling new points. Then the algorithm at step $i\geq 1$, might proceed by first generating $X^*_{i} $ and then $Y_{n,i}^* \sim \hat P_{Y|X_i^*} $. This means finding the nearest neighbor to $X_i^*$ among the source data and represents only $\log(n)$  operations using the $kd$-tree. Having this been done, the update is simply
$$ \theta_ {i} = \theta_{i-1} - \gamma_{i} \nabla_\theta m_\theta(Y^*_{n,i}, X^*_{i}).$$
It results that each iteration in the above is similar to standard stochastic gradient descent, the only difference being the additional $1$-nearest neighbor search. We stress that this is contrasting with the re-weighting approach for which a new sample, say $X_{i}$, would require evaluating $ \hat p_X(X_{i})$ and therefore would need to compute all $n$ distances between $X_j$, $j=1,\ldots, n$ and  the new $X_{i}$.

\paragraph{Semiparametric estimation}
Simulating the labels to obtain a new sample is also convenient  in semiparametric problems where quantities of interest often involve additional estimated parameters. Typical semiparametric problems involve expectations of functions that are indexed by an unknown parameter, $\E [ h_\theta (X,Y) ]$, and $\theta$ is estimated from the data using some transformation $\hat \theta $ of the sample. In such a situation, while estimating $\theta$ using reweighting is unclear without more information on $\hat \theta$, one can directly use our sampling approach by introducing
$m^{-1}\sum_{j=1}^m h_{\hat{\theta}^{*}}\left(Y_j^{*},X_j^{*}\right)$ where 
$\hat{\theta}^{*}=\hat{\theta}\left(X_1^{*},Y_1^{*},\ldots,X_m^{*},Y_m^{*}\right)$. This allows to obtain a semiparametric estimate with covariate shift adaptation. See \citet{van2000asymptotic}, Chapters $19.4$ and $25.8$ for more details and examples in parametric or semiparametric estimation.

\section{Experiments}\label{sec:experiments}
The main purpose of the experiments is to compare our $k$-NN-CSA approach with several state-of-the-art competitors when facing multiple situations from mean estimation to empirical risk minimization with synthetic and real-world data.

We consider the following instances of our proposed method.
\begin{description}
\item[$1$-NN-CSA:] the \emph{Conditional Sampling Covariate-shift Adaptation (CSA)} (Algorithm~\ref{alg:general_condisample_covshift}) with $k$-Nearest Neighbor ($k$-NN) conditional sampler (Algorithm~\ref{alg:kNNcondi}) with $k = 1$.\item[$\log n$-NN-CSA:] the same as above but with $k = \log n$.
\end{description}
We use the Python module \texttt{cKDTree}~\citep{ckdtree} from SciPy~\citep{2020SciPy-NMeth} for nearest neighbor search in our methods.
We compare them with the following existing covariate-shift adaptation methods.
\begin{description}
\item[KDE-R-W (KDE-Ratio-Weighting):] the weighting method using the ratio of the Kernel Density Estimates (KDEs) of $p_X$ and $q_X$ (see Section~\ref{sec:related}).
\item[KMM-W (KMM-Weighting):] the weighting method estimating $q_X/p_X$ using the \emph{Kernel Mean Matching (KMM)}~\citep{huang2006kmm,gretton2008kmm}. We use the Gaussian kernel.
\item[KLIEP-W (KLIEP-Weighting):] the weighting method estimating $q_X/p_X$ by the \emph{Kullback-Leibler Importance Estimation Procedure (KLIEP)}~\citep{sugiyama2007direct,sugiyama2008direct}.
The linear combination of the Gaussian basis functions centered at the sample points are used for modeling the weight function.
\item[KLIEP100-W]: the same as KLIEP-W but with 100 randomly subsampled basis functions~\citep{sugiyama2007direct,sugiyama2008direct} for reducing the time- and space-complexities.
\item[RuLSIF-W (RuLSIF-Weighting):] the weighting method using $q_X/(\alpha p_X + (1-\alpha)q_X)$ estimated by \emph{Relative unconstrained Least-Squares Importance Fitting (RuLSIF)}~\citep{yamada2013rulsif}, where $\alpha \in [0, 1]$ is a hyper-parameter.
We use the default value $\alpha = 0.1$.
As a model of the weight function, the Gaussian basis functions centered at the sample points are used.
\item[RuLSIF100-W:] the same as RuLSIF-W but with 100 randomly subsampled basis functions~\citep{yamada2013rulsif} for reducing the time- and space-complexities.
\end{description}
See Section~\ref{sec:related} for more explanations of those methods.
For KMM-W, KLIEP-W, and RuLSIF-W, we used the implementations from \emph{Awesome Domain Adaptation Python Toolbox (ADAPT)}~\citep{demathelin2021adapt}. All the computations were performed on the cluster, \emph{Grid5000}~\citep{grid5000}.
For the methods using Gaussian basis functions (KLIEP-W, KLIEP100-W, RuLSIF-W, RuLSIF100-W), we use 5-fold cross-validation for choosing the Gaussian bandwidth from $\{0.001, 0.01, 0.1, 1, 10\}$.
KMM-W does not offer a way to do cross-validation, and we fixed to $1$.
More details are in the supplementary material.

Furthermore, we also report the results for the following baseline method and ideal method.
\begin{description}
    \item[NoCorrection:] the method that takes the average $\frac{1}{n} \sum_{i=1}^{n} h(X_{i}, Y_{i})$ only using the source sample $(X_{i}, Y_{i})_{i=1}^{n} \sim P^{m}$, ignoring the target sample.
    \item[OracleY:] the result for taking the average $\frac{1}{m} \sum_{i=1}^{m} h(X^{*}_{i}, Y^{\circ}_{i})$ using a sample $(X^{*}_{i}, Y^{\circ}_{i})_{i=1}^{m} \sim Q^{m}$. Note that $Y^{\circ}_{i}$ are not available in practical scenarios of our interest and made invisible to other methods.
\end{description}
We conduct experiments in three setups, detailed below, with different sample sizes $n$ ($= m$) and data dimensionalities $d$: $(n, d) \in \{50, 100, 500, 1000, 5000, 10000\} \times \{5, 10\}$. Each experiment is repeated 50 times with different random seeds.

\paragraph{Setup of Experiment E1 (mean estimation with synthetic data):} % toy2_linear
The task here is to estimate $Q(h) = \iint y\, P_{Y|X}(dy \given x) Q_X (dx)$ under the following setup.
We define $h$ by $h(x, y) = y$, $P_X$ as the uniform distribution over $[-1, 1]^d$, $Q_X$ as that over $[0, 1] \times [-1, 1]^{d-1}$, and $P_{Y|X=x}$ as the normal distribution with mean $x$ and variance $0.1$.
Figure~\ref{app:fig:data0} in Appendix~\ref{app:moreplots} shows an illustration of the setup.
In this setup, we have
$P(h) = \iint y\, P_{Y|X}(dy \given x) P_X (dx) = 0$ while $Q(h) = \iint y\, P_{Y|X}(dy \given x) Q_X (dx) = 0.5$.
Because of this difference, covariate shift adaptation is essential for correctly estimating $Q(h)$.

\paragraph{Comparison of estimation errors for Experiment E1:}
The results are presented in Figure~\ref{fig:toy2_linear_errors}.
First, the errors for NoCorrection are not decreasing as the sample sizes increase, ending up with large errors in all cases, because of the bias due to the covariate shift.
Other methods with covariate-shift adaptation had always smaller errors than that of this baseline.
Excluding OracleY, an ideal method unavailable in practice, KLIEP100-W, KMM-W, 1-NN-CSA, and $\log n$-NN-CSA were among the best for smaller dimensionalities $d \in \{1, 2\}$ (Figures~\ref{subfig:toy2_linear_dim1_errors} and \ref{subfig:toy2_linear_dim2_errors}).
For the larger dimensionalites $d \in \{5, 10\}$, KMM-W and 1-NN-CSA outperformed other methods. In particular, 1-NN-CSA gave outstanding performances in many cases except $d = 5$ and $n \in \{100, 500, 1000\}$, for which KMM-W was even better.
The errors of most methods roughly follow power laws, where the slope of a line corresponds to the power of the convergence rate (steeper is better).
1-NN-CSA and $\log n$-NN-CSA seem to have the steepest slopes for $d = 10$, although comparison is difficult for the lower dimensionalities.

\paragraph{Comparison of running times in Experiment E1:}
Figure~\ref{fig:toy2_linear_times} shows the comparison in running times.
1- and $\log n$-NN-CSA were much faster than other methods in all cases except for $(d, n) = (10, 50)$.
Their advantage is most pronounced for larger sample sizes. For instance, 1- and $\log n$-NN-CSA were at least $100$ times faster than other methods for $(n, d) = (10000, 1)$ (Figure~\ref{subfig:toy2_linear_dim1_times}).

\begin{figure*}[tbp]
  \centering
  \begin{subfigure}[b]{0.24\textwidth}
    \centering
    \includegraphics[width=\textwidth]{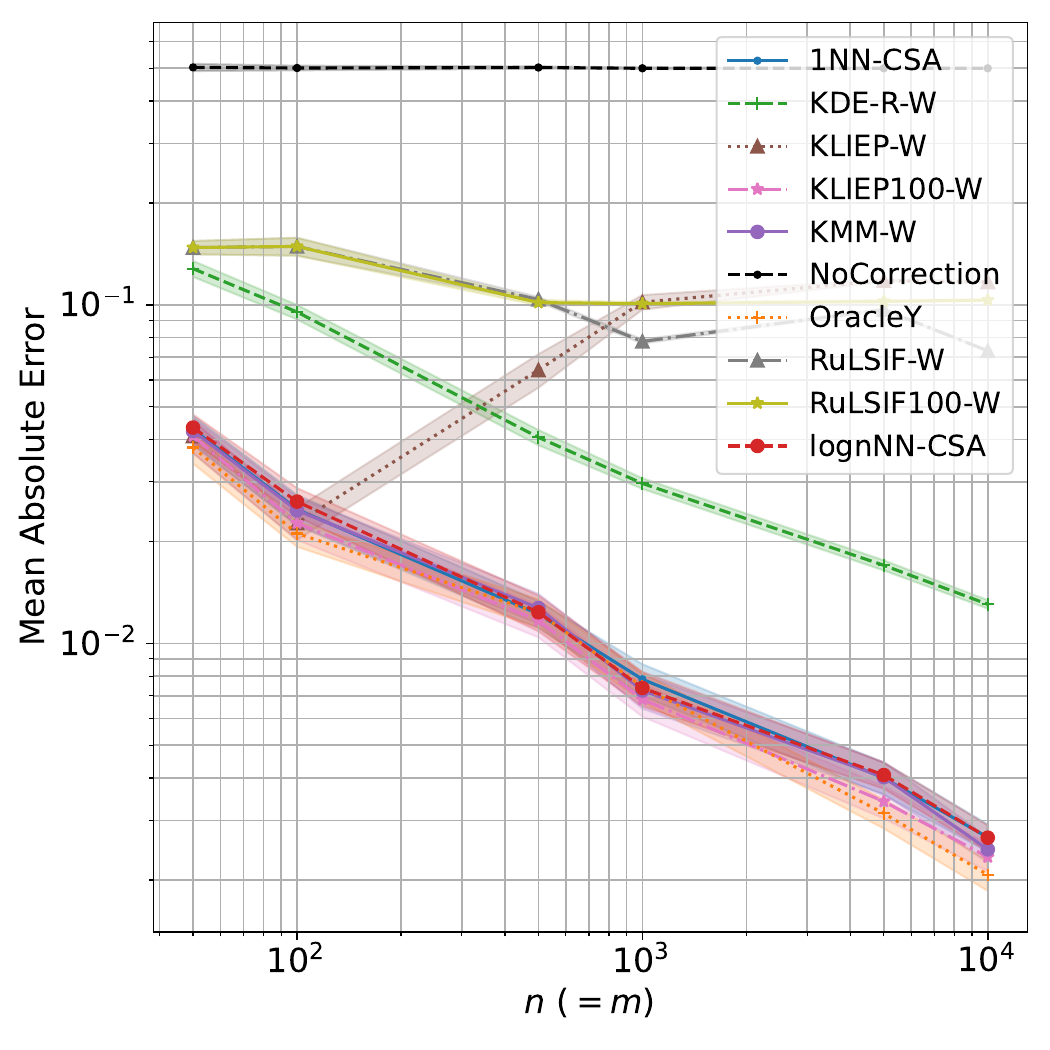}
    \caption{$d = 1$}
    \label{subfig:toy2_linear_dim1_errors}
  \end{subfigure}
  \begin{subfigure}[b]{0.24\textwidth}
    \centering
    \includegraphics[width=\textwidth]{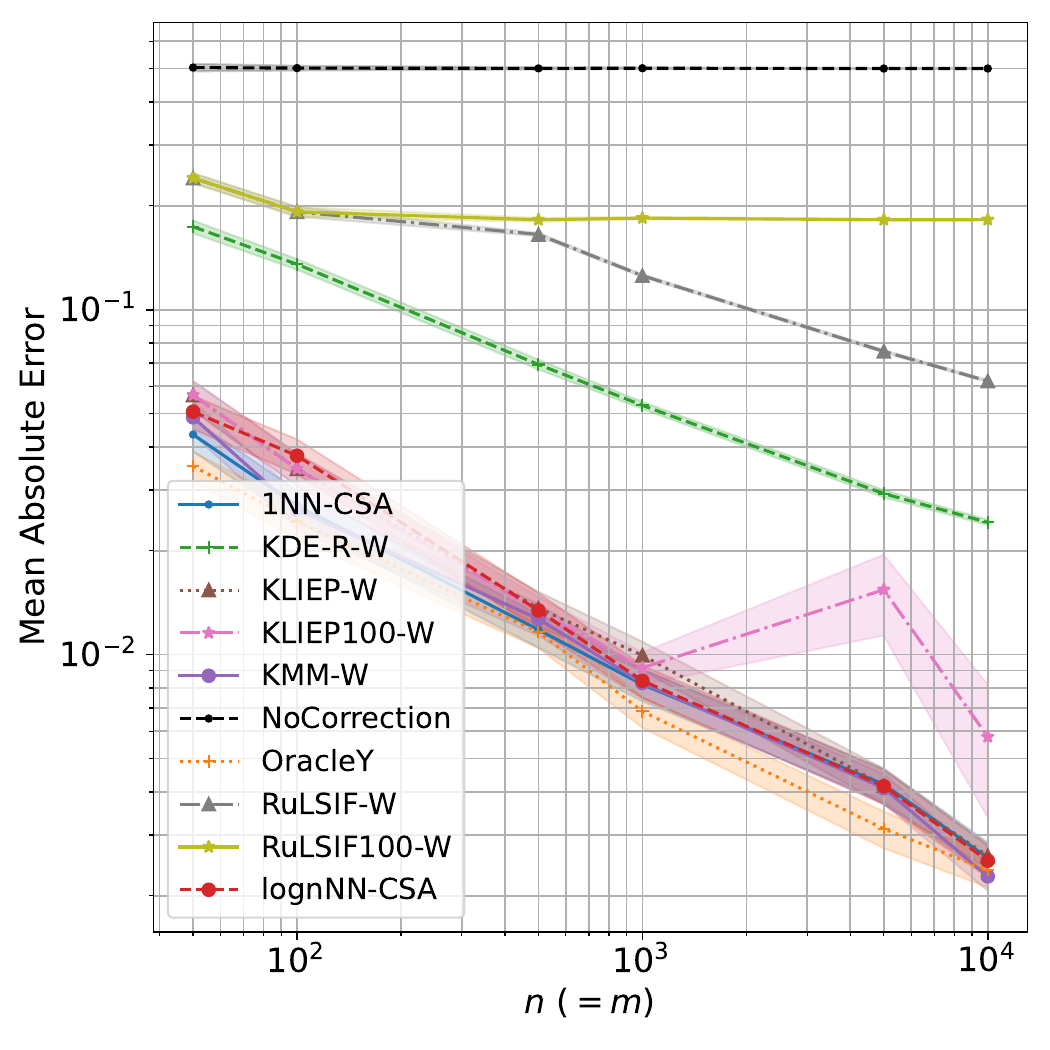}
    \caption{$d = 2$}
    \label{subfig:toy2_linear_dim2_errors}
  \end{subfigure}
  \hfill
  \begin{subfigure}[b]{0.24\textwidth}
    \centering
    \includegraphics[width=\textwidth]{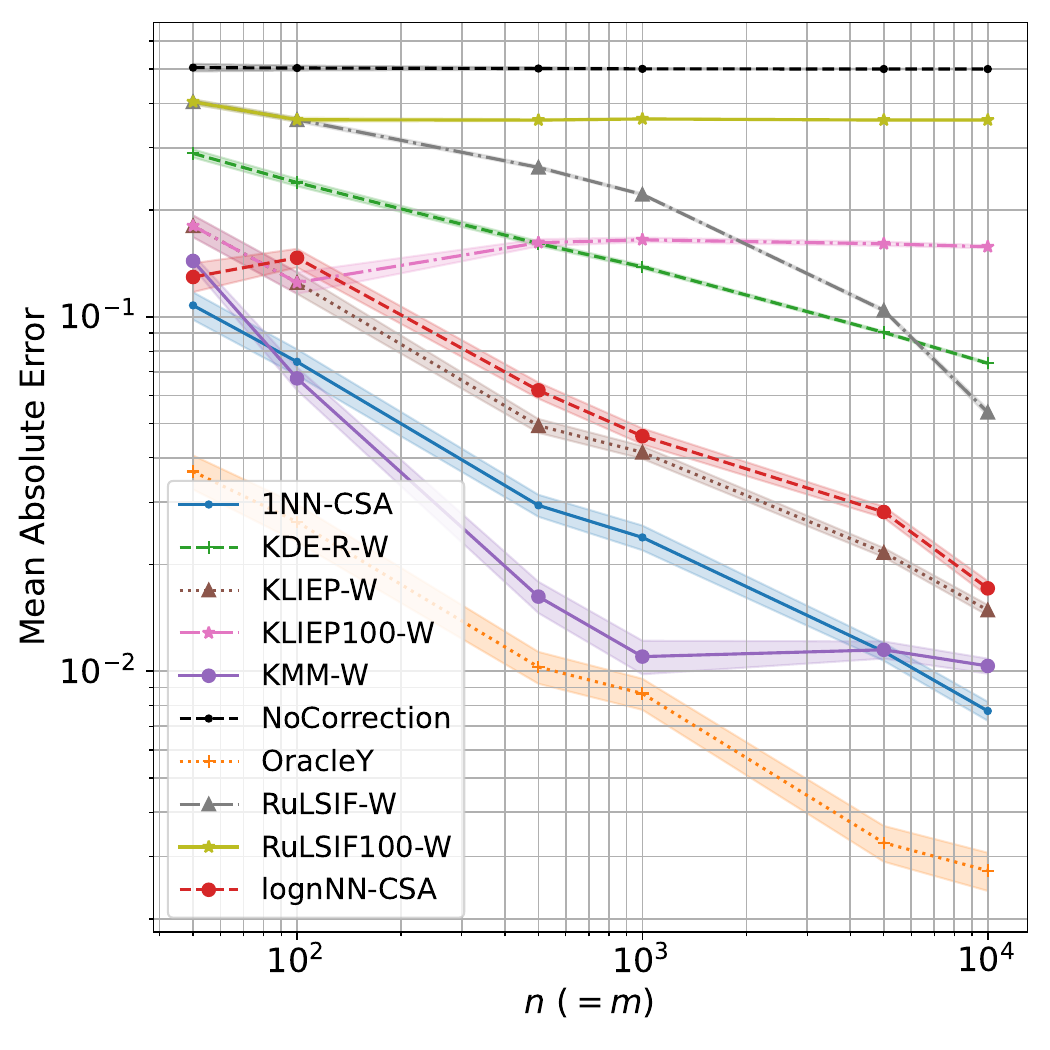}
    \caption{$d = 5$}
    \label{subfig:toy2_linear_dim5_errors}
  \end{subfigure}
  \begin{subfigure}[b]{0.24\textwidth}
    \centering
    \includegraphics[width=\textwidth]{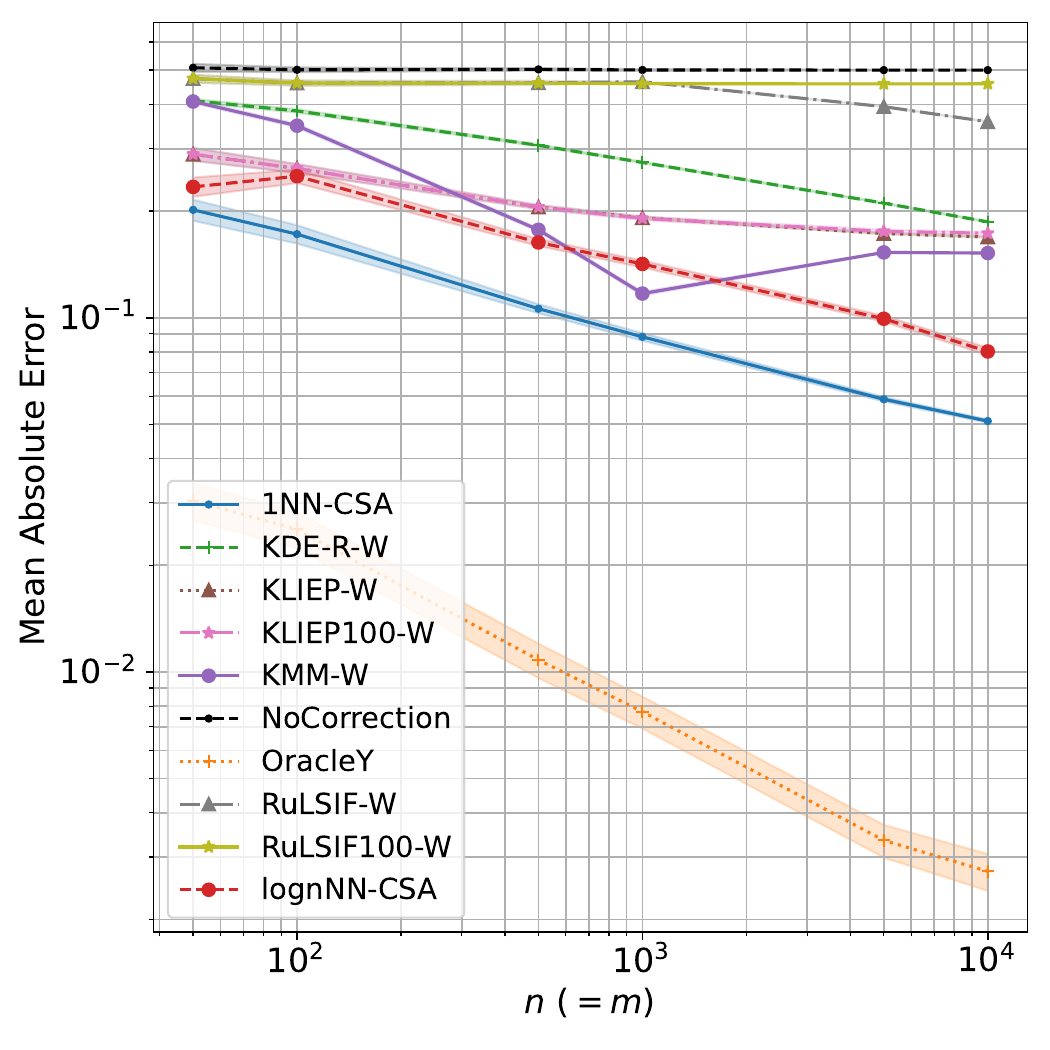}
    \caption{$d = 10$}
    \label{subfig:toy2_linear_dim10_errors}
  \end{subfigure}
  \caption{Mean Squared Errors (MSE) for Experiment E1 (estimation of $\int y\, Q(dy)$). The horizontal axis is for the sample sizes $n$ ($= m$), and the vertical axis is for the mean absolute error of each estimate.  The four figures are for different data dimensionalities.}
  \label{fig:toy2_linear_errors}
\end{figure*}

\begin{figure*}[tbp]
  \centering
  \begin{subfigure}[b]{0.24\textwidth}
    \centering
    \includegraphics[width=\textwidth]{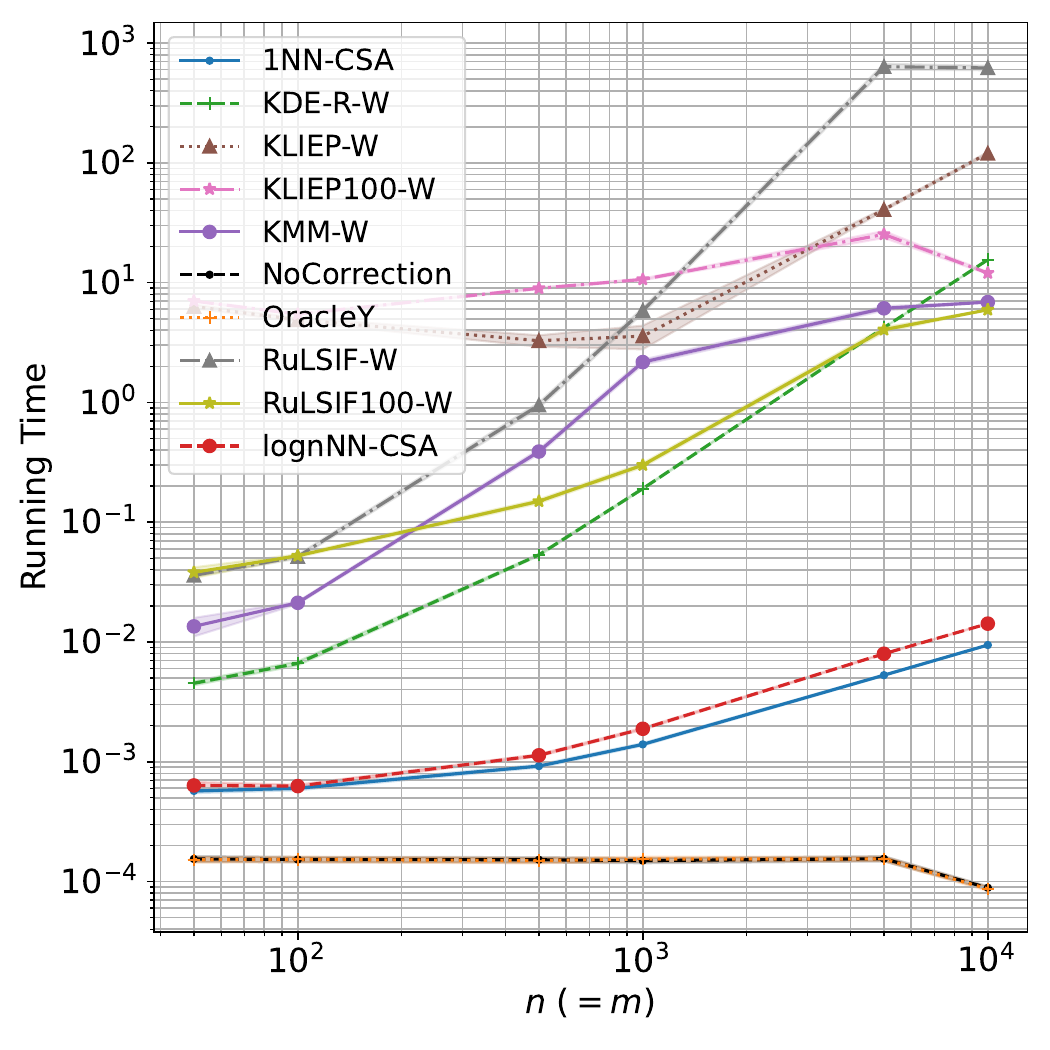}
    \caption{$d = 1$}
    \label{subfig:toy2_linear_dim1_times}
  \end{subfigure}
  \begin{subfigure}[b]{0.24\textwidth}
    \centering
    \includegraphics[width=\textwidth]{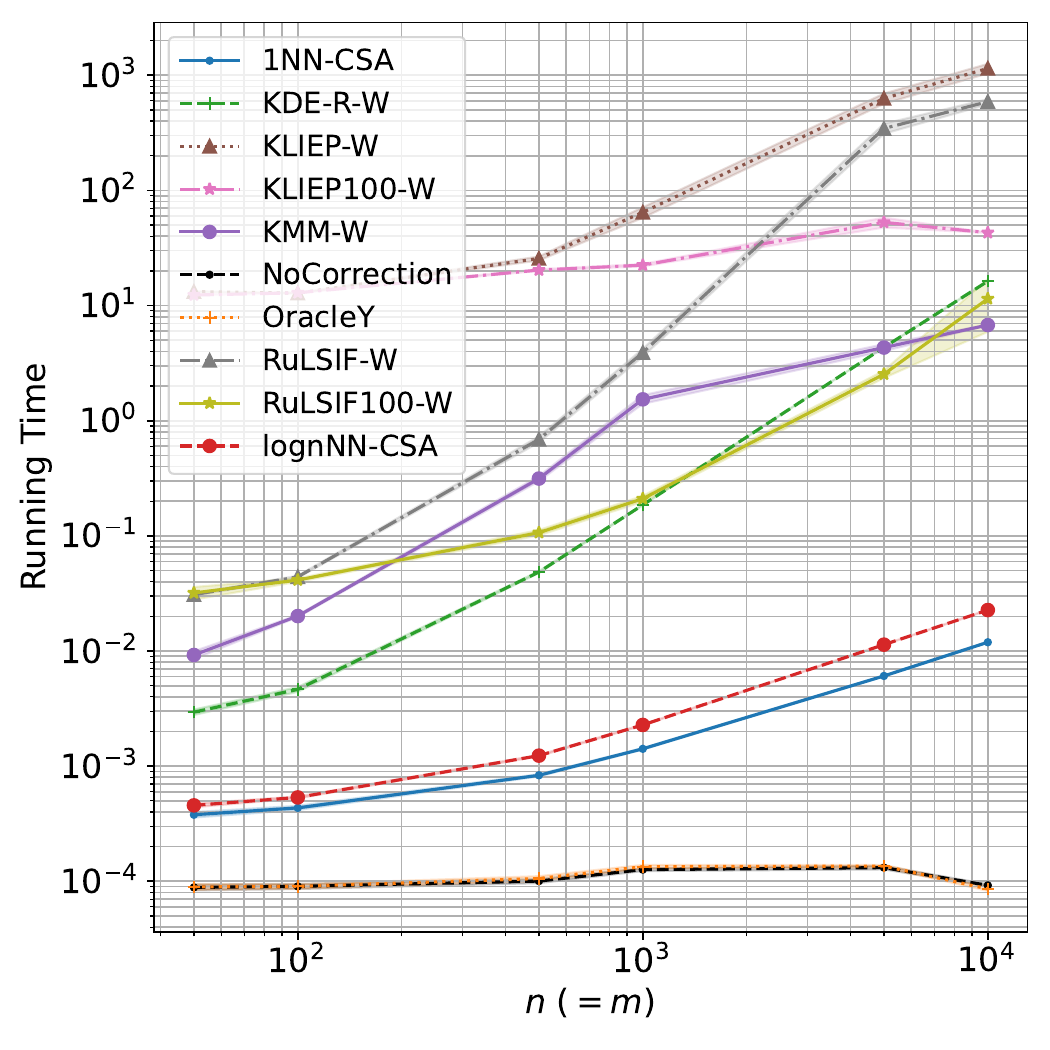}
    \caption{$d = 2$}
    \label{subfig:toy2_linear_dim2_times}
  \end{subfigure}
  \hfill
  \begin{subfigure}[b]{0.24\textwidth}
    \centering
    \includegraphics[width=\textwidth]{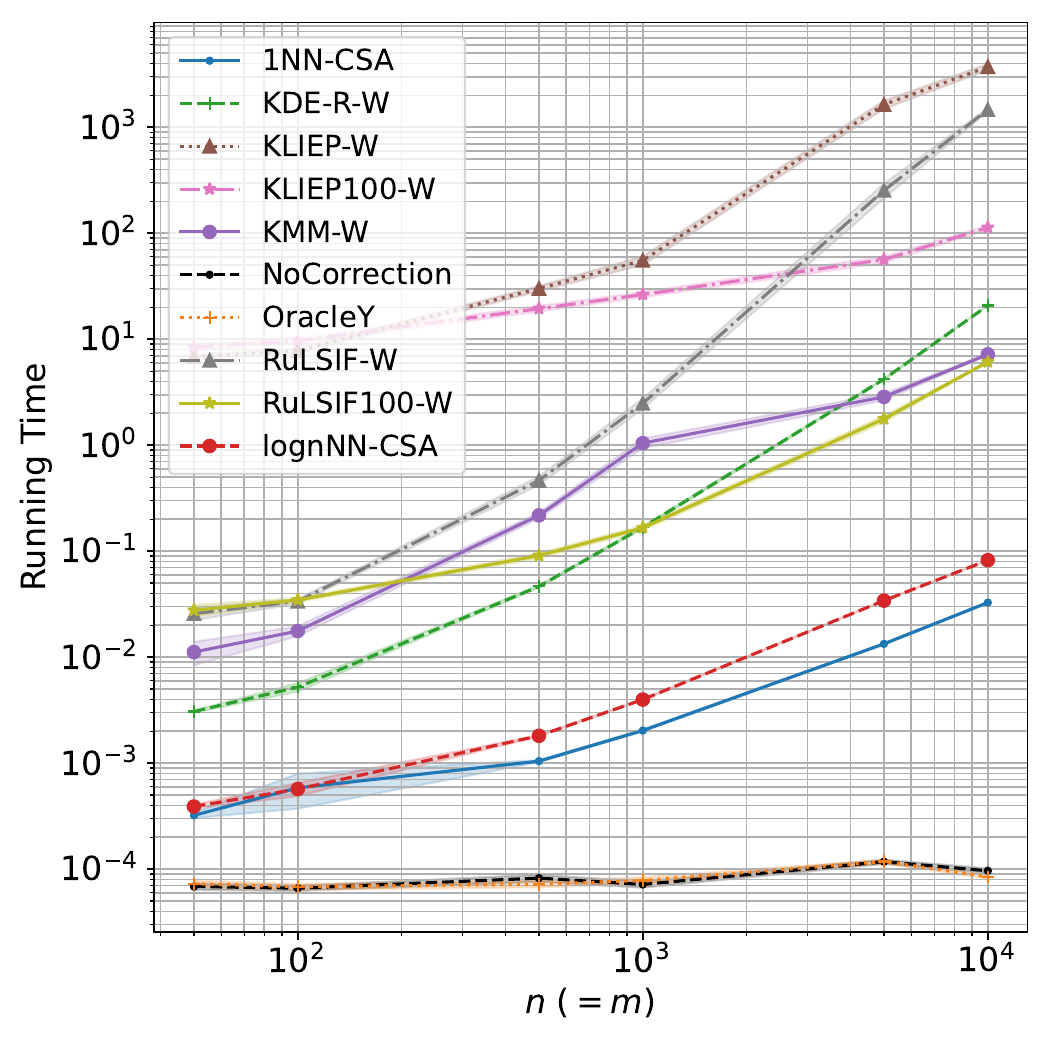}
    \caption{$d = 5$}
    \label{subfig:toy2_linear_dim5_times}
  \end{subfigure}
  \begin{subfigure}[b]{0.24\textwidth}
    \centering
    \includegraphics[width=\textwidth]{./figures/time_toy2_linear_dim10.pdf}
    \caption{$d = 10$}
    \label{subfig:toy2_linear_dim10_times}
  \end{subfigure}
  \caption{Running times for Experiment E1.  The horizontal axis is for the sample sizes $n$ ($= m$), and the vertical axis is for the mean running time of each method. The four figures are for different data dimensionalities.}
  \label{fig:toy2_linear_times}
\end{figure*}

\paragraph{Setup of Experiment E2 (risk estimation with synthetic data):} % toy2_abs
In this experiment, we compare the methods in the context of risk estimation of a fixed function $f_{0}$.
Setting $f_{0}\colon x \mapsto - x_1$, where $x_1$ is the first coordinate of $x$, we estimate the expected loss (i.e., risk) of $f_{0}(X)$ with the square loss in predicting the response $Y$ when $(X, Y)$ follows $Q$.
In other words, we set $h$ as $h(x, y) := (y - f_{0}(x))^2$, and the goal is to estimate the risk $Q(h) = \int (y - f_{0}(x))^2 Q(dx, dy)$.
We use the uniform distribution over $[-1, 1]^d$ for $P_{X}$ and that over $[0, 1] \times [0, 1]^{d-1}$ for $Q_{X}$.
The conditional distribution $P_{Y|X=x}$ is the normal distribution with mean $\abs{x_1}$ and variance $0.1$ for any $x \in \mathcal X := [-1, 1]^d$.
Under this setup, the function $f_{0}$ performs poorly on the support of $Q_{X}$ and should incur a large risk.
See Figure~\ref{app:fig:data1} for an illustration of the setup.
In this setup, the risks under $P_{X}$ and $Q_{X}$ largely differ  because $f_{0}$ fits $(X, Y)$ well in a half of the support of $P_{X}$ but not in that of $Q_{X}$.

\paragraph{Comparison of estimation errors for Experiment E2:}
We present the estimation errors for Experiment E2 in Figure~\ref{fig:toy2_abs_errors}.
KMM-W, 1-NN-CSA, $\log n$-NN-CSA gave similar results, almost matching those of OracleY, while KMM-W and 1-NN-CSA were advantageous for $d = 5$, and 1-NN-CSA outperformed other methods for $d = 10$.
We can notice that KDE-W, KMM-W, RuLSIF-W, and RuLSIF100-W did not always improve errors over NoCorrection (Figures~\ref{subfig:toy2_linear_dim5_errors} and \ref{subfig:toy2_linear_dim10_errors}).
 Some methods such as KMM-W and KLIEP-W showed great performance in some cases while giving poor results in other cases. In contrast, 1-NN-CSA showed stable and often best performances in these experiments.

\begin{figure*}[tbp]
  \centering
  \begin{subfigure}[b]{0.24\textwidth}
    \centering
    \includegraphics[width=\textwidth]{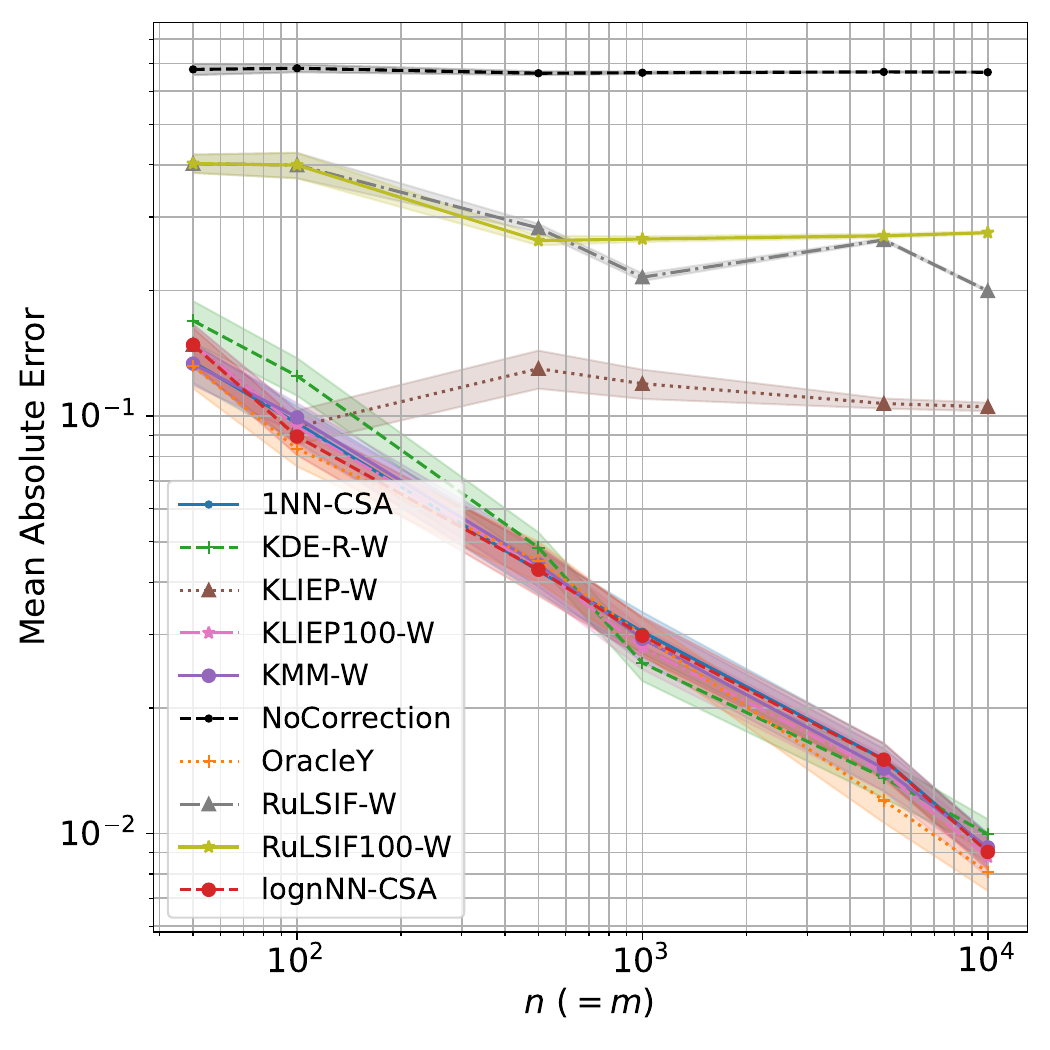}
    \caption{$d = 1$}
    \label{subfig:toy2_dim1_errors}
  \end{subfigure}
  \begin{subfigure}[b]{0.24\textwidth}
    \centering
    \includegraphics[width=\textwidth]{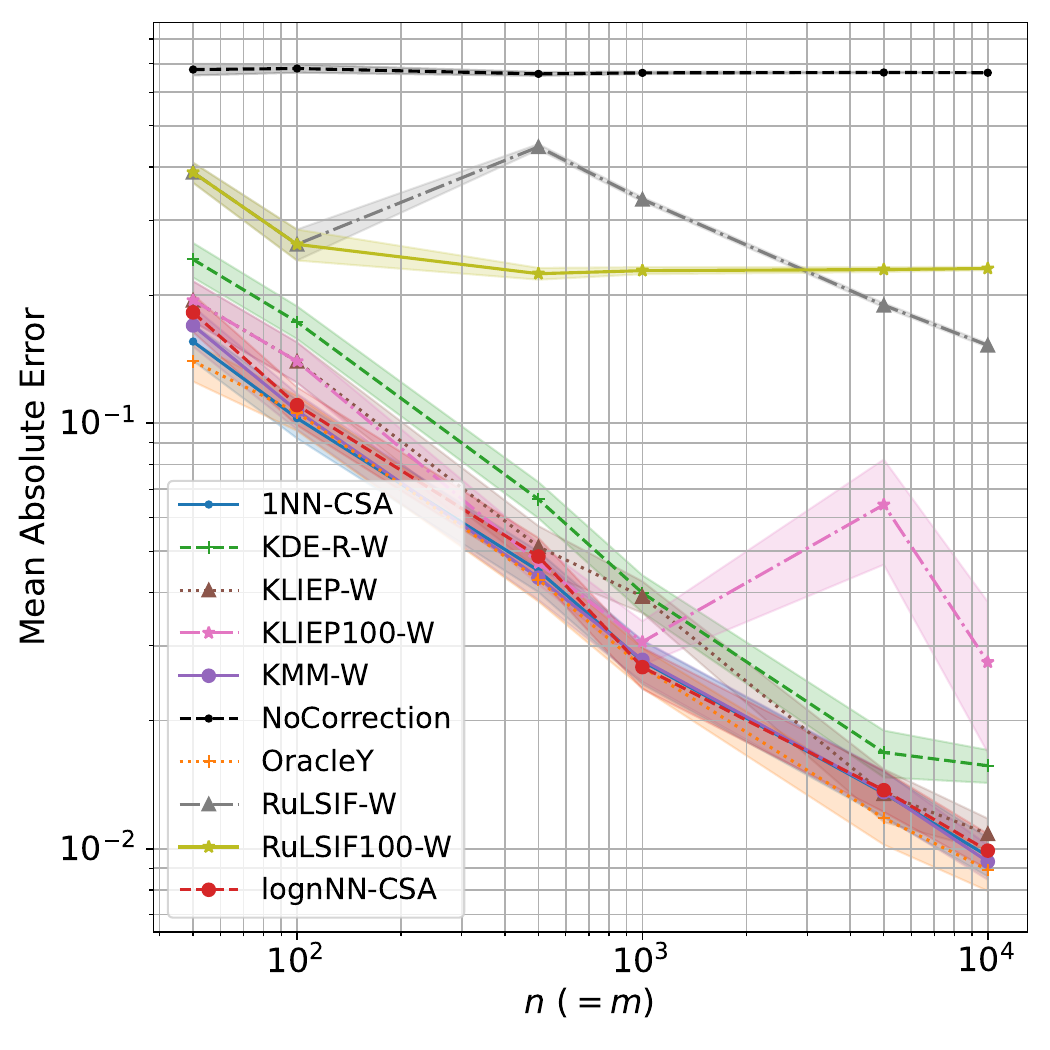}
    \caption{$d = 2$}
    \label{subfig:toy2_dim2_errors}
  \end{subfigure}
  \hfill
  \begin{subfigure}[b]{0.24\textwidth}
    \centering
    \includegraphics[width=\textwidth]{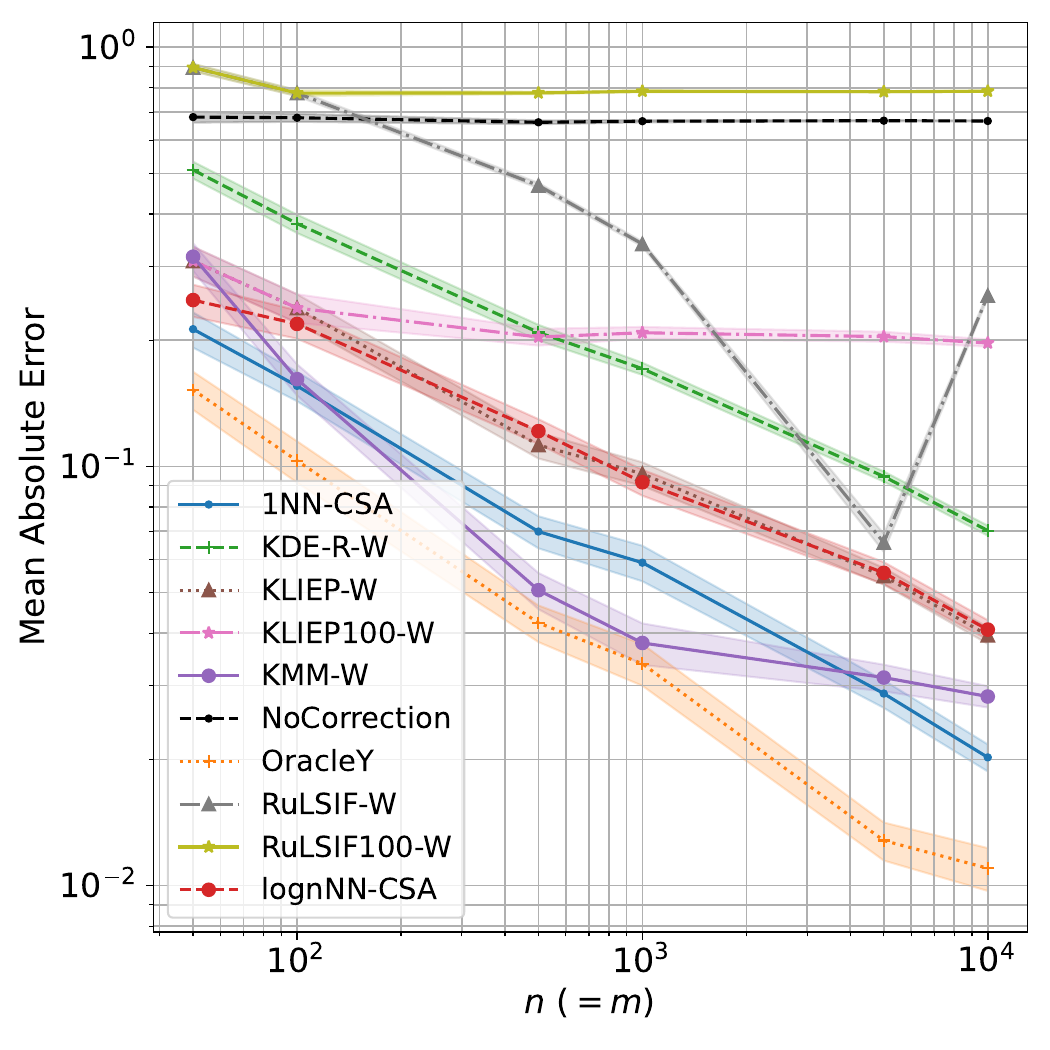}
    \caption{$d = 5$}
    \label{subfig:toy2_dim5_errors}
  \end{subfigure}
  \begin{subfigure}[b]{0.24\textwidth}
    \centering
    \includegraphics[width=\textwidth]{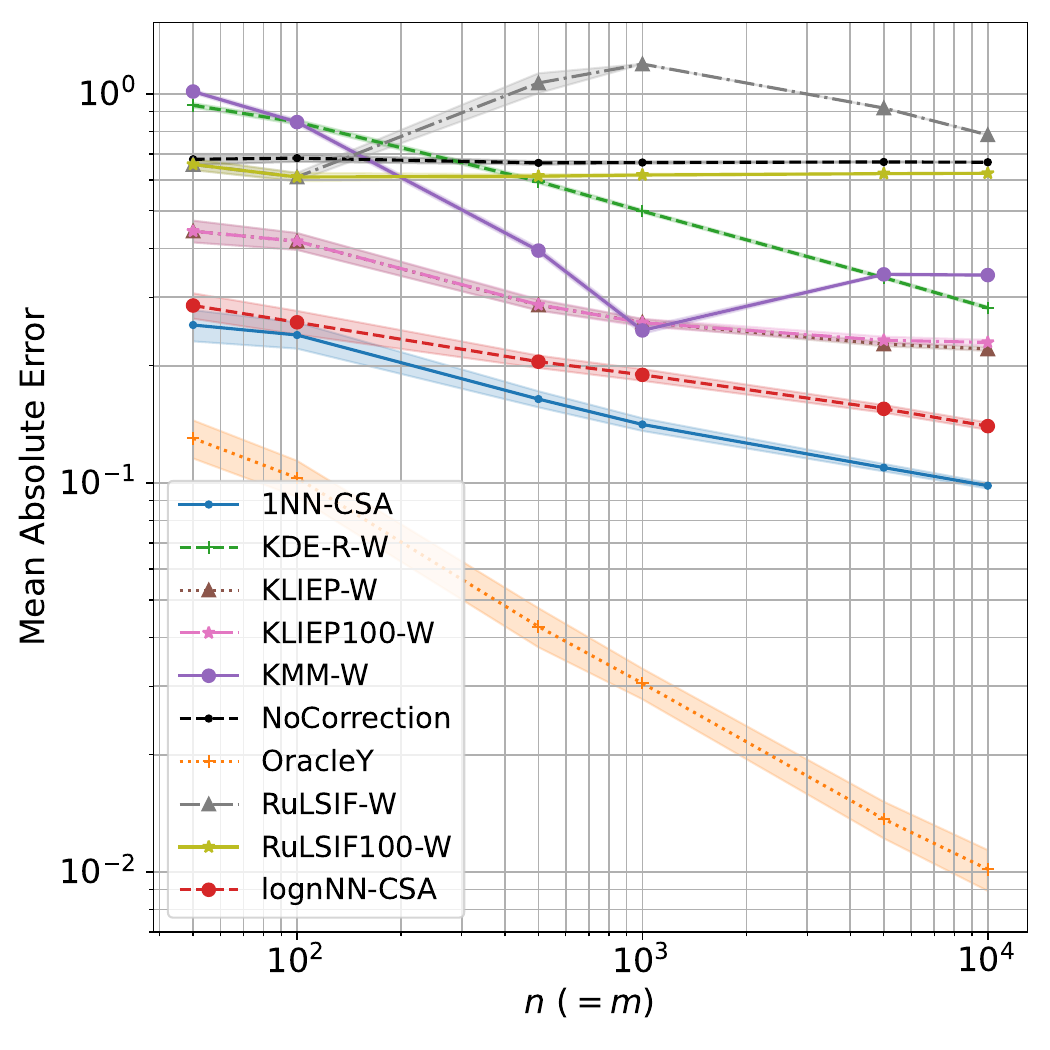}
    \caption{$d = 10$}
    \label{subfig:toy2_dim10_errors}
  \end{subfigure}
  \caption{Estimation errors for Experiment E2 (estimation of $\int (y - f_0(x) )^2\, Q(dx, dy)$)}
  \label{fig:toy2_abs_errors}
\end{figure*}

\begin{figure*}[tbp]
  \centering
  \begin{subfigure}[b]{0.24\textwidth}
    \centering
    \includegraphics[width=\textwidth]{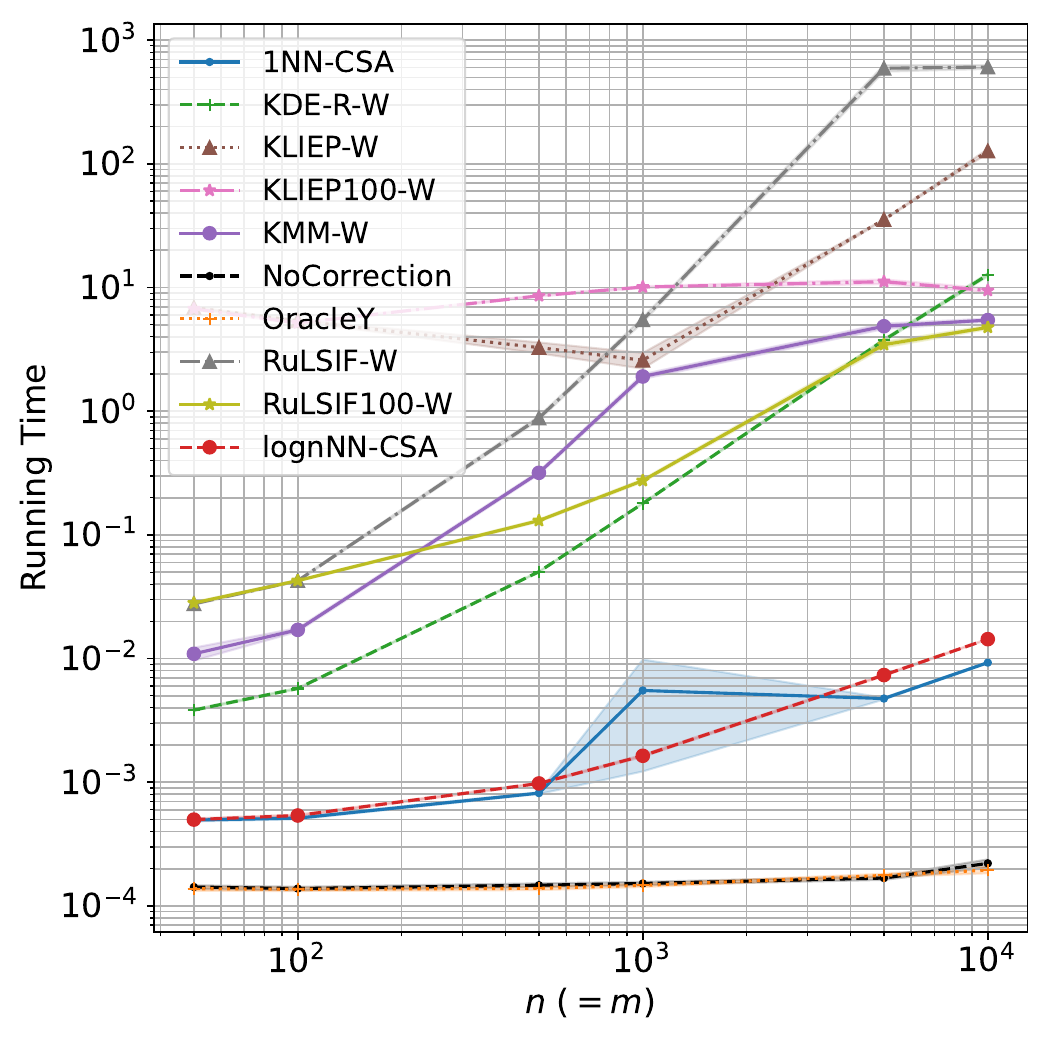}
    \caption{$d = 1$}
    \label{subfig:toy2_dim1_times}
  \end{subfigure}
  \begin{subfigure}[b]{0.24\textwidth}
    \centering
    \includegraphics[width=\textwidth]{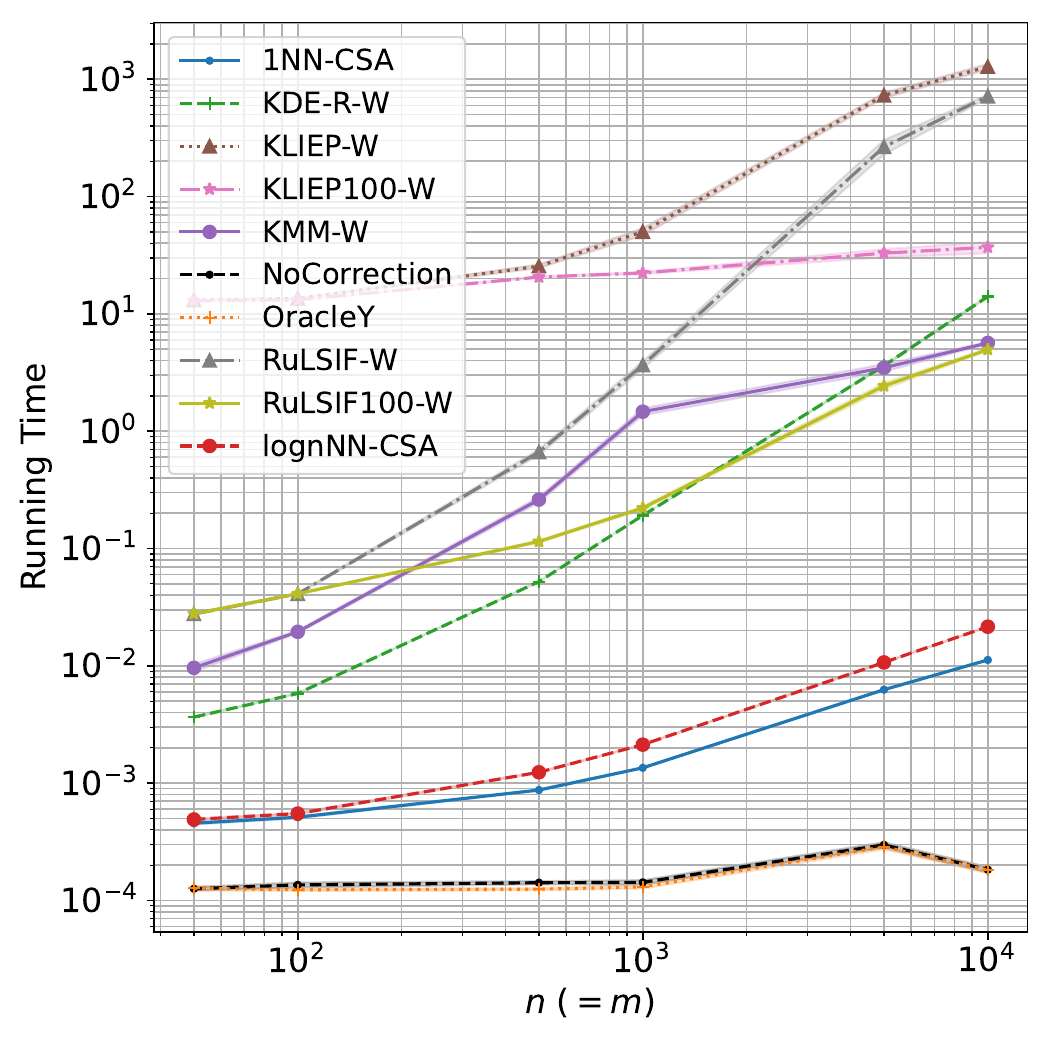}
    \caption{$d = 2$}
    \label{subfig:toy2_dim2_times}
  \end{subfigure}
  \hfill
  \begin{subfigure}[b]{0.24\textwidth}
    \centering
    \includegraphics[width=\textwidth]{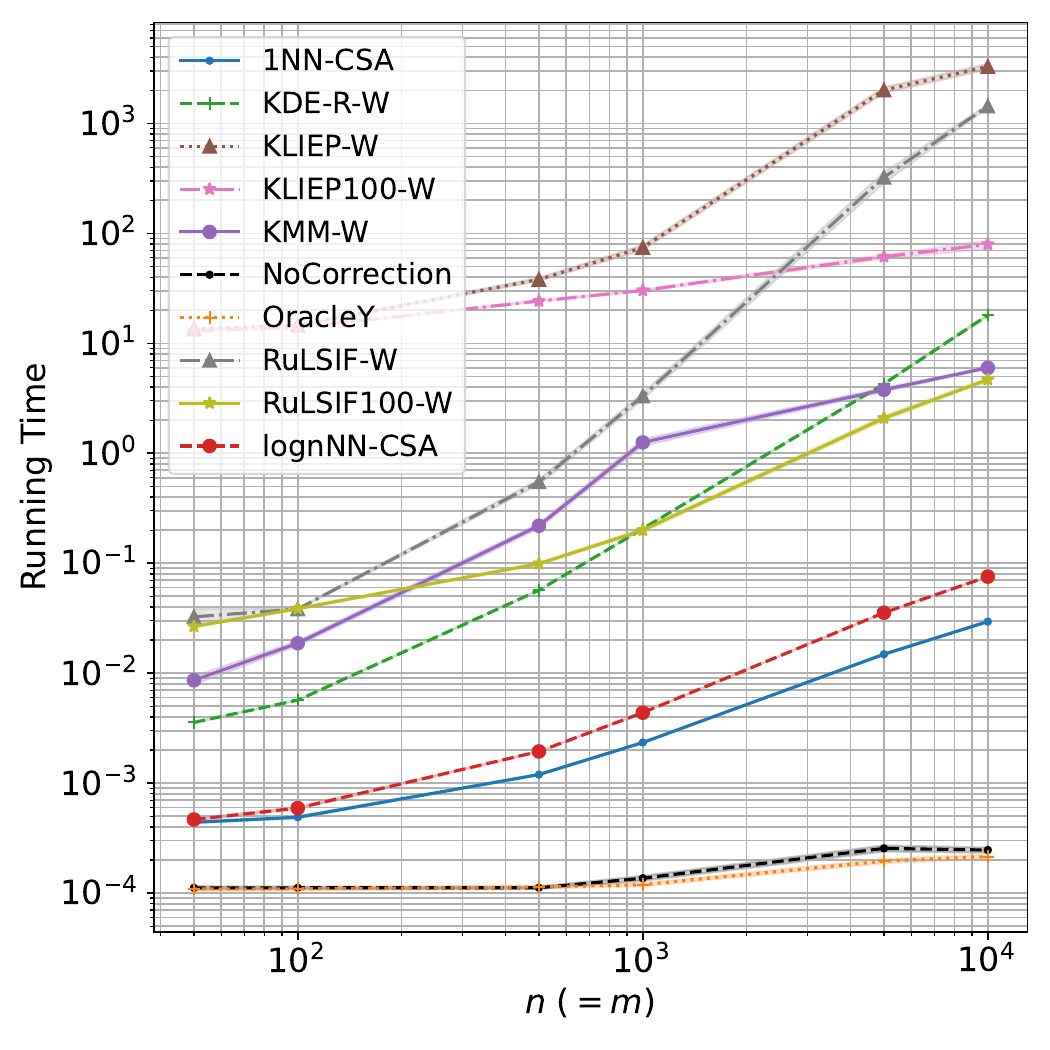}
    \caption{$d = 5$}
    \label{subfig:toy2_dim5_times}
  \end{subfigure}
  \begin{subfigure}[b]{0.24\textwidth}
    \centering
    \includegraphics[width=\textwidth]{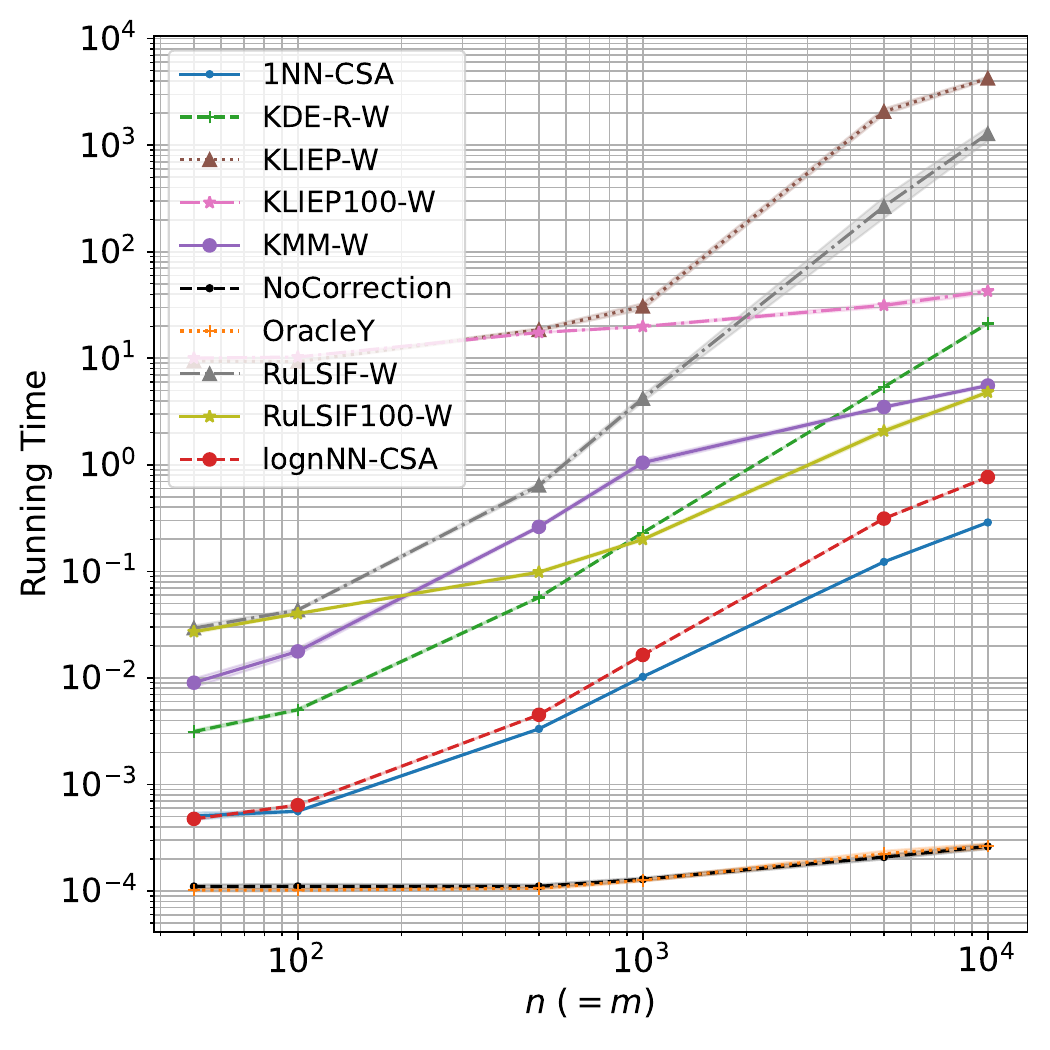}
    \caption{$d = 10$}
    \label{subfig:toy2_dim10_times}
  \end{subfigure}
  \caption{Running times in Experiment E2}
  \label{fig:toy2_abs_times}
\end{figure*}

\paragraph{Comparison of running times in Experiment E2:}
The running times in Experiments E2 (Figure~\ref{fig:toy2_abs_times}) were very similar to those in Experiment E1 (Figure~\ref{fig:toy2_abs_times}), but we can see more clearly that 1- and $\log n$-NN-CSA outperform other methods even for the smallest sample size.

\paragraph{Setup of Experiment E3 (linear regression with synthetic data):} % toy3
Next, we present experiments of linear regression. Using samples from the same source and test distributions as in Experiment E2, we perform the ordinary least squares after covariate adaptation.
More precisely, we aim to optimize the parameters $\theta \in \Re^d$ of the model $f_{\theta}\colon \Re^d \to \Re, x \mapsto \theta^\top x$ so that the Mean Squared Error (MSE) $\E[(Y^* - f_{\theta}(X^*))^2]$ in the target domain will be minimized. To do so, we minimize the MSE estimated by each covariate shift adaptation method.

\paragraph{Comparison of estimation errors for Experiment E3:} % toy3
The results are summarized in Figures~\ref{fig:toy3_errors}.\footnote{We plot the MSEs subtracted by $0.0095$ to better present the curves in the region close to the minimum population MSE $0.01$ while keeping values positive.}
KMM-W performed better than any other methods for the higher dimensions $d \in \{5, 10\}$ and the small-to-moderate sample sizes $50 \le n \le 500$, 1-NN-CSA being the second best.
For $n = 10000$, 1-NN-CSA showed performance better than or comparable to KMM-W.

\begin{figure*}[tbp]
  \centering
  \begin{subfigure}[b]{0.24\textwidth}
    \centering
    \includegraphics[width=\textwidth]{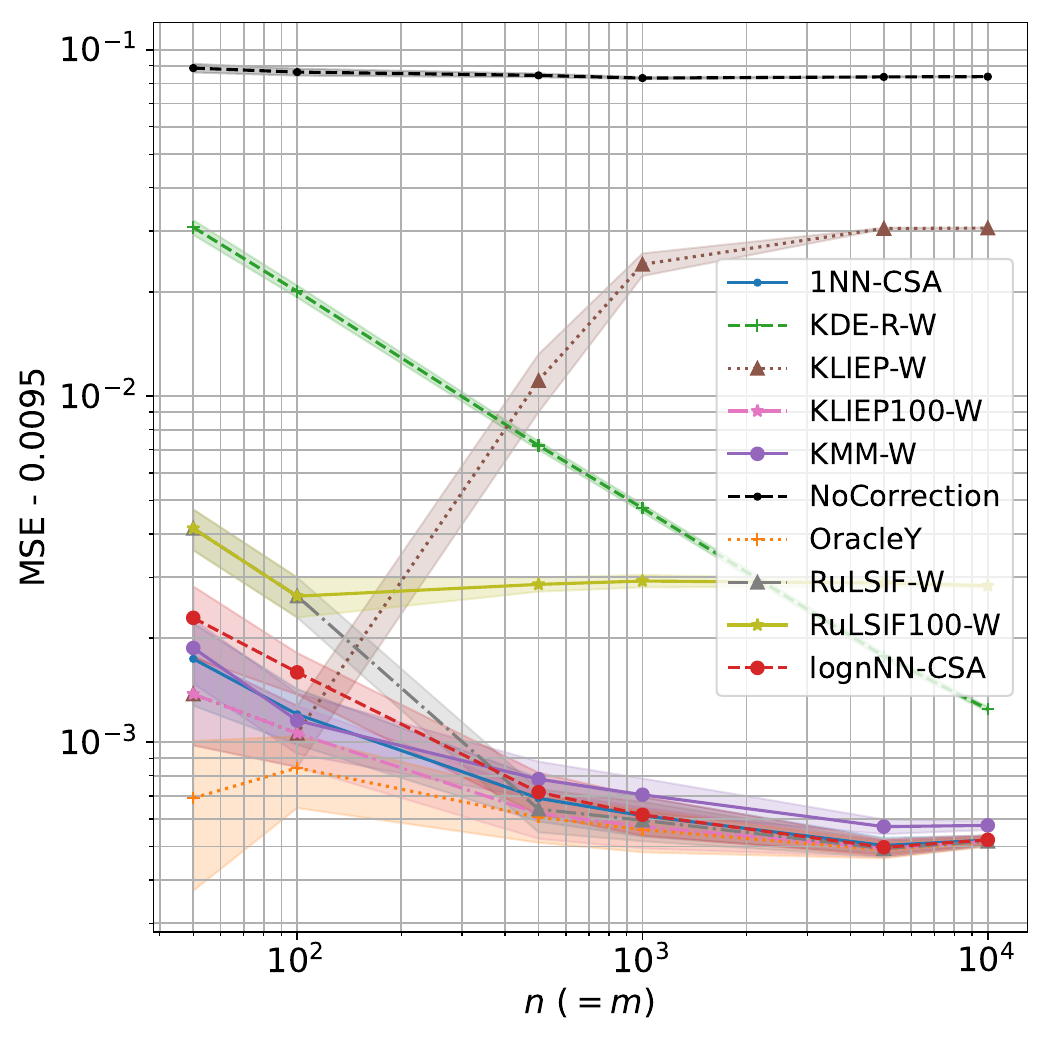}
    \caption{$d = 1$}
    \label{subfig:toy3_dim1_errors}
  \end{subfigure}
  \begin{subfigure}[b]{0.24\textwidth}
    \centering
    \includegraphics[width=\textwidth]{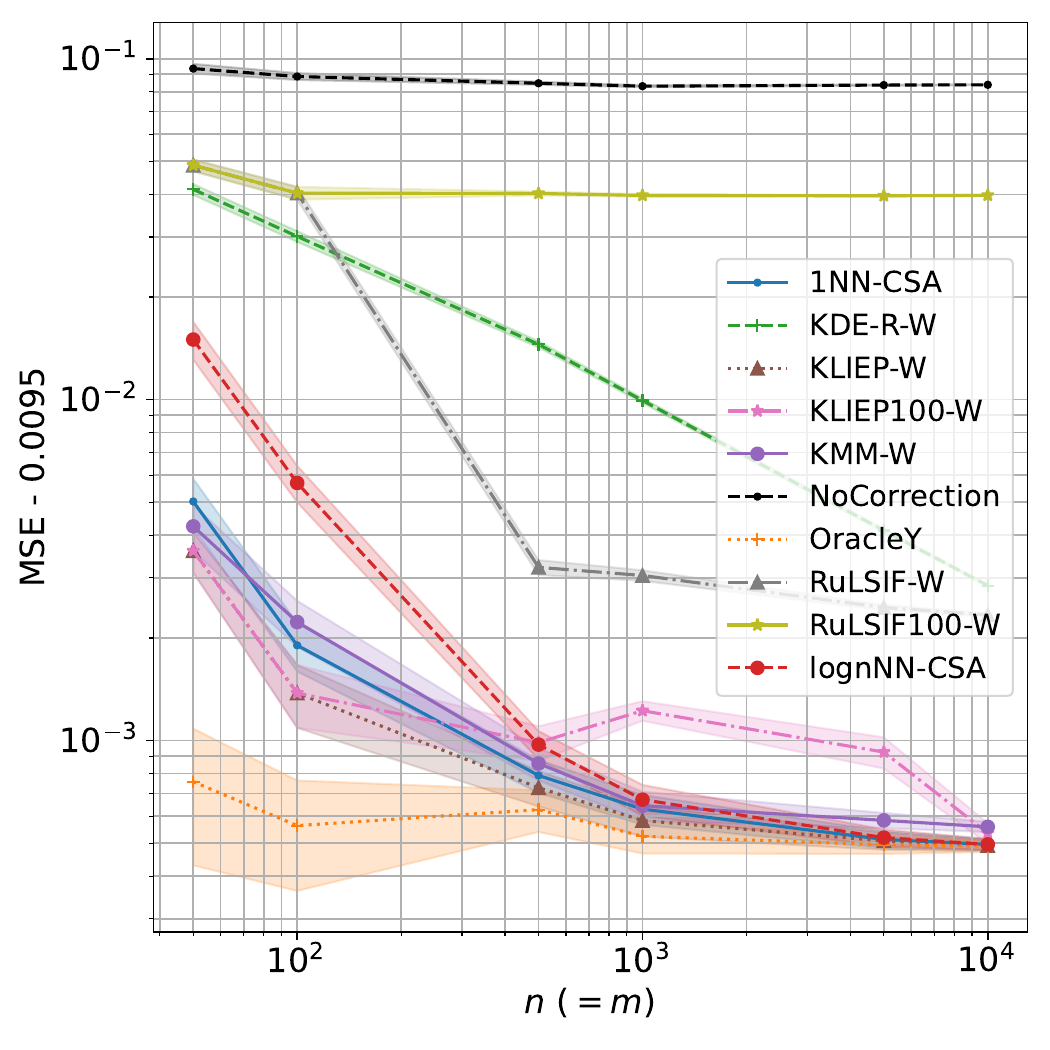}
    \caption{$d = 2$}
    \label{subfig:toy3_dim2_errors}
  \end{subfigure}
  \hfill
  \begin{subfigure}[b]{0.24\textwidth}
    \centering
    \includegraphics[width=\textwidth]{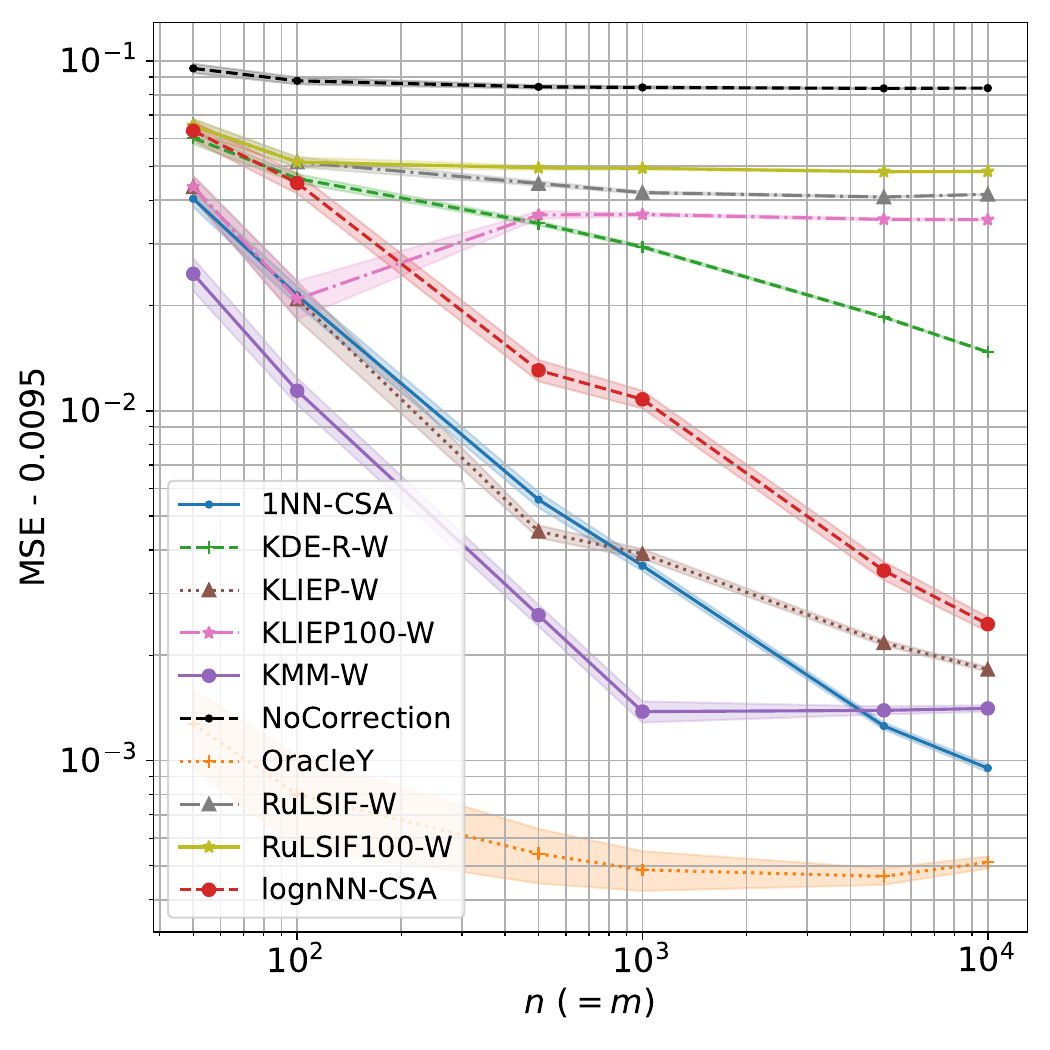}
    \caption{$d = 5$}
    \label{subfig:toy3_dim5_errors}
  \end{subfigure}
  \begin{subfigure}[b]{0.24\textwidth}
    \centering
    \includegraphics[width=\textwidth]{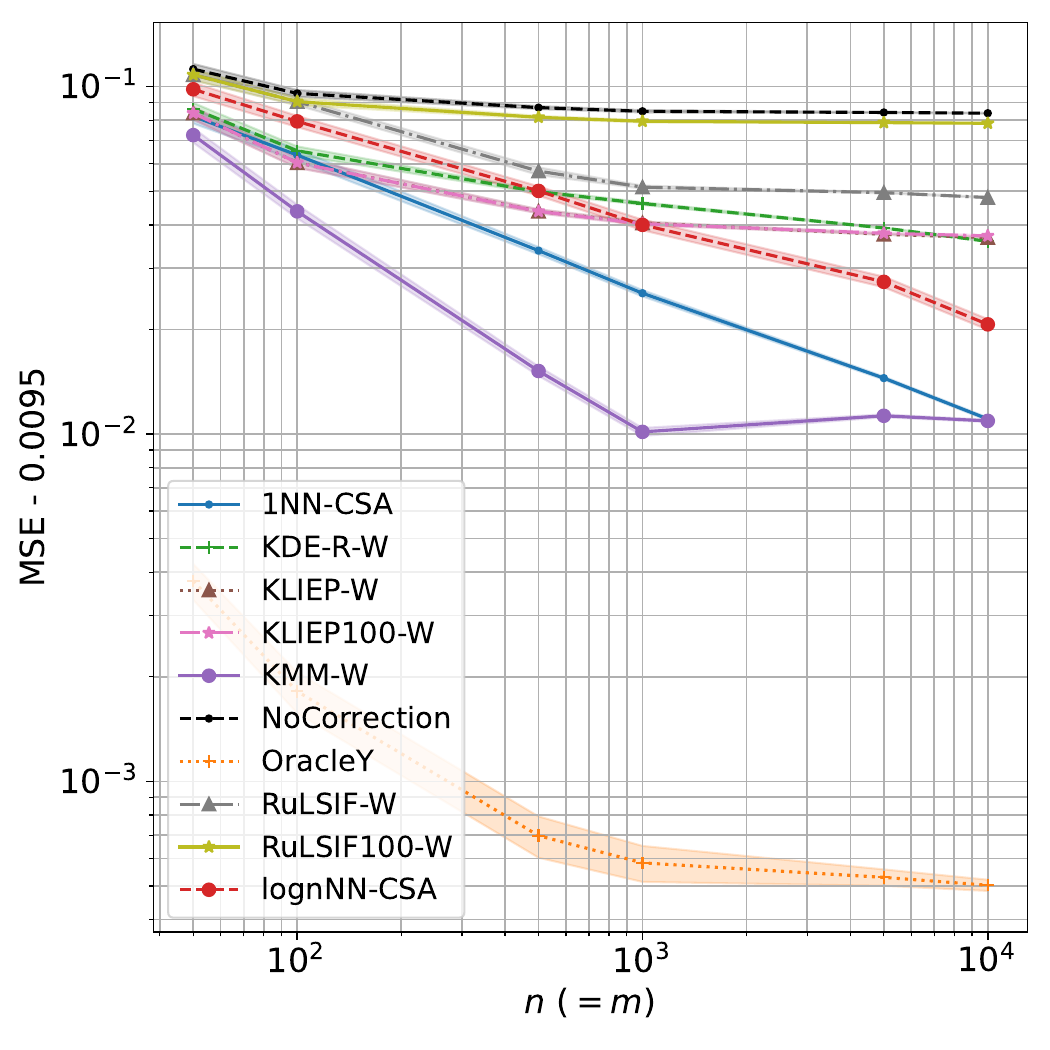}
    \caption{$d = 10$}
    \label{subfig:toy3_dim10_errors}
  \end{subfigure}
  \caption{Mean Squared Errors (MSE) (subtracted by $0.0095$) for Experiment E3 (linear regression)}
  \label{fig:toy3_errors}
\end{figure*}

\paragraph{Comparison of running times in Experiment E3:} % toy3
As in Experiments E1--E2, 1-NN-CSA and $\log n$-NN-CSA finished their computations faster than the other adaptation methods by large margins (Figure~\ref{fig:toy3_times}).

\begin{figure*}[tbp]
  \centering
  \begin{subfigure}[b]{0.24\textwidth}
    \centering
    \includegraphics[width=\textwidth]{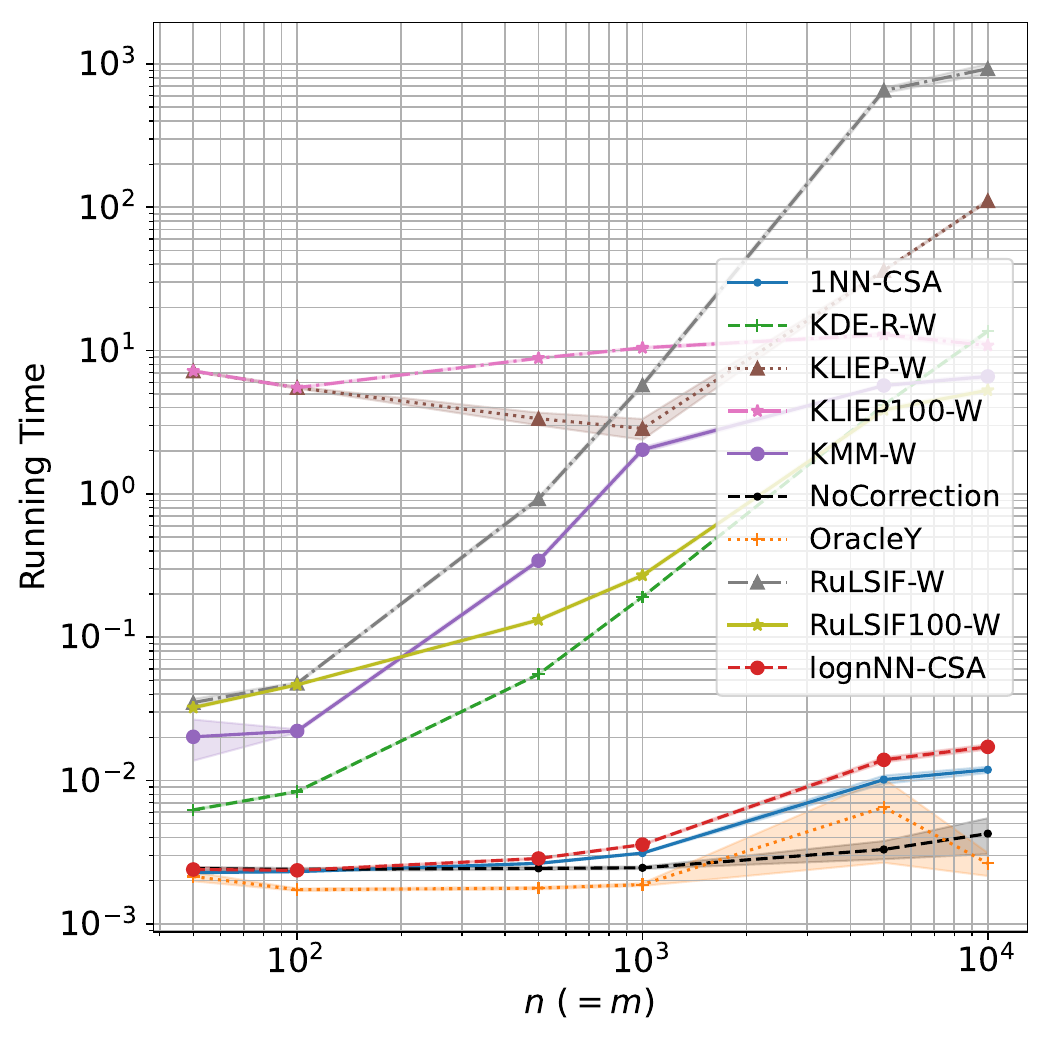}
    \caption{$d = 1$}
    \label{subfig:toy3_dim1_times}
  \end{subfigure}
  \begin{subfigure}[b]{0.24\textwidth}
    \centering
    \includegraphics[width=\textwidth]{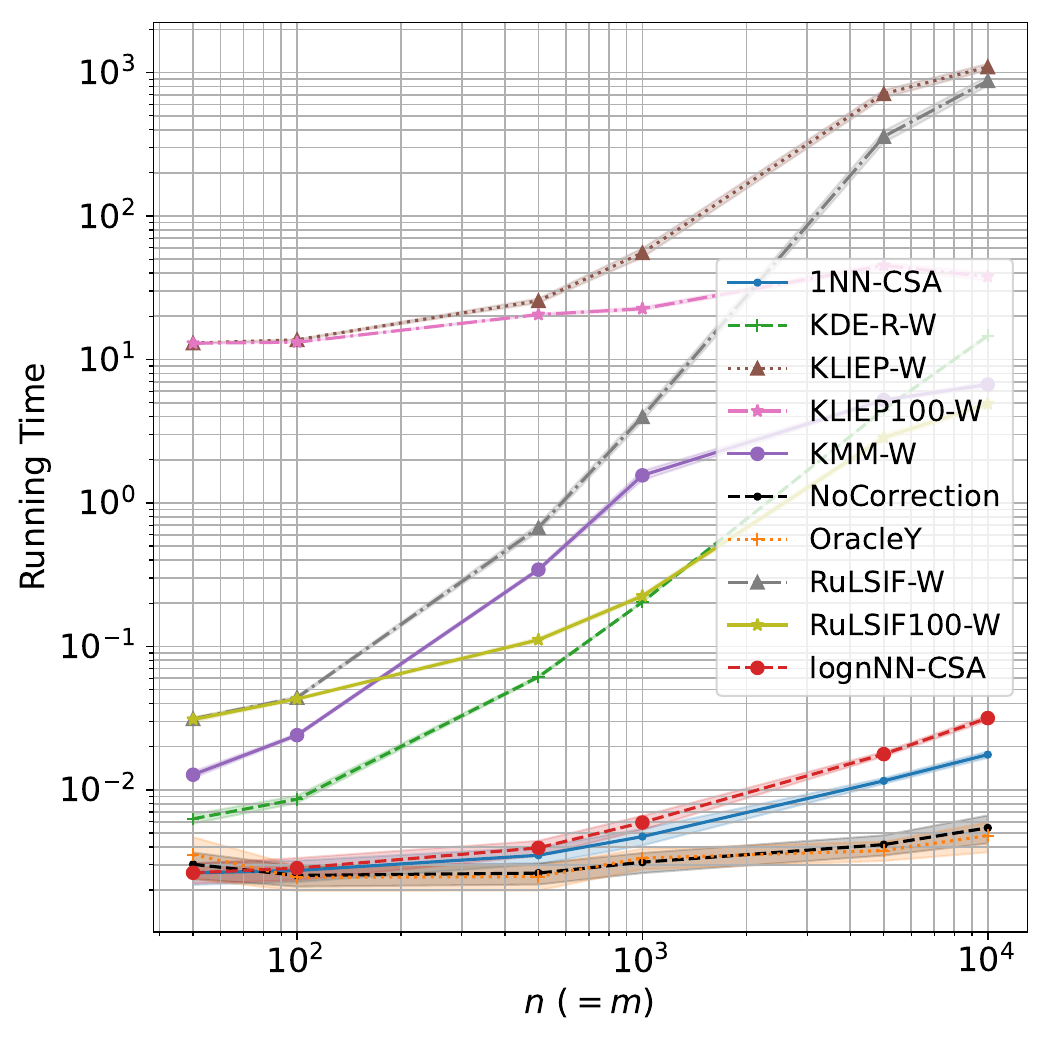}
    \caption{$d = 2$}
    \label{subfig:toy3_dim2_times}
  \end{subfigure}
  \hfill
  \begin{subfigure}[b]{0.24\textwidth}
    \centering
    \includegraphics[width=\textwidth]{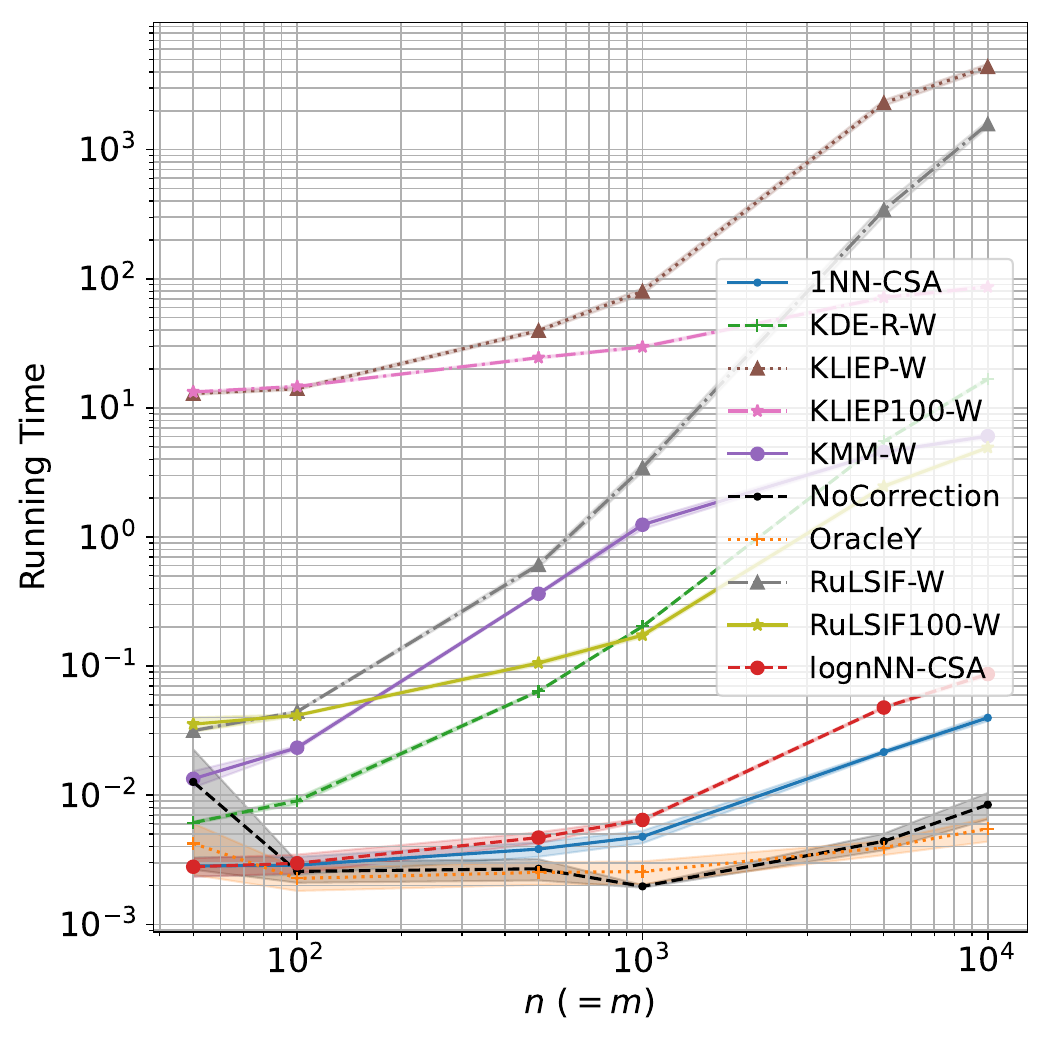}
    \caption{$d = 5$}
    \label{subfig:toy3_dim5_times}
  \end{subfigure}
  \begin{subfigure}[b]{0.24\textwidth}
    \centering
    \includegraphics[width=\textwidth]{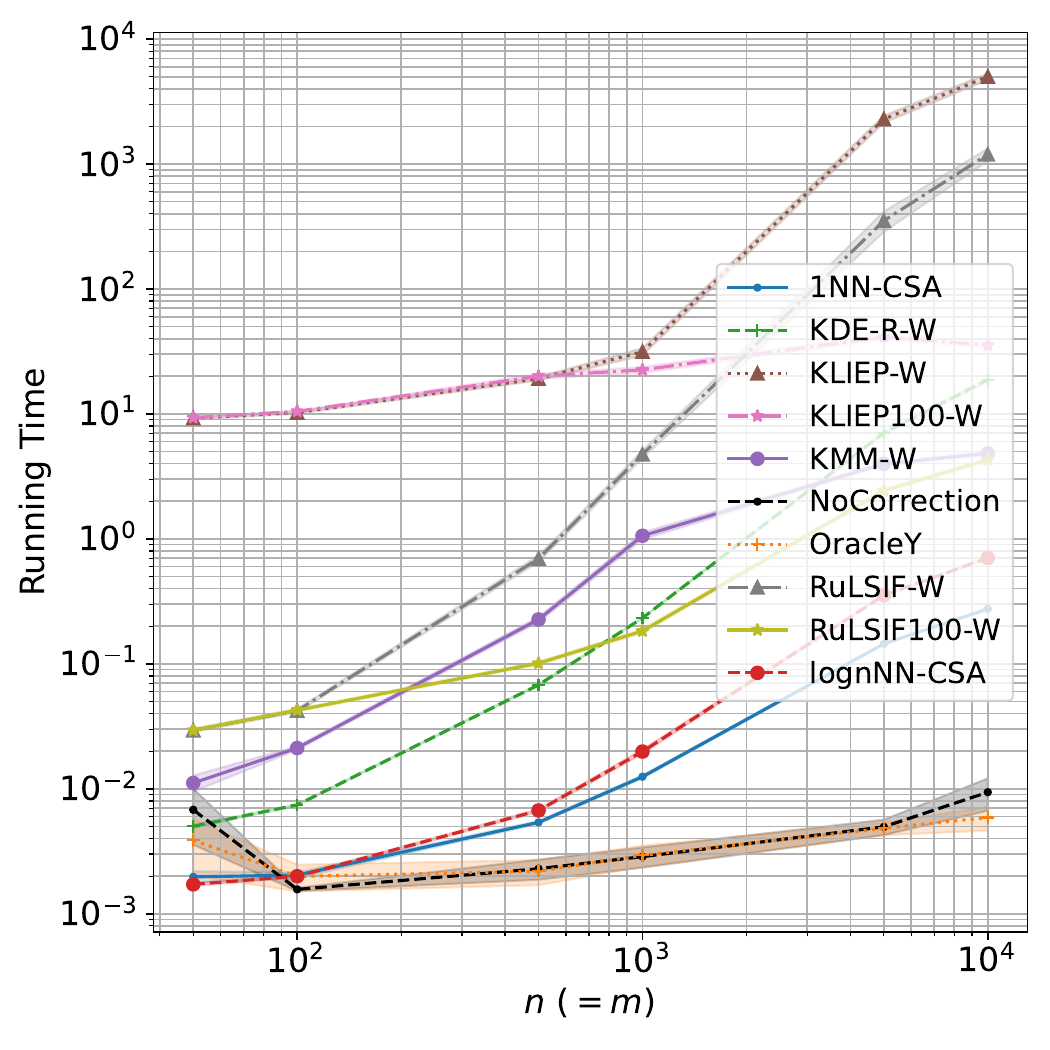}
    \caption{$d = 10$}
    \label{subfig:toy3_dim10_times}
  \end{subfigure}
  \caption{Computation times (seconds) in Experiment E3}
  \label{fig:toy3_times}
\end{figure*}

In Experiments E1--E3, the proposed methods, 1- and $\log n$-NN-CSA were able to finish computation much faster than other adaptation methods without compromising on the statistical performance.
$\log n$-NN-CSA did not show advantages in accuracy, with increased computation costs.
We can conclude that 1-NN-CSA is preferred over $\log n$-NN-CSA.
A reason that we were not able to conduct experiments with larger sample sizes than $10000$ is that the existing adaptation methods have too demanding computational requirements. For instance, the running times of RuLSIF-W in Figure~\ref{subfig:toy3_dim5_times} grows about 100 times as the sample size increases by 10 times, taking more than $10^3$ seconds for $n = 10000$. For $n = 10^5$, we would need at least $10^3 \times 100$ seconds, that is $27$ hours of compute for a single run.
In contrast, the time complexity of 1-NN-CSA being $\mathcal{O}(n \log n)$ and its running time less than one second for $n=10^4$, we can estimate its running time for $n = 10^5$ as $1 \times (10^5 \log 10^5)/(10^4 \log 10^4) = 12.5$ seconds.
1-NN-CSA would stay feasible in applications of even larger scales.

The previous methods construct the distance matrix between pairs of data points, which takes running time and memory space quadratic in the sample size.
Additionally, RuLSIF-W  computes the inverse of the distance matrix, taking cubic running time. KMM-W and KLIEP-W solve convex optimization problems with iterative procedures, for which the implementations from \citet{demathelin2021adapt} use stopping criteria based on objective function values. This resulted in good accuracy and milder growth in running time in our experiments. However, tuning the solvers can be involved in practice.
In contrast, $k$-NN-CSA does not have such subtle issues around optimization solvers: we only have to perform nearest neighbor search.

In all cases, we can observe that 1-NN-CSA showed clear power-law, with nearly straight lines in the logarithmic scales.
This is a significant advantage in predicting returns when one invests on increasing the sample size.

\paragraph{Experiment E4 (linear regression and logistic regression with benchmark datasets):}
We use regression benchmark datasets, \texttt{diabetes}\footnote{Available at \url{https://archive.ics.uci.edu/ml/index.php}.\label{footnote:uci}}, \texttt{california}~\citep{california_data}\footnote{Available at \url{https://www.dcc.fc.up.pt/~ltorgo/Regression/cal_housing.html}.} and
classification datasets, \texttt{twonorm}~\citep{delve_datasets}\footnote{Available at \url{https://www.cs.utoronto.ca/~delve/data/datasets.html}.\label{footnote:delve}} and \texttt{breast\_cancer}$^{\ref{footnote:uci}}$.
We apply the ridge regression and the logistic regression, respectively. The evaluation metric is the mean squared error for the regression tasks and the classification accuracy for classification tasks.
We synthetically introduce covariate shift by subsampling test data.
See Appendix~\ref{app:sec:experiment_details} for more details.

\paragraph{Remark:}
For fair comparison, the benchmark experiments presented in this paper follow the standard protocol used in the literature as similarly done in previous research~\citep{gretton2008kmm,kanamori2009ulsif,yamada2013rulsif,sugiyama2007direct,sugiyama2008direct}: we apply biased resampling to synthetically simulate a target dataset under covariate shift.
It is thus important to note that they are not completely real-world data.
Nevertheless, this ensures that the methods are tested in isolation from other types of distribution shifts while using real data for the source covariate distribution as well as the conditional distributions.

\begin{table}[btp]
  \caption{MSE/accuracy for regression/classification benchmark datasets. We repeat the experiment using 50 different random subsamples and calculate the average scores (and standard errors). The results comparable to the best in terms of Wilcoxon's signed rank test with significance level $1\%$ are shown in bold fonts.}
  \label{tab:benchC0}
  \centering
   \begin{tabular}{ccccc}
    \toprule
    & \multicolumn{2}{c}{Regression (MSE)} & \multicolumn{2}{c}{Classification (accuracy)} \\
    \cmidrule(r){2-3} \cmidrule(r){4-5}
            & \texttt{diabetes}    & \texttt{california}    & \texttt{breast\_cancer}    & \texttt{twonorm}                  \\
    \midrule
1NN-CSA     & 3470 (35)          & \textbf{0.146} (0.001) & \textbf{0.9633} (0.002) & 0.9327 (0.002)           \\
lognNN-CSA  & 3605 (40)          & 0.150 (0.001)          & 0.9595 (0.002)          & 0.9293 (0.002)           \\
KDE-R-W     & 3673 (52)          & 3.864 (1.067)          & 0.9596 (0.002)          & 0.5260 (0.009)           \\
KMM-W       & 3831 (60)          & 3.702 (1.160)          & 0.9594 (0.002)          & \textbf{0.9583} (0.001)  \\
KLIEP-W     & \textbf{3221} (31) & 2.896 (0.798)          & \textbf{0.9648} (0.002) & 0.9482 (0.001)           \\
KLIEP100-W  & \textbf{3223} (31) & 3.034 (0.843)          & \textbf{0.9648} (0.002) & 0.9480 (0.001)           \\
RuLSIF-W    & 3235 (31)          & 3.039 (0.843)          & 0.7794 (0.015)          & 0.9512 (0.001)           \\
RuLSIF100-W & 3238 (31)          & 3.045 (0.844)          & 0.7794 (0.015)          & 0.9539 (0.001)           \\
    \bottomrule
  \end{tabular}
\end{table}

\begin{table}[bt]
  \caption{Total running times in seconds spent for the training including the hyper-parameter tuning (if any) for benchmark datasets. We repeat the experiment using 50 different random subsamples dataset and calculate the average running times (and standard errors). The results comparable to the best in terms of Wilcoxon's signed rank test with significance level $1\%$ are shown in bold.}
  \label{tab:times_total_benchmark}
  \centering
  \begin{tabular}{cccccc}
    \toprule
            & \texttt{diabetes}        & \texttt{california}      & \texttt{breast\_cancer}  & \texttt{twonorm}     \\
    \midrule
1NN-CSA     & \textbf{0.0015} (0.0000) & \textbf{0.0084} (0.0001) & \textbf{0.0036} (0.0000) & \textbf{0.0051} (0.0000) \\
lognNN-CSA  & 0.0016 (0.0000)          & 0.0128 (0.0001)          & 0.0037 (0.0000)          & 0.0052 (0.0000)          \\
KDE-R-W     & 0.0078 (0.0000)          & 0.2121 (0.0008)          & 0.0117 (0.0000)          & 0.0124 (0.0000)          \\
KMM-W       & 0.0373 (0.0015)          & 0.4067 (0.0038)          & 0.0542 (0.0014)          & 0.0220 (0.0006)          \\
KLIEP-W     & 7.602 (0.051)            & 29.98 (0.34)             & 8.67 (0.07)              & 8.86 (0.16)              \\
KLIEP100-W  & 7.501 (0.045)            & 16.91 (0.07)             & 8.68 (0.07)              & 8.26 (0.10)              \\
RuLSIF-W    & 0.0575 (0.0014)          & 1.686 (0.011)            & 0.0529 (0.0020)          & 0.2014 (0.0016)          \\
RuLSIF100-W & 0.0401 (0.0007)          & 0.1237 (0.0004)          & 0.0454 (0.0014)          & 0.0391 (0.0002)          \\
    \bottomrule
  \end{tabular}
\end{table}

\paragraph{Results for Experiment E4:}Table~\ref{tab:benchC0} shows the obtained MSEs and classification accuracies.
$1$-NN-CSA and $\log n$-NN-CSA gave the best performance for \texttt{california} and performances comparable to the best for \texttt{breast\_cancer}.
For the other datasets, different methods performed the best depending on the dataset.
On the other hand, in terms of running time, 1NN-CSA was consistently faster than the previous methods (Table~\ref{tab:times_total_benchmark}).

Our experiments show that the proposed method is almost always faster than the previous methods and gives great accuracy in many cases, even though it is not always the best.
1-NN-CSA is highly recommended as an off-the-shelf method applicable even in larger scales, although the previous methods such as KMM-W, KLIEP-W, and RuLSIF-W should not be neglected, as far as the computational budget allows.
The times spent for adaptation are summarized in Table~\ref{tab:times_total_benchmark},
showing that the proposed methods $1$-NN-CSA and $\log n$-NN-CSA are much faster than other methods.

\section{Conclusion}
We proposed a $k$-NN-based covariate shift adaptation method.
We provided error bounds, which suggest setting $k = 1$ is among the best choices.
This resulted in a scalable non-parametric method with no hyper-parameter.
For future research directions, one could complete our results for the parametric inference on the target domain, in particular for finding the asymptotic distribution of $M$-estimators. For the average treatment effect, \citet{abadie2006large} derived 
asymptotic normality of their estimator and it could be interesting to get a similar result in our context. Investigating non-parametric estimation on the target domain could be also an interesting direction. However, non-parametric estimators computed with the source sample can be already optimal when the ratio of densities is bounded. See for instance \citet{ma2023optimally} in the reproducing kernel Hilbert space framework.
It could be then interesting to extend our result to cases with unbounded density ratios.
Finally, it may be interesting to extend our approach with approximate nearest neighbor methods for further scalability.

\section*{Acknowledgement}
Experiments presented in this paper were carried out using the Grid'5000 testbed, supported by a scientific interest group hosted by Inria and including CNRS, RENATER and several Universities as well as other organizations (see \url{https://www.grid5000.fr}).
IY was supported by the \emph{Allocation d'Installation Scientifique (AIS) 2023} from Rennes M\'etropole.

\newpage

\bibliographystyle{chicago}
\bibliography{main.bib}

%\begin{appendix}
\appendix
\section{Preliminary results}

 The first preliminary result is concerned about the order of magnitude of $P_X(B (x, \tau ))$ for which we obtain a lower bound and an upper bound.
 
\begin{lem}\label{lem:aux1}
Under \ref{cond:reg0}, \ref{cond:reg1}, and \ref{cond:reg2}, it holds, for every $ x\in S_X$ and $\tau \in [0, T]$,
\begin{equation*}
  M_{1,d}  \tau ^{d} \leq P_X(B (x, \tau )) \leq   M_{2,d}  \tau ^{d} ,
\end{equation*}
with $M_{1,d}  = cb_X V_d $ and $M_{2,d} = U_XV_d  $.
\end{lem}
\begin{proof}
The proof of the lower bound follows from
\begin{equation*}
  P_X(B (x, \tau ))  = \int _{B (x, \tau ) \cap S_X } p_X(y) dy \geq b_X \int _{B (x, \tau ) \cap S_X }  dy \geq b_X c \int _{B (x, \tau ) }  dy,
\end{equation*}
where we have used \ref{cond:reg2} to get the first inequality and then \ref{cond:reg1} to obtain the second one. We conclude by change of variable $y = x +\tau u$. The proof of the upper bound is similar:
\begin{equation*}
  P_X(B (x, \tau ))  = \int _{B (x, \tau ) \cap S_X } p_X(y) dy \leq U_X \int _{B (x, \tau ) \cap S_X }  dy \leq U_X \int _{B (x, \tau ) } dy.
\end{equation*}
\end{proof}

The same type of result can be obtained for $ P_X(B (x_1, \tau_1 )\cup B (x_2, \tau_2 ) ) $ as follows.

\begin{lem}\label{lem:aux2}
Under \ref{cond:reg0}, \ref{cond:reg1}, and \ref{cond:reg2}, it holds, for every $(x_1,x_2) \in S_X\times S_X $ and $(\tau_1, \tau_2) \in [0, T]^{2}$,
\begin{equation*}
  \frac 1 2 M_{1,d}  ( \tau_1 ^{d} + \tau _2 ^d ) \leq  P_X(B (x_1, \tau_1 )\cup B (x_2, \tau_2 ) )  \leq   M_{2,d} ( \tau_1 ^{d}+ \tau_2 ^{d}) ,
\end{equation*}
with $M_{1,d}  = cb_X V_d $ and $M_{2,d} = U_XV_d  $.
\end{lem}
\begin{proof}
The proof of the upper bound follows from the union bound and Lemma \ref{lem:aux1}. For the lower bound, start noting that  for any events $A$ and $B$,
 $ 1_{A\cup B} \ge (1_{A} + 1_{B})/ 2$. Then the conclusion follows from Lemma \ref{lem:aux1}.

\end{proof}

Based on the previous results, an upper and a lower bound are obtained on the moments of the nearest neighbor radius $\hat{\tau}_{n,k,x}$. A similar upper bound is stated as Lemma~3 in \cite{leluc_nn}).

\begin{lem}\label{lem:aux3}
Let $q$ be a positive real number.
Under \ref{cond:reg0}, \ref{cond:reg1}, and \ref{cond:reg2}, there exist two positive real numbers $c_{q,d}$ and $C_{q,d}$, depending
on $q$, $d$ and on the distribution of $X$ such that
\begin{equation}\label{lub}
c_{q,d}\frac{k^{q/d}}{(n+1)^{q/d}}\leq \E \hat{\tau}_{n,k,x}^q\leq C_{q,d}\frac{k^{q/d}}{(n+1)^{q/d}}.
\end{equation}
A more precise expression of the two constants are
$$c_{q,d}=\frac{M_{2,d}^{-q/d}}{2 \Gamma\left(1+[q/d]\right)},\quad C_{q,d}=2\Gamma\left(1+[q/d]\right)M_{1,d}^{-q/d},$$
where $[x]$ denotes the integer part of the real number $x$.
\end{lem}

\begin{proof}
We have $\hat{\tau}_{n,k,x}=Z_{(k)}(x)$ the kth-order statistics of $Z_i(x)=\vert x-X_i\vert$. Moreover, for any measurable and non-negative function $f$,
$$\E f\left(Z_{(k)}(x)\right)=\E f\circ F_x^{-1}\left(U_{(k)}\right),$$
where $U_{(k)}$ is the kth-order statistics of a $n$ sample of uniform random variables and since $F_x(z)=\P\left(X_1\in B(x,z)\right)\in \left[M_{1,d}z^d,M_{2,d}z^d\right]$,
$$F_x^{-1}(u)=\inf\left\{z\in \R: F_x(z)\geq u\right\}\in \left[\frac{u^{1/d}}{M_{2,d}^{1/d}},\frac{u^{1/d}}{M_{1,d}^{1/d}}\right].$$
Note that the range of $Z_i(x)$ is $[0,diam(S_X)]$ and we use a constant $c$ in the definition of $M_{1,d}=cb_X V_d$ such that
$$\inf_{x\in S_X}\lambda_d\left(B(x,z)\right)\geq c V_d z^d\mbox{ for } 0\leq z\leq diam(S_X).$$
If  $z=u^{1/d}/M_{1,d}^{1/d}\geq diam(S_X)$, we have $F_x(z)=1\geq u$ and we still have $F_x^{-1}(u)\leq z$.
 
Moreover, if $g$ is measurable and nonnegative,
$$\E g\left(U_{(k)}\right)=n! \int g(u_k)\mathds{1}_{0\leq u_1\leq\cdots \leq u_n\leq 1}du_1\cdots du_n=\frac{\Gamma(n+1)}{\Gamma(k)\Gamma(n-k+1)}\int_0^1 g(u)u^{k-1}(1-u)^{n-k}du.$$
When $f(z)=z^q$ for some $q>0$, we get
$$\E\left[Z_{(k)}^q\right]\leq M_{1,d}^{-q/d}\E\left[U_{(k)}^{q/d}\right]=M_{1,d}^{-q/d}\frac{\Gamma(n+1)\Gamma(k+q/d)}{\Gamma(k)\Gamma(n+q/d+1)},$$
$$\E\left[Z_{(k)}^q\right]\geq M_{2,d}^{-q/d}\E\left[U_{(k)}^{q/d}\right]=M_{2,d}^{-q/d}\frac{\Gamma(n+1)\Gamma(k+q/d)}{\Gamma(k)\Gamma(n+q/d+1)}.$$
For $x\geq 1$ and $s>0$, let $N_{1,s}=\inf_{x\geq 1}\frac{\Gamma(x+s)}{x^s\Gamma(x)}$ and $N_{2,s}=\sup_{x\geq 1}\frac{\Gamma(x+s)}{x^s\Gamma(x)}$. 
We then get 
$$M_{2,d}^{-q/d}\frac{N_{1,q/d}}{N_{2,q/d}}\frac{k^{q/d}}{(n+1)^{q/d}}\leq \E\left[Z_{(k)}(x)^q\right]\leq M_{1,d}^{-q/d} \frac{N_{2,q/d}}{N_{1,q/d}}\frac{k^{q/d}}{(n+1)^{q/d}}.$$
By Wendel's inequality~\citep{wendel}, for $s\in (0,1)$, we have $N_{1,s}\geq \inf_{x\geq 1}\left(\frac{x}{x+s}\right)^{1-s}\geq 1/2$ and $N_{2,s}\leq 1$. 
For $s\geq 1$, using the equality $\Gamma(z+1)=z\Gamma(z)$,  one can deduce that 
$$N_{1,s}\geq 1/2,\quad N_{2,s}\leq \Gamma(2+[s]).$$
Indeed if $s=s'+\ell$ with $\ell\in \N$ and $0\leq s'<1$, 
$$\frac{\Gamma(x+s)}{x^s \Gamma(x)}=\prod_{j=1}^{\ell}\left(1+\frac{j+s'-1}{x}\right) \frac{\Gamma(x+s')}{x^{s'}\Gamma(x)}$$
and 
$$1\leq \prod_{j=1}^{\ell}\left(1+\frac{j+s'-1}{x}\right)\leq \prod_{j=1}^{\ell}\left(1+j\right)=\Gamma(\ell+2).$$
This completes the proof of Lemma \ref{lem:aux3}.
\end{proof}

\section{Proofs of the results on the marginal sampling error (Section \ref{sec:sampling})}\label{app:sec:sampling_proof}

\subsection{Proof of Proposition \ref{prop:slln_clt}}\label{app:sec:slln_clt}

The proof relies on  the Lindeberg central limit theorem as given in Proposition 2.27 in \citet{van2000asymptotic} conditionally to $\mathcal F_n$.
We need to show the two properties:
\begin{align*}
& m^{-1} \sum_{i=1} ^ m    \mathbb E [ ( h( X_{i}^*, Y_{n,i} ^* )   - \mathbb E  [  h( X_{i}^*, Y_{n,i} ^* )  \given \mathcal F_n ] ) ^2 \given \mathcal F_n ]  \to V,   \\
& m^{-1} \sum_{i=1} ^ m    \mathbb E [  h( X_{i}^*, Y_{n,i} ^* )    ^2  1_ { \{ \lvert   h( X_{i}^*, Y_{n,i} ^* )  \rvert > \epsilon \sqrt n \}} \given \mathcal F_n ]  \to 0,
\end{align*}
where each convergence needs to happen with probability $1$. Equivalently, using that $( X_{i}^*, Y_{n,i} ^* )  _{i=1,\ldots, m}$ is identically distributed according to $\hat Q$, we need to show that
\begin{align*}
& \hat Q(h^2) - \hat Q(h)^2 \to V,\\
& \hat Q(h^2 1_ { \{ \lvert h \rvert >\epsilon \sqrt n \}}) \to 0 \text{ for each $\epsilon > 0$}.
\end{align*}
The first result is a direct consequence of the assumption. Fix $M>0$. For all $n$ sufficiently large, we have $M \le \epsilon \sqrt n$, implying that
\begin{equation*}
\hat Q  (  h ^2  1_ { \{ \lvert h \rvert >\epsilon \sqrt n \}} )  \leq
\hat Q  (  h ^2  1_ { \{ \lvert h \rvert > M\}} ),
\end{equation*}
which converges to $  Q (h^2 1_ { \{ \lvert h \rvert > M \}}  )  $ by assumption. Since $Q(h^2)$ is finite, one can choose $M$ large enough to make $  Q (h^2 1_ { \{ \lvert h \rvert > M \}}  )  $ arbitrarily small.

\subsection{Proof of Proposition \ref{expsampling}}\label{app:sec:expsampling_proof}

Set $Z_{n,i}^{*}=h\left(Y_{n,i}^{*},X_i^{*}\right)-\hat{Q}(h)$. We have $\left\vert Z_{n,i}^{*}\right\vert\leq 2 U_h$ and
$\mbox{Var }\left(Z_{n,i}^* \,\middle\vert\, \mathcal{F}_n\right)=\hat v_n$.
Note that $\hat{Q}^{*}(h) -\hat{Q}(h)=\frac{1}{m}\sum_{i=1}^m Z_{n,i}^{*}$.
Bernstein's concentration inequality leads to
$$\P\left(\left\vert \hat{Q}^{*}h -\hat{Q} h\right\vert>u \,\middle\vert\, \mathcal{F}_n\right)\leq \exp\left(-\frac{1/2 u^2m}{\hat{Q}(h^2)-(\hat{Q}(h))^2+2/3 U_h u}\right).$$
Then setting
$$\hat{u}_h(\delta)=\frac{4/3 U_h}{m}\log(2/\delta)+\sqrt{2\frac{\hat{Q}h^2-(\hat{Q}h)^2}{m}\log(2/\delta)},$$
we get
$$\P\left(\left\vert \hat{Q}^{*}h -\hat{Q} h\right\vert>\hat{u}_h(\delta) \,\middle\vert\, \mathcal{F}_n\right)\leq \delta$$
and then integrate both sides to obtain
$$\P\left(\left\vert \hat{Q}^{*}h -\hat{Q} h\right\vert>\hat{u}_h(\delta)\right)\leq \delta,$$
which leads to the stated bound.

\section{Proofs of the results on  the $k$-NN conditional sampling error (Section \ref{sec:nn})}\label{app:sec:nn_proof}
Here, we give proofs of the results on the $k$-NN conditional sampling error appearing in Section \ref{sec:nn}.

\subsection{Proof of Proposition \ref{lionel_prop}}\label{app:sec:lionel_prop_proof}

We start with a useful bias-variance decomposition. Introduce
\begin{align*}
&\epsilon_i(x) =  h(Y_i, x)  -  \int h(y,x)   P _{Y \given X} (dy \given   X_i ) , \\
&\Delta(x,X_i)  = \int h(y,x)    (P _{Y\given X} (dy \given   X_i )  -   P _{Y\given X} (dy \given  x ))  .
\end{align*}
 We have
\begin{align*}
& \int h(y,x)  ( \hat P _{Y\given X} (dy \given x )  -  P _{Y\given X} (dy \given x ) ) \\
&=  k^{-1} \sum_{i=1} ^ n  \epsilon_i(x)    1_{\{ B (x, \hat \tau_{n,k,x} ) \}} (X_i)  + k^{-1} \sum_{i=1} ^ n\Delta(x,X_i)   1_{\{ B (x, \hat \tau_{n,k,x} ) \}} (X_i) .
\end{align*}
Integrating with respect to $Q_X (dx)$, we obtain
\begin{align}
 \label{bias_var_decomp}  (\hat Q - Q ) (h)  & = B_h + S_h
 \end{align}
with 
\begin{align*}
& B_h =    k^{-1} \sum_{i=1} ^ n \int \Delta(x,X_i)   1_{\{ B (x, \hat \tau_{n,k,x} ) \}} (X_i)  Q_X(dx),\\
& S_h =  k^{-1} \sum_{i=1} ^ n  \int \epsilon_i(x)    1_{\{ B (x, \hat \tau_{n,k,x}) \}} (X_i) Q_X(dx).
\end{align*}
The term $B_h$ is a bias term and the term $S_h$ (which has mean $0$) is a variance term.

The proof is divided into $3$ steps. The first step takes care of bounding the bias term. The second step deals with the variance upper-bound. The third step is concerned with the variance lower bound.

\paragraph{The bias.}
   First \ref{cond:reg4} gives that for any $X_i \in S_X \cap B (x, \hat \tau_{n,k,x} ) $ and $x\in S_X$,
\begin{equation*}
  \left\vert \Delta(x,X_i)\right\vert\leq g_h(x) \hat \tau_{n,k,x}.
\end{equation*}
Consequently, using \ref{cond:reg3} and the fact that $   \sum_{i=1}^{n} 1_ { \| X_i - x \| \leq \hat \tau_{n,k,x}  }  = k$, we have
\begin{align*}
\left| k^{-1} \sum_{i=1} ^ n \int \Delta(x,X_i)   1_{\{ B (x, \hat \tau_{n,k,x} ) \}} (X_i)  Q_X(dx) \right|
\leq     \int \hat \tau_{n,k,x} g_h(x) Q_X(dx)
\end{align*}
and from Jensen inequality 
\begin{align*}
\left| k^{-1} \sum_{i=1} ^ n \int \Delta(x,X_i)   1_{\{ B (x, \hat \tau_{n,k,x} ) \}} (X_i)  Q_X(dx) \right|^2
\leq     \int \hat \tau_{n,k,x}^2 g_h^2(x) Q_X(dx)
\end{align*}
From Lemma \ref{lem:aux3}, we have $\sup_{x\in S_X}\E[\hat{\tau}_{n,k,x}^2] \leq C_{2,d} k^{2/d}(n+1)^{-2/d}$ %with $C_d=2 ( cV_d b_X ) ^{-1/d}$
and the control of the bias is given by
$$\E\left[ \left\vert B_h^2\right\vert \right] \leq C_{2,d} \int g_h^2(x)Q_X(dx)\cdot\frac{k^{2/d}}{(n+1) ^{2/d}}.$$

\paragraph{The variance upper-bound.} For the proof, we assume that $1 \leq k < n /2$.  We have for each $(x,x') \in S_X^{2}$ and $(i, j) \in \{1, \dots, n\}^{2}$,
\begin{align*}
  \mathbb E [ \epsilon_i(x) \epsilon_j(x')  \given  X_1, \ldots, X_n]
  &= \begin{cases}
0 & \text{ if $i \neq j$},\\
\E\left[\epsilon_i(x)\epsilon_i(x') \given X_i \right]
  \le \sqrt{\E\left[\epsilon_i(x)^2 \given X_i \right] \E\left[\epsilon_i(x')^2 \given X_i \right]} \leq \sigma_+^2 & \text{ if $i \neq j$}.
     \end{cases}
\end{align*}
For the second case, we used \ref{cond:reg5} and the Cauchy-Schwarz inequality.
As a consequence, the variance is given by
\begin{align*}
\mathbb E \left[ S_h^2 \right] &=\mathbb E \left[\left(  k^{-1}  \sum_{i=1}^{n}  \int 1_ { \| X_i - x \| \leq \hat \tau_{n,k,x}  } \epsilon_i(x) Q_X(dx) \right)^2  \right]\\
& =  k^{-2}  \sum_{i,j}^{n} \mathbb E  \left[   \left(   \int 1_ { \| X_i - x \| \leq \hat \tau_{n,k,x}  } \epsilon_i(x) Q_X(dx) \right)\left(   \int 1_ { \| X_j - x' \| \leq \hat \tau_{n,k,x'}  } \epsilon_j(x') Q_X(dx') \right) \right]\\
&=  k^{-2}  \sum_{i,j}^{n} \mathbb E  \left[   \left(   \int\int  1_ { \| X_i - x \| \leq \hat \tau_{n,k,x}  }   1_ { \| X_j - x' \| \leq \hat \tau_{n,k,x'} } \E\left[ \epsilon_i(x)   \epsilon_j(x') \given X_{1}, \dots, X_{n} \right] Q_X(dx)Q_X(dx')   \right) \right]\\
&\le k^{-2} \sigma_+^{2} \sum_{i=1}^{n} \E\left[ \int \int 1_{ \| X_i - x \| \leq \hat \tau_{n,k,x} } 1_{ \| X_i - x' \| \leq \hat \tau_{n,k,x'} } Q_X(dx) Q_X(dx') \right]\\
&= k^{-2} \sigma_+^{2} \sum_{i=1}^{n} \E\left[ \hat{Y}_{i}^2 \right]
\end{align*}
with  $\hat{Y}_i=\int \mathds{1}_{\vert X_i-x\vert\leq \hat{\tau}_{n,k,x}}Q_X(dx)$.
Let $Z_i(x)=\Vert x-X_i\Vert$ and $Z_{(k)}(x)$, the $k-$th order statistics of the sample $\left(Z_i(x)\right)_{1\leq i\leq n}$.
One can observe that 
$$Z_i(x)\leq Z_{(k)}(x)\Longleftrightarrow Z_i(x)<Z_{(k)}^{-i}(x),$$
where $Z_{(k)}^{-i}(x)$ is the $k-$th order statistics of the sample $\left(Z_j(x)\right)_{1\leq j\neq i\leq n}$. Note that the two sigma fields generated respectively by $\left\{Z_i(x):x\in S_X\right\}$ and $\left\{Z_{(k)}^{-i}(x):x\in S_X\right\}$ are independent. 
For one mapping $\rho:S_X\rightarrow diam(S_X)$, we first first bound
\begin{eqnarray*}
I&=&\E\left\vert \int \mathds{1}_{Z_1(x)\leq \rho(x)}Q(dx)\right\vert^2\\
&=& \int\int \P\left(Z_1(x)\leq \rho(x),Z_1(y)\leq \rho(y)\right)Q(dx)Q(dy)\\
&=& 2\int\int \mathds{1}{\{\rho(x)\leq \rho(y)\}}\P\left(Z_1(x)\leq \rho(x),Z_1(y)\leq \rho(y)\right)Q(dx)Q(dy)\\
&=& 2\int\int \mathds{1}{\{\rho(x)\leq \rho(y),\Vert x-y\Vert\leq 2\rho(y)\}}\P\left(Z_1(x)\leq \rho(x),Z_1(y)\leq \rho(y)\right)Q(dx)Q(dy)\\
&\leq& 2M_{2,d}\int\int \mathds{1}{\{\rho(x)\leq \rho(y),\Vert x-y\Vert\leq 2\rho(y)\}}\rho(y)^d Q(dx)Q(dy)\\
&\leq& 2M_{2,d}\int Q\left(B(y,2\rho(y))\right)\rho(y)^d Q(dy)\\
&\leq& 2^{d+1}M_{2,d}^2\int \rho(y)^{2d}Q(dy).
\end{eqnarray*}
The fourth inequality is due to the fact that when $\Vert x-y\Vert>2\rho(y)$, the two balls $B(x,\rho(x))$ and $B(y,\rho(y))$ do not intersect.
We then get using Lemma $3$,
$$E\hat{Y}_i^2\leq 2^{d+1}M_{2,d}^2\int \E \left\{Z_{(k)}^{-i}(y)\right\}^{2d} Q(dy)\leq \frac{2^{d+3}M_{2,d}^2}{M_{1,d}^2} \frac{k^2}{n^2}.$$ 
This leads to the variance upper-bound.

\paragraph{The variance lower-bound.}

Here we assume the $\sigma^2_{-}:=\inf_{x\in S_X}\var\left(h(Y)\vert X=x\right)>0$. 
From previous computations, we have 
$$\var\left(\hat{Q} h\right)\geq \frac{\sigma^2_{-}}{k^2}\sum_{i=1}^n \E\left[\hat{Y}_i^2\right].$$
We have $\E \hat{Y}_i^2\geq \E^2 \hat{Y}_i$ and we have to find a lower bound for $\E\hat{Y}_i$. 
But from the arguments used in the proof for the upper-bound, we have
$$\E\hat{Y}_i=\int_{S_X}\P\left(Z_i(x)\leq Z_{(k)}^{-i}(x)\right)Q(dx)\geq M_{1,d}\int_{S_X}\E\left\vert Z_{(k)}^{-i}(x)\right\vert^dQ(dx)\geq \frac{k}{2n M_{2,d}},$$
where the last inequality follows from Lemma \ref{lem:aux3}. This shows the lower-bound and the proof of Proposition \ref{lionel_prop} is now complete.

\subsection{Proof of Proposition \ref{Portier+}}\label{app:sec:portier+_proof}

Define
\begin{align*}
&\overline{\tau}_{n,k}  = \left(\frac{ 2  k }{ n M_{1,d}}  \right)^{1/ d}.
\end{align*}
The following Lemma~\citep[Lemma 4]{portier2021nearest} controls the size of the $k$-NN balls uniformly over all $x\in S_X$.
%from Lemma 2 in \cite{jiang2019non}.

\begin{lem}[{\citet[Lemma 4]{portier2021nearest}}]\label{prop:tau}
Suppose that \ref{cond:reg0} \ref{cond:reg1} and \ref{cond:reg2} hold true. Then for all $n\geq 1$, $\delta \in (0,1)$ and $1\leq k\leq n$ such that $24 d \log(12n / \delta ) \leq k   \leq T ^d  n b_X c V_d /2 $
, it holds, with probability at least $1-\delta$:
\begin{align*}
 \sup _{x\in S_X} \hat  \tau_{n,k,x}   \leq \overline \tau_{n,k}.
\end{align*}
\end{lem}

We now deal with the variance term of our estimator.
The variance term of the nearest-neighbors estimator is given by $V_n=k^{-1}\sum_{i=1}^n \hat{Y}_i$, where
$$\hat{Y}_i=\int \epsilon_i(x)\mathds{1}_{B(x,\hat{\tau}_{n,k,x})}(X_i)Q_X(dx).$$
and
Set $\hat{\tau}_k=\sup_{x\in S_X}\hat{\tau}_{n,k,x}$. From our assumptions and Jensen's inequality, we have
$$\vert \hat{Y}_i\vert \leq 2 U_h M_{2,d}\hat{\tau}_k,\quad \mbox{Var }\left(\hat{Y}_i\vert X_1,\ldots,X_n\right)\leq \sigma_+^2 M_{2,d}^2\hat{\tau_k}^{2d}.$$

Applying Bernstein's inequality for i.i.d. random variables (we recall that the $Y_i's$ are independent conditionally on the $X_i's$), we get for $t>0$,
$$\P\left(\vert V_n\vert> t\vert X_1,\ldots,X_n\right)\leq 2\exp\left(-\frac{\frac{1}{2}k^2t^2}{n\sigma_+^2M_{2,d}^2\hat{\tau}_k^{2d}+\frac{2}{3}U_h M_{2,d}\hat{\tau}_k^d k}\right).$$

This leads to
$$\P\left(\vert V_n\vert>t, \hat{\tau}_k\leq \overline{\tau}_{n,k}\vert X_1,\ldots,X_n\right)\leq 2\exp\left(-\frac{\frac{1}{2}t^2n}{L_1\sigma_+^2+L_2t}\right).$$
Note that this upper-bound is not random  and we get
\begin{equation}\label{expo}
\P\left(\vert V_n\vert>t, \hat{\tau}_k\leq \overline{\tau}_{n,k}\right)\leq 2\exp\left(-\frac{\frac{1}{2}t^2n}{L_1\sigma_+^2+L_2t}\right).
\end{equation}
Setting
$$\widetilde{t}_n(\delta,h)=\frac{L_2}{n}\log(2/\delta)+\sqrt{\frac{L_2^2}{n^2}\log^2(2/\delta)+\frac{2L_1\sigma_+^2}{n}\log(2/\delta)},$$
which is smaller than 
$$t_n(\delta,h) = L_0 \left(\frac{k}{n}\right)^{1/d}+\frac{2L_2}{n}\log(2/\delta)+\sqrt{\frac{2L_1\sigma_+^2}{n}\log(2/\delta)}$$ 
we then get
\begin{eqnarray*}
\P\left(\vert V_n\vert >t_n(\delta,h)\right)&\leq& \P\left(\vert V_n\vert>\widetilde{t}_n(\delta,h),\hat{\tau}_k\leq \overline{\tau}_{n,k}\right)+P\left(\hat{\tau}_k>\overline{\tau}_{n,k}\right)\\
&\leq& 2\delta,
\end{eqnarray*}
where the last inequality is a consequence of (\ref{expo}) and Lemma \ref{prop:tau}. Moreover, from the proof of Theorem \ref{lionel_prop} and Lemma \ref{prop:tau}, the bias part can be dominated by $L_0 (k/n)^{1/d}$ with probability at least $1-\delta$.
This concludes the proof.

\section{Proof of Theorem \ref{global}}\label{app:sec:global}
First, note that the boundedness of $\sup_{x\in S_X}\E\left[h^2(Y,x)\vert X\right]$ entails {\bf H2}.
Setting $A_{m,n}=\hat{Q}^{*}(h)-\hat{Q}(h)$ and $B_n=\hat{Q}(h)-Q(h)$, Proposition \ref{lionel_prop} guarantees that 
$$\E B_n^2\leq  C_1\left\{\frac{1}{n^{2/d}}+\frac{1}{n}\right\},$$
for some $C_1>0$ only depending on the distribution of $(X,Y)$, $X^{*}$ and on $h$. 
It remains to show that the same bound can obtained for $A_{m,n}$. Since, 
$$\E\left[A_{m,n}^2\vert \mathcal{F}_n\right]\leq \frac{\hat{Q}(h^2)-\hat{Q}(h)^2}{m}\leq \frac{\hat{Q}(h^2)}{m}.$$
It only remains to show that $\E \hat{Q}(h^2)$ is bounded with respect to $n$. The approach is similar to the control of the variance term for $k=1$ studied in the proof of Proposition \ref{lionel_prop}. If $L$ is an upper-bound for $\sup_{x\in S_X}\E\left[h^2(Y,x)\vert X\right]$, we have  
\begin{eqnarray*}
\E\hat{Q}(h^2)&=& \sum_{i=1}^n \int_{S_X} Q_X(dx)\E\left[h^2(Y_1,x)\mathds{1}_{\Vert X_1-x\Vert\leq \min_{2\leq j\leq n}\Vert X_j-x\Vert}\right]\\
&\leq& n L\int_{S_X}\P\left(\Vert X_1-x\Vert\leq \min_{2\leq j\leq n}\Vert X_j-x\Vert\right)Q_X(dx)\\
&\leq& n L M_{2,d}\int_{S_X}\E \min_{2\leq j\leq n}\Vert X_j-x\Vert^d  Q_X(dx)\\
&\leq& \frac{2LM_{2,d}}{M_{1,d}}.
\end{eqnarray*}
The last upper bound is obtained from Lemma \ref{lem:aux3} using the fact that $\min_{2\leq j\leq n}\Vert X_j-x\Vert$ 
has the same probability distribution as $\hat{\tau}_{n-1,1,x}$. We deduce the result taking $C$ as the maximum between $C_1$
and $2LM_{2,d}/M_{1,d}$. \QEDA

\section{A corollary bounding the sampling error for $k$-NN sampling}\label{app:sec:cor_kNN_error_bound}

\begin{cor}\label{mainpoint}
Suppose that \ref{cond:reg0} \ref{cond:reg1} \ref{cond:reg2} \ref{cond:reg3} are fulfilled and that  \ref{cond:reg4} and \ref{cond:reg5} are fulfilled for both $h$ and $h^2$ and $h$ is bounded. Let $k=k_n$ satisfying the condition of Lemma \ref{prop:tau} and $\delta\in (0,1/7)$ and set $s^2(h)=Q h^2-(Q h)^2$ and $\hat{s}^2(h)=\hat{Q}h^2-(\hat{Q} h)^2$.
Let $\delta\in (0,1/7)$. Then with probability greater than $1-7\delta$,
\begin{equation}\label{fb2}
\hat{Q}^{*}h-\hat{Q}h\leq \frac{4/3 U_h}{m}\log(2/\delta)+\sqrt{2\frac{v^2_n(\delta,h)}{m}\log(2/\delta)}+\sqrt{2\frac{s^2(h)}{m}\log(2/\delta)},
\end{equation}
where
$$v^2_n(\delta,h)=t_n(\delta,h^2)+t_n(\delta,h)^2+2t_n(\delta,h) Q\vert h\vert$$
 and $t_n(\delta,h)$ is defined in the statement of Proposition \ref{Portier+}.
\end{cor}

\begin{proof}
We first use the result of Proposition \ref{expsampling}. In particular, setting
$$\hat{u}_h(\delta)=\frac{4/3 U_h}{m}\log(2/\delta)+\sqrt{2\frac{\hat{Q}h^2-(\hat{Q}h)^2}{m}\log(2/\delta)},$$
we have
$$\P\left(\left\vert \hat{Q}^{*}h -\hat{Q} h\right\vert>\hat{u}_h(\delta)\right)\leq \delta,$$
We then use the decomposition 
$$\hat{Q}h^2-(\hat{Q}h)^2=s^2(h)+\hat{Q}h^2-Q h^2-\left(\hat{Q}h-Q h\right)^2-2 Qh\left(\hat{Q}h-Qh\right).$$
From Proposition \ref{Portier+}, we know that
$$\left\vert \hat{Q}h^2-Q h^2\right\vert\leq t_n(\delta,h^2)$$
with probability greater than $1-3\delta$ and
$$\left\vert \hat{Q}h-Q h\right\vert \leq t_n(\delta,h)$$
with probability greater than $1-3\delta$. Collecting these three bounds, we easily obtain the conclusion of the second point of Corollary \ref{mainpoint}.
\end{proof}

\section{Proofs of the results on the empirical risk minimization (Section \ref{submit})}\label{app:sec:submit_proof}
Here, we present proofs of the result on the application to empirical risk minimization.

\subsection{Proof of Theorem \ref{notsomuch1}.}\label{app:sec:notsomuch1_proof}
From  \ref{cond:reg_wcv1}, $\sup_{\theta\in\Theta}\left\vert m_{\theta}\left(Y_1^{*},X^{*}_1\right)\right\vert$ is integrable and $\theta\mapsto \mathcal{R}^{*}(\theta)$ is continuous over the compact set $\Theta$. As a consequence, weak consistency will follow from Theorem $5.7$ in \cite{van2000asymptotic} if we show that 
\begin{equation}\label{but}
\sup_{\theta\in\Theta}\left\vert \mathcal{R}^{*}_{m,n}(\theta)-\mathcal{R}^{*}(\theta)\right\vert=o_{\P}(1).
\end{equation}
Pointwise convergence holds true from assumptions~\ref{cond:reg_wcv1}, \ref{cond:reg_wcv2} as each mapping $m_{\theta}$ satisfies Assumptions~\ref{cond:reg4}, \ref{cond:reg5}. One can then apply 
Theorem \ref{global} and the Markov inequality to get for any $\theta\in\Theta$,
$$\mathcal{R}^{*}_{m,n}(\theta)-\mathcal{R}^{*}(\theta)=o_{\P}(1),\quad m,n\rightarrow \infty.$$
 We now prove uniform convergence. Let $\delta>0$. One can cover 
the compact set $\Theta$ with finitely many balls $B(\theta_i,\delta)$, $1\leq i\leq k$. For $\theta\in \Theta\cap B(\theta_i,\delta)$, 
we have
\[
\left\vert \mathcal{R}^{*}_{m,n}(\theta)-\mathcal{R}^{*}_{m,n}(\theta_i)\right\vert\leq \eta(\delta)\frac{1}{m}\sum_{i=1}^m h\left({Y}^{*}_{n,i}\right).
\]
Moreover, from Assumptions \ref{cond:reg_wcv2} and Theorem \ref{global} with the Markov inequality, we know that
\[
\frac{1}{m}\sum_{i=1}^m h\left({Y}^{*}_{n,i}\right) \to \E h\left(Y_1^{*}\right), \quad m,n\rightarrow \infty,
\]
in probability. We also have 
$$\left\vert \mathcal{R}^{*}(\theta)-\mathcal{R}^{*}(\theta_i)\right\vert\leq \eta(\delta)\E h\left(Y_1^{*}\right).$$
Finally, one can use the bound 
\begin{eqnarray*}
\sup_{\theta\in \theta}\left\vert \mathcal{R}^{*}_{m,n}(\theta)-\mathcal{R}^{*}(\theta)\right\vert&\leq& \max_{1\leq i\leq k}\left\vert \mathcal{R}^{*}_{m,n}(\theta_i)-\mathcal{R}^{*}(\theta_i)\right\vert + \eta(\delta)\left\{\frac{1}{m}\sum_{i=1}^m h\left({Y}^{*}_{n,i}\right)+\E h\left(Y_1^{*}\right)\right\}\\
&=& 2\eta(\delta)\E h\left(Y_1^*\right)+o_{\P}(1).
\end{eqnarray*}
Given that $\delta$ is arbitrary, the above implies (\ref{but}) and the weak consistency of $\hat{\theta}^{*}$ follows. The second assertion about the excess risk then follows easily using that
$$\mathcal{R}^{*}(\hat{\theta}^{*})-\mathcal{R}^{*}({\theta}^{*})\leq 2 \sup_{\theta\in\Theta}\left\vert \mathcal{R}^{*}_{m,n}(\theta)-\mathcal{R}^{*}(\theta)\right\vert.$$
 \QEDA
\subsection{Proof of Theorem \ref{notsomuch}.}\label{app:sec:notsomuch_proof}
Let $Z _ i^* = \Gamma^{-1/2} X_i^*$ and define
$$\Sigma_m=\frac{1}{m}\sum_{i=1}^m Z_i^{*}{Z_i^{*}}^T,\quad N_m=\frac{1}{m}\sum_{i=1}^m Z_i^{*}\left[{Y}^{*}_{n,i}-{X_i^{*}}^T\theta^{*}\right].$$
The proof first requires some analysis of the smallest eigenvalues of $\Sigma_m$. From the matrix Chernoff inequality given in \citet{tropp2012user}, see Corollary $5.2$ and Remark $5.3$, we have 
$$ P ( \lambda_ {\min} ( \Sigma_m ) \leq 1- \eta) \leq d \exp(-\eta ^2 m/ 2 B )   $$ 
 where $B= \lambda_{\min} (\Gamma)^{-1} d \max_{x\in S_X} | x|_\infty^2$ is defined so as to satisfy
$ \| Z\|_2^2 \leq \lambda_{\min} (\Gamma)^{-1} \| X\|_2^2 \leq B$, with probability $1$. Inverting the previous we obtain that with probability  at least $ 1-\delta$,
$$ \lambda_ {\min} ( \Sigma_m ) > 1 - \sqrt{ (2B / m\log(d \delta^{-1})) } $$
and therefore as soon as $ (8 B / n)\log(d \delta^{-1}) \leq 1 $, we have that $ \lambda_ {\min} ( \Sigma_m ) \geq  1/2$. 
On the previous event, we have that $$\Gamma^{1/2} 
(\hat \theta^* - \theta^*) = \Sigma_m^{-1}N_m$$ 
It follows that
\begin{align*}
     \mathcal{R}^{*}(\hat{\theta}^{*})-\mathcal{R}^{*}(\theta^*)=  \|\Gamma^{1/2} 
(\hat \theta^* - \theta^*) 
\|_2^2 
= \|\Sigma_m^{-1} 
N_m 
\|_2^2 
\leq 2  \|N_m\|_2^2.
\end{align*}
We conclude using Theorem \ref{global} with $h(y,x)=\Gamma^{-1/2}x\left(y-x^T\theta^{*}\right)$.
Note that $\E\left[h\left(Y^{*},X^{*}\right)\right]=0$ by definition of $\theta^{*}$. \QEDA

\section{Illustration for Experiments E1--E3}\label{app:moreplots}

Illustrations of data used in Experiments E1--E3 can be found in Figure~\ref{app:fig:data1}.

\begin{figure*}[htb]
  \centering
  \begin{subfigure}[b]{0.45\textwidth}
    \centering
    \includegraphics[width=\textwidth]{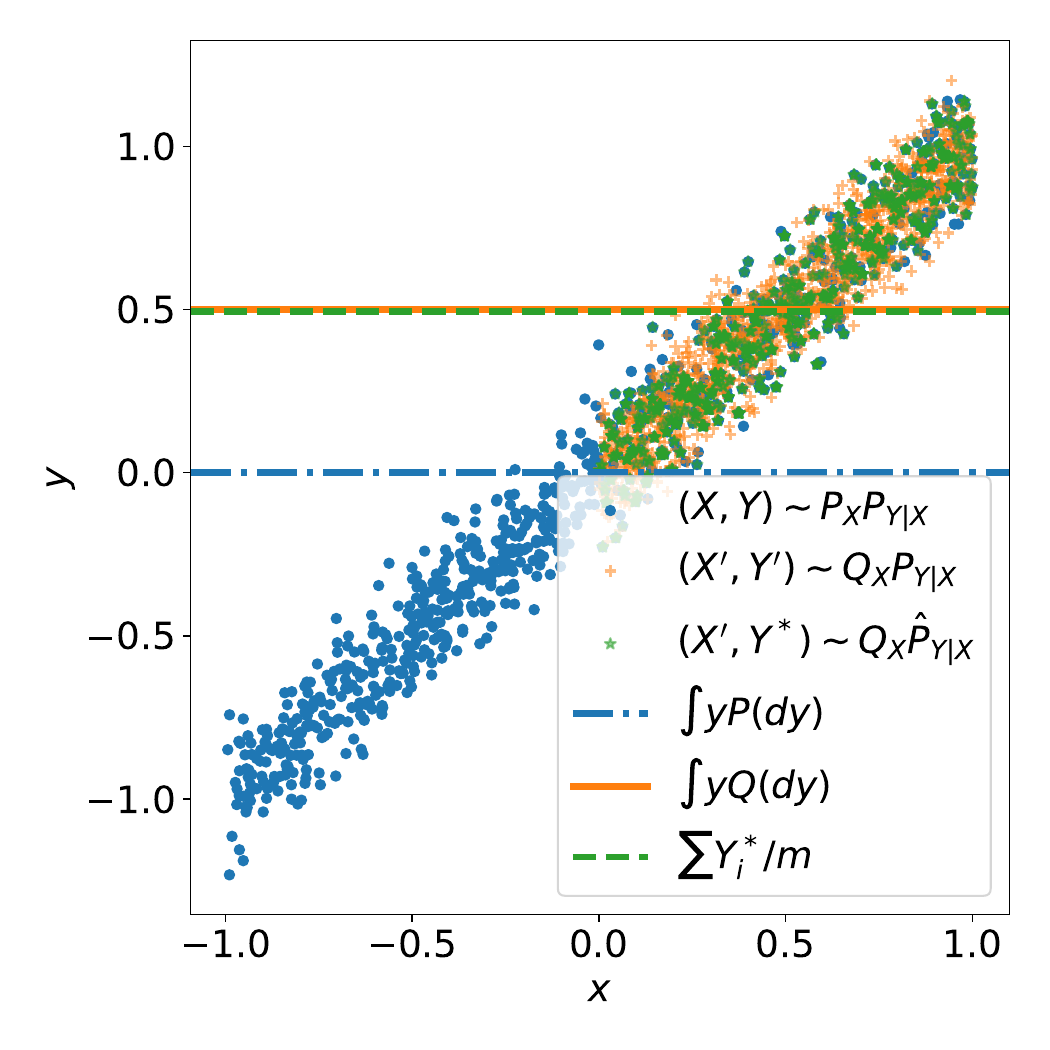}
    \caption{Data for Experiment  E1}
    \label{app:fig:data0}
  \end{subfigure}
  \hfill
  \begin{subfigure}[b]{0.45\textwidth}
    \centering
    \includegraphics[width=\textwidth]{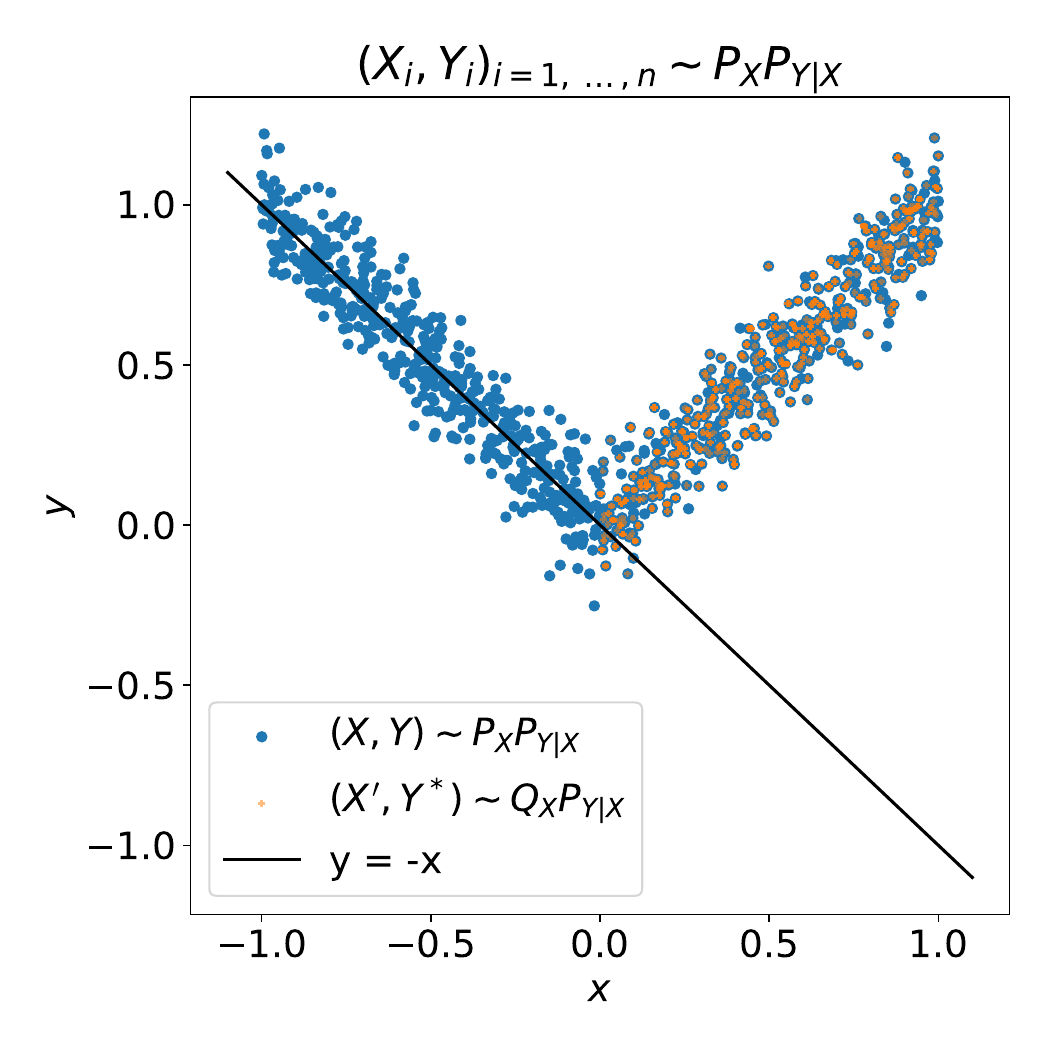}
    \caption{Data for Experiment E2 and E3}
    \label{app:fig:data1}
  \end{subfigure}
  \caption{Visualization of data for Experiments E1--E3}
  \label{app:fig:errors0}
\end{figure*}

\section{Details of the benchmark data experiments}\label{app:sec:experiment_details}
We use the following datasets.
\begin{itemize}
  \item \texttt{california}:  Regression dataset called ``California Housing'' available from \url{https://www.dcc.fc.up.pt/~ltorgo/Regression/cal_housing.html}.
  
  \item \texttt{diabetes}: Regression dataset available from \url{https://archive.ics.uci.edu/ml/index.php}~\citep{dua2017uci}.

  \item \texttt{breast cancer}: Classification dataset available from \url{https://archive.ics.uci.edu/ml/index.php}~\citep{dua2017uci}.
  
  \item \texttt{twonorm}: Classification dataset available from \url{https://www.cs.utoronto.ca/~delve/data/datasets.html}.
\end{itemize}

\paragraph{Data splitting and sampling bias simulation}
We split the original to the training and test set and simulate covariate shift by rejection sampling from the test set with rejection probability determined according to the value of a covariate.
For \texttt{california}, \texttt{twonorm}, \texttt{breast cancer}, we follow the procedure of \citet{sugiyama2007direct}:
we include each target data point $X_{i}$ to the target set with probability $\min(1, 4X_{i, c}^2)$ or reject it otherwise, where $X_{i, c}$ is the $c$-th attribute of $X_{i}$.
For \texttt{diabetes}, we used a different biasing procedure for this data set because the technique of \citet{sugiyama2007direct} rejects too many data points to perform our experiment for this dataset.
We instead use the procedure of an example from the ADAPT package~\citet{demathelin2021adapt}\footnote{\url{https://adapt-python.github.io/adapt/examples/Sample_bias_example.html}} for \texttt{diabetes}:
for each data point $X_{i}$, we accept it with probability proportional to $\exp(-20\times\abs{X_{i, \text{age}} + 0.06})$, where $X_{i, \texttt{age}}$ is the \texttt{age} attribute of $X_{i}$ and reject (i.e., exclude) otherwise.

\paragraph{Pre-processing}
We use the hot-encoding for all categorical features.
We center and normalize all the data using the mean and the dimension-wise standard deviation of the source set.
We do the same centering and normalization for the output variables for regression datasets.

After training and prediction, we post-process the output using the inverse operation. Table~\ref{tab:data_info} shows basic information about the datasets after the bias-sampling and pre-processing.

\begin{table}[hbt]
  \caption{Basic information of the datasets}
  \label{tab:data_info}
  \centering
  \begin{tabular}{ccccc}
    \toprule
                       & california & twonorm & diabetes & breast cancer\\
    \midrule
Input dimension    $d$ & 8          & 20      & 10       & 9  \\
source sample size $n$ & 1000       & 100     & 150      & 200 \\
Target sample size $m$ & 1000       & 500     & 150      & 100 \\
    \bottomrule
  \end{tabular}
\end{table}

%\end{appendix}

\end{document}